\documentclass{article}
\usepackage{iclr2027_conference,times}
\usepackage[utf8]{inputenc}
\usepackage[T1]{fontenc}
\usepackage{hyperref}
\usepackage{url}
\usepackage{booktabs}
\usepackage{amsfonts}
\usepackage{nicefrac}
\usepackage{microtype}
\usepackage{graphicx}
\usepackage{amsmath,amssymb,amsthm}
\usepackage{multirow}
\usepackage{float}
\usepackage{subcaption}
\usepackage{algorithm}
\usepackage{algorithmic}
\newcommand{\mab}[1]{\mathbb{#1}}
\newcommand{\mac}[1]{\mathcal{#1}}
\newtheorem{theorem}{Theorem}

\newtheorem{lemma}[theorem]{Lemma}
\newtheorem{assumption}{Assumption}[section]
\newtheorem{definition}{Definition}[section]

\theoremstyle{remark}
\newtheorem{remark}[theorem]{Remark}
\title{Joint Contrastive-Adversarial Representation Learning for
Unsupervised Domain Adaptation}
\hypersetup{
  pdftitle={Joint Contrastive-Adversarial Representation Learning for Unsupervised Domain Adaptation},
  pdfauthor={Yuling Jiao, Lican Kang, Jerry Zhijian Yang, Zhou Yu, Fusheng Zhou}
}
\author{
  Yuling Jiao \\
  \texttt{yulingjiaomath@whu.edu.cn}
  \And
  Lican Kang \\
  \texttt{kanglican@whu.edu.cn}
  \AND
  Jerry Zhijian Yang \\
  \texttt{zjyang.math@whu.edu.cn}
  \And
  Zhou Yu \\
  \texttt{zyu@stat.ecnu.edu.cn}
  \And
  Fusheng Zhou \\
  \texttt{zfs1720@whu.edu.cn}
}
\iclrfinalcopy
\begin{document}
\maketitle
\lhead{Preprint. Under review at ICLR 2027}
\begin{abstract}
Effective unsupervised domain adaptation (UDA) requires a representation that captures transferable structure 
from both the labeled source domain and the unlabeled target domain. We introduce Joint Contrastive-Adversarial 
Representation Learning (JCAR), a two-stage framework that learns a common representation from the union of 
unlabeled source and target samples by combining contrastive learning with adversarial domain alignment, 
and then freezes the representation to fit a source-supervised predictor. We establish a non-asymptotic 
excess-risk bound for this procedure that couples representation and domain-discriminator estimation while separating 
downstream prediction, with explicit neural-network approximation errors. Under the stated assumptions and common H\"older smoothness $\beta\geq1$,
the bound identifies a sufficiently large unlabeled-sample regime in which the downstream rate 
is $\widetilde{\mathcal{O}}(m^{-\beta/(2\beta+d)})$, where $m$ counts labeled source samples and logarithmic factors are suppressed. 
The labeled-sample exponent is governed by the latent dimension $d$ rather than the 
ambient dimension $p$, which enters the pretraining cost. Experiments on Office-31, Office-Home, and DomainNet demonstrate clear transfer 
gains across diverse domain shifts, with particularly strong performance when representations are 
learned from scratch.
\end{abstract}
\section{Introduction}
Domain adaptation (DA) \citep{pan2010} aims to transfer a predictor from a labeled source domain to an unlabeled target domain. 
UDA methods use target samples in different ways: for example, statistical matching aligns feature distributions \citep{gretton2012optimal}, 
reweighting corrects distribution shift \citep{kato2021non}, and adversarial methods train a feature extractor against a domain 
discriminator \citep{long2018conditional}. In many such methods, however, task-discriminative supervision is available only 
in the source domain, while target samples enter primarily through an alignment objective. 
This asymmetry motivates us to ask whether both domains can participate directly in unsupervised representation 
construction before the source-supervised predictor is fitted. Figure \ref{fig_rep1} contrasts the conventional 
source-supervised adaptation pipeline with the joint-domain representation-learning pipeline in Figure \ref{fig_rep2}.
\begin{figure}[ht]
\vspace{-0.1em}
    \centering
    \begin{subfigure}[t]{0.45 \linewidth}
        \centering
        \includegraphics[width=\textwidth]{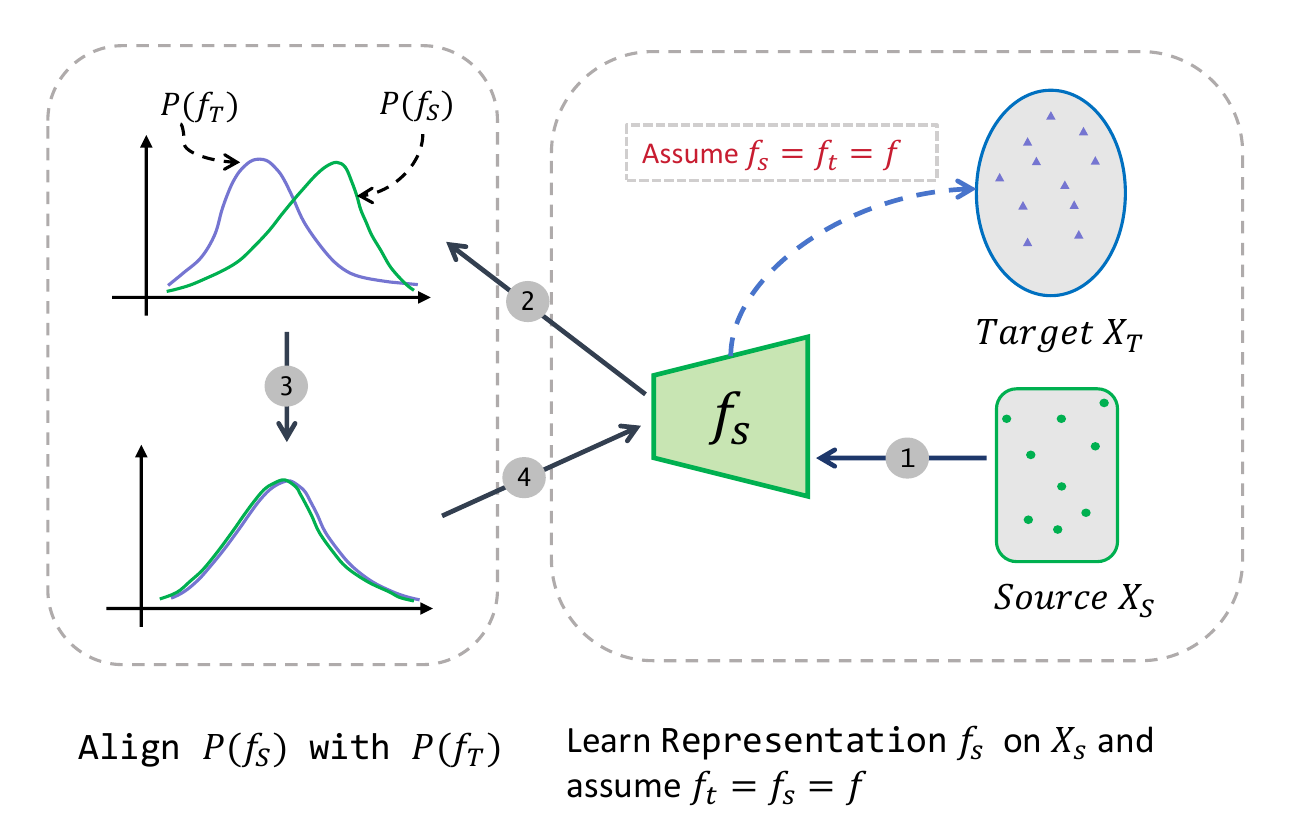}
        \caption{A conventional UDA representation-learning pipeline.}
        \label{fig_rep1}
    \end{subfigure}
    \hfill
    \begin{subfigure}[t]{0.45 \linewidth}
        \centering
        \includegraphics[width=\textwidth]{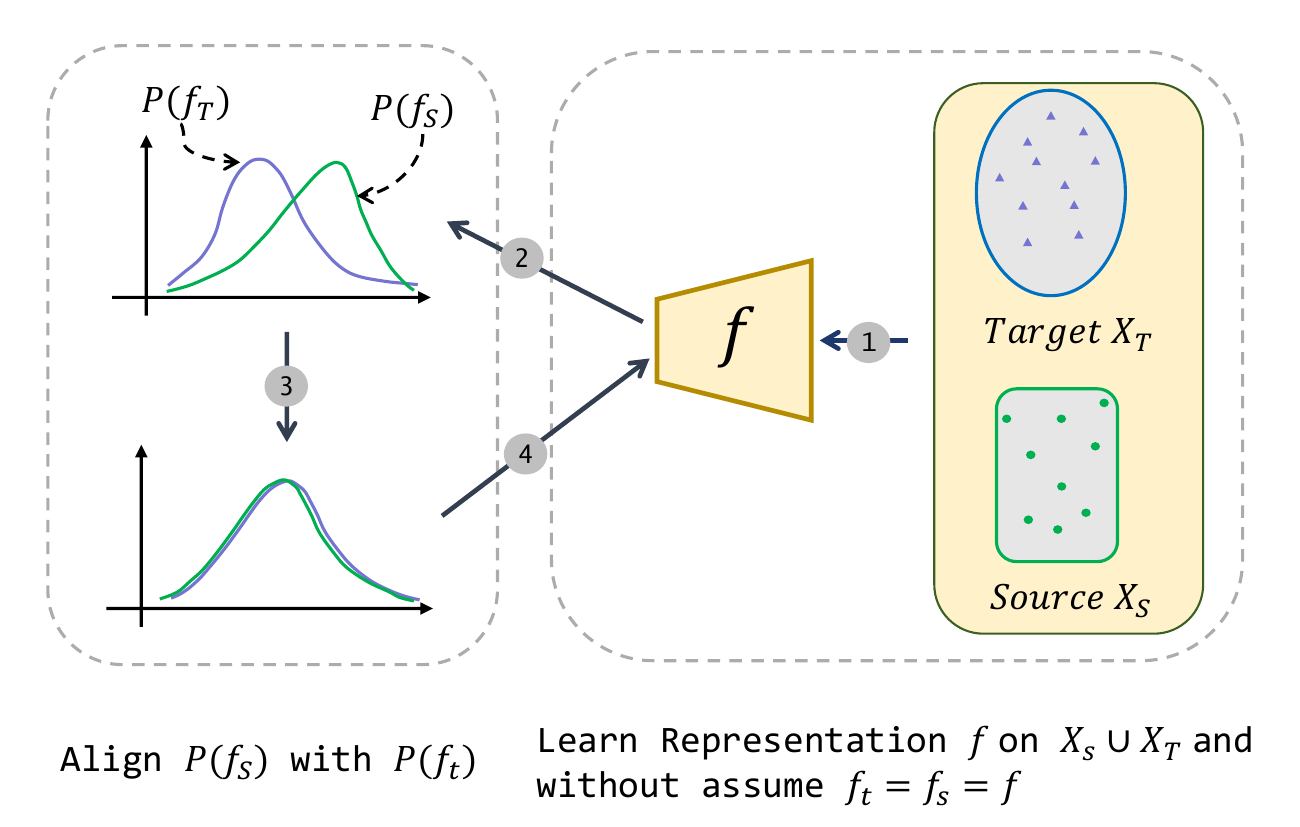}
        \caption{The proposed joint-domain representation-learning pipeline.}
        \label{fig_rep2}
    \end{subfigure}
\end{figure}
To address this question, we introduce Joint Contrastive-Adversarial Representation Learning (JCAR). 
In Stage 1, contrastive learning \citep{chen2020simple} allows unlabeled source and target samples to jointly shape the representation, 
while an adversarial domain objective explicitly reduces the distinguishability of their feature distributions. 
In Stage 2, the learned feature extractor is frozen and a predictor is trained on labeled source samples. 
This design exploits target-domain structure without target labels and cleanly separates task-agnostic representation learning from downstream prediction.
Our non-asymptotic analysis jointly controls representation learning and domain discrimination while separating their cost from downstream prediction. 
It explicitly balances neural-network approximation and statistical errors, 
linking the coupled complexity of contrastive--adversarial pretraining to downstream label efficiency. 
Under the theorem's assumptions, sufficiently many unlabeled samples yield a downstream polynomial rate governed by the representation dimension $d$, 
while the ambient dimension $p$ enters the pretraining cost.
Our contributions are summarized as follows:
\begin{itemize}
\item We formulate a two-stage UDA procedure in which unlabeled source and target samples jointly drive contrastive representation learning and explicit adversarial alignment. Freezing the learned extractor in the second stage decouples this task-agnostic representation objective from source-supervised prediction.
\item We derive a coupled pretraining rate by jointly balancing the statistical and approximation costs of the feature extractor and domain discriminator. Under the assumptions of Theorem \ref{theorem_main}, common H\"older smoothness $\beta\geq1$, and sufficiently many unlabeled samples, the excess-risk bound reduces to $\widetilde{\mathcal O}(m^{-\beta/(2\beta+d)})$, where $m$ counts labeled source samples and the tilde suppresses logarithmic factors. This quantifies the distinct roles of unlabeled data in high-dimensional feature learning and labels in low-dimensional prediction.
\item We evaluate JCAR under both random initialization and supervised pretraining on Office-31, Office-Home, and DomainNet. In the from-scratch setting, JCAR leads the ResNet group on four of six Office-31 transfers, and its ViT variant improves over ViT-B on four of six Office-31 and five of six Office-Home transfers; feature visualizations further illustrate the cross-domain alignment learned by the method.
\end{itemize}
\subsection{Organization}
Section~\ref{Perli} fixes the notation and basic concepts before JCAR is developed in Section~\ref{sec_Pre}. Its guarantees and empirical behavior are examined in Sections~\ref{sec_Error} and~\ref{sec_Numer}, respectively, while Section~\ref{sec_Conclu} closes with the main conclusions.
\section{Preliminaries}\label{Perli}
\subsection{Notations}
Let $\mac{X}$ and $\mac{Y}$ be the input and output spaces, respectively, where $\mac{X} \subset \mab{R}^p$ and $\mac{Y} \subset \mab{R}$.
Without loss of generality, we constrain $\mac{X}= [0,1]^p$ in this work. 
Let $\mac{S}$ and $\mac{T}$ be subsets of $\mac{X} \times \mac{Y}$, where $\mac{S}$ is the source domain and $\mac{T}$ is the target domain. Denote 
the source covariate--response pair by $(X^{\mac{S}},Y^{\mac{S}})\in\mac{X}\times\mac{Y}$ and its distribution by $P_S(X,Y)$. Similarly, $(X^{\mac{T}},Y^{\mac{T}})$ has target distribution $P_T(X,Y)$. The corresponding covariate marginals are $P_S(X)$ and $P_T(X)$. In the domain-adaptation setting, we assume $P_S(X)\neq P_T(X)$ and $P_S(Y\mid X)=P_T(Y\mid X)$.
 Let $\mac{M}_{f} :=\{f: \mab{R}^p \to \mab{R}^d\}$,  $\mac{M}_h :=\{h : \mab{R}^d \to \mab{R}\}$ and $\mac{M}_g :=\{g : \mab{R}^d \to \mab{R}\}$ be the
 H\"older classes  where $d < p$. Let
 $\mac{F}:=\{f: \mab{R}^p \to \mab{R}^d\}$, $\mac{H}:=\{h:\mab{R}^d \to \mab{R}\}$,
 and $\mac{G}:=\{g:\mab{R}^d \to \mab{R}\}$ denote the hypothesis function classes. In this work,  these functions are implemented as DNNs. Let $\|f\|_{L_q(\mu)}$ denote the $L_q$ norm of $f$ with respect to the probability measure $\mu$. 
For $x \in \mathbb{R}^p$, we write $\|x\|_q:=(\sum_{j=1}^p|x_j|^q)^{1/q}$ for $q \in [1,\infty)$.
Let $a, b >0$.
 We write 
 $a \lesssim b$ if there exists a constant $c>0$ such  that $a \leq cb$. For any $y\in \mab{R}$, we denote $\lfloor y \rfloor:= \max\{n \in \mab{Z}|n \leq y\}$, where $\mab{Z}$ is the integer set.
\begin{definition}[Contrastive loss]\label{cl_loss}
    Let $\{X_i\}_{i=1}^n$ be i.i.d. samples from $P_{\mac X}$ on $\mathcal X$. For each anchor $X_i$, let $X_i^+$ be an independently augmented view of the same underlying sample, and let $X_i^-$ be a view associated with a different sample drawn independently from $P_{\mac X}$.
    Let $f: \mab{R}^p \to \mab{R}^d$ be a representation function, where $d<p$. For the theoretical analysis, we use the following pairwise logistic contrastive loss with one sampled negative:
\begin{align*}
\ell_{cl}(X, X^+, X^-, f):= \log \left(1 + \exp\{f(X)^{\top} \big(f(X^-)-f(X^+)\big)\} \right).   
\end{align*}
To simplify notation, we write $\ell_{cl}(X,f)$ when the sampling rule for $(X^+,X^-)$ is clear. 
The experimental mini-batch objective extends this pairwise comparison by using other mini-batch samples as negatives (Appendix \ref{sup_expe}), while the expression above enables a tractable statistical analysis of its core contrastive mechanism. The empirical risk of $f$ is 
    \begin{align*}
    \hat{\mac{R}}_{cl}
        (f):= \frac{1}{n} \sum_{i = 1}^n 
        \log \left(1 + \exp\{f(X_i)^{\top} \big(f(X_i^-)-f(X_i^+)\big)\} \right).
    \end{align*}
\end{definition}
\begin{definition}[Domain discrimination loss]\label{da_loss}
Let $\{(X_k, Z_k)\}_{k = 1}^n$ be i.i.d. samples from a joint distribution $P_{\mu}(X,Z)$, where $Z_k$ is a domain label taking values in $\{-1,1\}$. Let $f: \mab{R}^p \to \mab{R}^d$ be a representation function, and $h: \mab{R}^d  \to \mab{R}$ be a domain discriminator.     
We define the domain-classification utility as the negative logistic loss:
   \begin{equation*}
        \ell_{d}(X, Z, f, h):= -\log(1 + \exp(-Z \cdot h\circ f(X))).
   \end{equation*}
The empirical risk of the pair $(f,h)$ is defined by
    \begin{equation*}
       \hat{\mac{R}}_{d}(f, h):= -\frac{1}{n}\sum_{k =1}^n \log(1 + \exp(-Z_k\cdot h\circ f(X_{k}))).
    \end{equation*}
Under the minimax objective \eqref{eq_f}, the discriminator $h$ maximizes $\hat{\mac R}_d$, 
equivalently minimizing the usual positive logistic loss, whereas the feature extractor $f$ minimizes the same utility 
to make the two domains harder to distinguish. This sign convention therefore implements adversarial feature alignment.
\end{definition}
\begin{definition}[Rademacher complexity]
Let $X_{1:n}:=\{X_i\}_{i =1}^n$ be $n$  i.i.d. random variables on $\mathcal X$, 
and let $\mathcal F$ be a set of measurable functions defined  on $\mathcal X$.   Let $
\zeta_{1:n}:=\{\zeta_i\}_{i =1}^n$ be $n$ i.i.d.  Rademacher random variables, each uniformly taking  values $-1$ or $1$ with probability $1/2$.
Then, the empirical Rademacher complexity of $\mathcal F$ is defined as
$$
\hat{\mathfrak{R}}_n(\mathcal{F})
= \mathbb{E}_{\zeta_{1:n}}\left[\sup_{f \in \mathcal{F}} \frac{1}{n}\sum_{i = 1}^n \zeta_i  f(X_i)  \right].
$$    
The Rademacher complexity of $\mac{F}$ is the expectation of $\hat{\mathfrak{R}}_{n}(\mac{F})$:
\begin{equation*}
   \mathfrak{R}_n(\mathcal{F}) = \mab{E}_{X_{1:n}}[\hat{\mathfrak{R}}_n(\mathcal{F})].
\end{equation*}
\end{definition}
\section{Problem Formulation}\label{sec_Pre}
JCAR has two stages: joint-domain representation learning (Section \ref{sec:rp}) and source-supervised predictor fitting on frozen features (Section \ref{sec:ds}).
The resulting predictor is then applied to the target domain.
\subsection{Representation learning}\label{sec:rp}
Our central objective is to learn representations that not only capture rich information from both the source and target domains but also exhibit domain invariance. 
To this end, we propose an unsupervised representation learning framework that integrates the principles of contrastive learning and domain discrepancy minimization, and we accordingly formulate a contrastive 
domain-adaptive loss to jointly enforce semantic consistency and 
domain alignment.
Unlike conventional domain-adaptation pipelines (Figure~\ref{fig_rep1}), which use source supervision to shape the representation and then align it with the target domain, our approach uses source and target samples jointly during representation learning (Figure~\ref{fig_rep2}). This directly incorporates target-domain structure into the learned representation.
We have a source domain dataset $D_{\mac{S}}:=\{X_i^{\mathcal{S}}\}_{i=1}^{n_s}$ i.i.d. drawn from $P_{S}(X)$, 
and a target domain dataset $D_{\mac{T}}:=\{X_j^{\mathcal{T}}\}_{j=1}^{n_t}$ i.i.d. drawn from $P_{T}(X)$. 
Let $n:=n_s+n_t$, and denote $\alpha_n:=\frac{n_s}{n}$.
Without loss of generality, we assume that $\alpha_n \rightarrow \alpha \in (0,1)$ as $n$ tends to $\infty$.
Accordingly, we introduce $X$ as a mixture of source and target covariates, and let $Z$ denote its pseudo label, with $Z = 1$ for $X = X^{\mac{S}}\sim P_S$ and $Z = -1$ for $X = X^\mac{T}\sim P_T$.
Precisely, the mixture $X$ can be formulated as
\begin{equation*}
X=\frac{1+Z}{2}X^\mac{S}+\frac{1-Z}{2} X^\mac{T},    
\end{equation*}
where $Z$ is a Rademacher-type domain indicator with $P(Z=1)=\alpha$ and $P(Z=-1)=1-\alpha$,
and is independent of  $X^\mac{S}$ and  $X^\mac{T}$.
In this case, we denote the joint distribution of $(X,Z)$ by
$P_{\mu}(X,Z)$,
whose marginal distribution of $X$ is $P_C(X)=\sum_{z\in \{-1,1\}}P(X\mid Z=z)P(Z=z)=\alpha P_S(X)+(1-\alpha)P_T(X)$.
Conditionally on an anchor $X=x$, a positive view is sampled from an augmentation kernel $P_A(\cdot\mid x)$, 
whereas a negative $X^-$ is associated with an independently drawn sample from $P_C$. 
Thus the joint sampling law used by Definition \ref{cl_loss} is $P_C(dx)P_A(dx^+\mid x)P_C(dx^-)$, 
rather than a product law that makes the positive view independent of its anchor.
Let $f: \mab{R}^p \to \mab{R}^d$ be a representation function,  and let $h: \mab{R}^d \to \mab{R}$ be a domain discriminator.
The contrastive domain-adaptive loss $\ell_{ad}(X, Z, f,h)$ is
defined as 
\begin{equation}\label{clda_loss}
        \ell_{ad}(X, Z, f, h):= \ell_{cl}(X, f) +  \ell_{d}(X, Z, f, h),
    \end{equation} 
where $\ell_{cl}$ is the contrastive loss (see Definition \ref{cl_loss}), and  $\ell_{d}$ is the domain discrimination loss (see Definition \ref{da_loss}). 
Let $D_{J}:=\{(X_k, Z_k)\}_{k = 1}^n$ be i.i.d. paired samples drawn from $P_{\mu}(X,Z)$. Let $D_{C}:=\{X_1,\ldots,X_n\}$ denote the covariates and let $D_{\bar{C}}:=\{(X_i,X_i^+,X_i^-)\}_{i=1}^n$.
In this work, we assume by convention that $D_{C}=D_{\mathcal S} \bigcup D_{\mathcal T}$.
The  empirical risk associated with the contrastive domain-adaptive loss is then given by 
\begin{equation}\label{eq_risk_adloss}
     \hat{\mac{R}}_{ad}(f,h)
     :=\hat{\mac{R}}_{cl}(f) +  \hat{\mac{R}}_{d}(f,h)
     = \frac{1}{n}\sum_{k = 1}^n \ell_{ad}(X_k, Z_k, f, h). 
 \end{equation}
Accordingly, the estimated representation function is defined as the solution to the following minimax problem:
\begin{equation}\label{eq_f}
       \hat{f} 
       \in \arg\min_{f \in \mac{F}} \max_{h \in \mac{H}} \hat{\mac{R}}_{ad}(f, h), 
\end{equation}
where $\mathcal F$ and $\mathcal H$ are the norm-constrained DNNs.
\subsection{Downstream tasks}\label{sec:ds}   
After obtaining the representation function $\hat{f}$  in \eqref{eq_f}, we proceed to learn a predictor for the downstream task. 
Owing to this 
domain-invariance property, the predictor can be effectively trained on the source domain, as   established in Lemma \ref{lem_risk4ST}.
In this work, we focus on binary classification as a representative example; however, the framework can be readily extended to other supervised learning tasks. 
Given $m$  labeled source samples
$\bar{D}_{\mac{S}} :=\{(X_j^{\mathcal{S}}, Y_j^{\mathcal{S}})\}_{j=1}^m$ with $Y_j^{\mathcal{S}} \in \{-1,1\}$,
the predictor $\hat{g}$ in \eqref{def_g} is trained under the empirical risk minimization (ERM) framework by combining the representation $\hat{f}$ with the margin-logistic loss:
\begin{equation}\label{def_g}
        \hat{g} \in \arg\min_{g \in \mac{G}} \frac{1}{m}\sum_{j =1 }^m 
         \ell(g\circ\hat{f}(X^\mac{S}_j), Y^\mac{S}_j),
\end{equation}   
where 
$\mac{G}$
is a class of DNNs mapping $\mab{R}^d$ to $\mab{R}$.
Here, $\ell(a,y):=\log(1+\exp(-ya))$ is the margin-logistic loss for a score $a\in\mathbb R$ and a signed label $y\in\{-1,1\}$. 
Since $|\partial \ell(a,y)/\partial a|=1/(1+\exp(ya))\leq 1$, the loss is $1$-Lipschitz in its score argument, uniformly over $y$. 
The complete procedure is the two-stage algorithm in Appendix \ref{sec:pro}: first estimate $\hat f$ using unlabeled source and target samples, 
then freeze $\hat f$ while fitting $\hat g$ on labeled source samples.
\section{Error Analysis}\label{sec_Error}
We bound the excess risk of the two-stage estimator $(\hat f,\hat g)$ through its representation and prediction errors. 
To connect source-supervised fitting with target prediction, we first consider the distributional conditions for 
transferring risk guarantees in the representation space.
\begin{lemma}\label{lem_risk4ST}
Let $f$ denote the common representation function for the source and target domains, and $\mac{M}_g$ be the H\"older class.
Let $U^\mac{S}:= f(X^\mac{S})$ and $U^\mac{T}:= f(X^\mac{T})$. Assume that $P_{S}(U^\mac{S})=P_{T}(U^\mac{T})$ and $P_{S}(Y^\mac{S}\mid U^\mac{S})=P_{T}(Y^\mac{T}\mid U^\mac{T})$.
Denote    
    \begin{equation*}
    \begin{split}
     g^{**} \in \arg\min_{g \in \mac{M}_g}\mab{E}_{P_{T}(U^\mac{T}, Y^\mac{T})}[\ell(g(U^\mac{T}), Y^\mac{T})],\quad
      g^{*} \in \arg\min_{g \in \mac{M}_g}\mab{E}_{P_{S}(U^\mac{S}, Y^\mac{S})}[\ell(g(U^\mac{S}), Y^\mac{S})],     
    \end{split}
    \end{equation*}
    then $g^{**} = g^*$.
\end{lemma}
When the source and target feature marginals and conditional label distributions coincide, 
their prediction risks agree for every predictor in the common representation space. 
This risk identity transfers source excess-risk guarantees to the target domain and provides the distributional 
basis for source-supervised downstream fitting.
We next introduce the notation used in the analysis. As defined in Section~\ref{sec:rp}, $P_{\mu}(X,Z)$ is the joint distribution of the mixture covariate $X\sim P_C$ and its domain label $Z$. Recalling the definitions of contrastive 
domain-adaptive loss $\ell_{ad}$ in  \eqref{clda_loss} and its empirical risk $\hat{\mac{R}}_{ad}(f, h)$ in \eqref{eq_risk_adloss}, we define the 
population-level 
contrastive domain-adaptive risk $\mathcal{R}_{ad}(f,h)$ as
\begin{equation*}
\begin{split}
\mac{R}_{ad}(f,h) := \mab{E}_{{P_{\mu},P_A }}[\ell_{ad}(X, Z, f, h)]= \mab{E}_{{P_{C}}\times P_A}[\ell_{cl}(X, f)] +  \mab{E}_{P_{\mu}}[\ell_{d}(X, Z, f, h)].    
\end{split}
\end{equation*}
Let $\mac{M}_f$ and $\mac{M}_h$ be the H\"older classes. We denote $\mac{R}_{cl}(f):=\mab{E}_{{P_{C}}\times P_A}[\ell_{cl}(X, f)]$ and $\mac{R}_{d}(f,h):=\mab{E}_{P_{\mu}}[\ell_{d}(X, Z, f, h)]$, and define the representation function $f^*$ at population level: 
\begin{equation*}
    f^* \in \arg\min_{f \in \mac{M}_f}\max_{h \in \mac{M}_h} \mac{R}_{ad}(f,h). 
\end{equation*}
To simplify the notations, we define
\begin{equation}\label{minmax_risk}
\begin{split}
\mac{R}_{\mac{M}_h}(f):= \sup_{h \in \mac{M}_h} \mac{R}_{ad}(f, h), \quad \mac{R}_{\mac{H}}(f):= \sup_{h \in \mac{H}} \mac{R}_{ad}(f, h), \quad
\hat{\mac{R}}_{\mac{H}}(f):=\sup_{h \in \mac{H}}\hat{\mac{R}}_{ad}(f,h).  
\end{split}
\end{equation}
Consequently, we have
\begin{equation*}
f^* \in \arg\min_{f \in \mac{M}_f} \mac{R}_{\mac{M}_h}(f), \quad \hat{f} \in \arg\min_{f \in \mac{F}} \hat{\mac{R}}_{\mac{H}}(f).
\end{equation*}
By Lemma \ref{lem_risk4ST}, we now consider the excess risk  on the source domain between $\mac{R}_{\hat{f}}(\hat{g})$ and $\mac{R}_{f^*}(g^*)$, i.e., 
\begin{align*}
\mac{E}_{\hat{f}}(\hat{g}):=\mac{R}_{\hat{f}}(\hat{g}) - \mac{R}_{f^*}(g^*) 
 =  \mab{E}_{P_{S}(X,Y)}[\ell(\hat{g}\circ \hat{f} (X), Y)] - \mab{E}_{P_{S}(X,Y)}[\ell(g^*\circ f^*(X), Y)].   
\end{align*}
To establish the error bound of $\mac{E}_{\hat{f}}(\hat{g})$, we require the following  regularity assumptions:
\begin{assumption}[Smoothness and bounded representations]\label{assum_holder}
The ideal classes are the H\"older balls of Definition~\ref{def_holderclass}, with fixed radii $B_f,B_h,B_g$ 
and smoothness $\beta_f,\beta_h,\beta_g>0$. The class $\mathcal M_f$ has $d$ outputs; $\mathcal M_h,\mathcal M_g$ are scalar-valued on $[0,1]^d$. Inputs and augmented views lie in $[0,1]^p$. Both $f^*$ and all $f\in\mathcal F$ take values in $[0,1]^d$, so $\|f(x)\|_2\leq R:=\sqrt d$, independently of norm budgets.
Write $\beta_j=\bar\beta_j+\tilde r_j$ with $\bar\beta_j=\lceil\beta_j\rceil-1$, $\tilde r_j\in(0,1]$, and put $\rho_h=\min\{\beta_h,1\}$.
\end{assumption}
The range restriction specifies the theoretical network class. Clipping gives bounded ReLU approximants 
with constant-factor norm overhead and one extra hidden layer (Lemma~\ref{lem_norm_constrain_appro}); 
this does not assert that experimental backbones implement clipping.
\begin{assumption}[Calibration]\label{assum_rep}
For the selected estimator $\hat f$, almost surely,
\[
\mathbb E_{P_C}\|\hat f(X)-f^*(X)\|_2^{\tilde r_g}
\leq C_{\rm cal}\{\mathcal R_{\mathcal M_h}(\hat f)-\mathcal R_{\mathcal M_h}(f^*)\},
\]
where $C_{\rm cal}$ is independent of sample sizes and network budgets and $\tilde r_g$ is as in Assumption~\ref{assum_holder}. 
The comparison uses a fixed representation-coordinate convention.
\end{assumption}
\begin{remark}
Assumption \ref{assum_holder} specifies H\"older regularity, a standard framework for nonparametric approximation \citep{gyorfi2002distribution, wasserman2006all}. 
The ideal functions need not belong to the chosen neural-network classes. Our analysis retains the resulting approximation errors and balances them against statistical complexity, making network capacity an explicit part of the convergence guarantee.
Assumption \ref{assum_rep} converts excess pretraining risk into representation-error control for the selected estimator, 
linking the contrastive--adversarial objective to the downstream prediction bound.
More generally, a calibration condition can be expressed as 
$\mathbb{E}_X G(f(X) - f^*(X)) \lesssim \mac{R}_{\mac{M}_h}(f) - \mac{R}_{\mac{M}_h}(f^*)$, where $G:\mathbb{R}^d \longrightarrow \mathbb{R}^+$.
This inequality asserts that the estimation error is controlled by the excess risk. In our setting, we take
$G(t)=\|t\|_2^{\tilde{r}_g}$, $t \in \mathbb{R}^d$.
Calibration conditions  are widely used in the theoretical analysis of quantile regression 
\citep{belloni2011,madrid2022risk} and classification \citep{bartlett2006convexity}. For example, 
in quantile regression one often considers
$
G(t)=\min\{|t|,|t|^2\}, t \in \mathbb{R}
$, 
while in classification, surrogate losses such as the logistic loss correspond to  $G(t)=|t|^2$,
and the hinge loss corresponds to $G(t)=|t|$, $t \in \mathbb{R}$.
\end{remark}
\subsection{Non-asymptotic error bounds }\label{sec_consis}
In this section, we present the 
non-asymptotic error bounds for the excess risk.
We define the hypothesis spaces $\mac{F}$, $\mac{H} $,
and $\mac{G}$ with the following characteristics:
\begin{enumerate}
\item[(a).] Let $\mac{F} := \mac{NN}(\mac{W}_f, \mac{D}_f, \mac{K}_f)$ be the norm-constrained DNNs with width
$\mathcal W_f\gtrsim d\mathcal K_f^{(2p+\beta_f)/(2p+2)}$, depth
$\mathcal D_f\geq2\lceil\log_2(p+\bar\beta_f)\rceil+3$, and norm constraint $\mac{K}_f \asymp n^{\beta_h(p+1)/(2\beta_{f}\beta_h +2\beta_f(d+1) +2\beta_h(p+1))}$,
\item[(b).] Let $\mac{H}:= \mac{NN}(\mac{W}_h, \mac{D}_h, \mac{K}_h)$
    be the norm-constrained DNNs with width
$\mathcal W_h\gtrsim\mathcal K_h^{(2d+\beta_h)/(2d+2)}$, depth
$\mathcal D_h\geq2\lceil\log_2(d+\bar\beta_h)\rceil+2$, and norm constraint $\mac{K}_h \asymp n^{\beta_f(d+1)/(2\beta_{f}\beta_h +2\beta_f(d+1) +2\beta_h(p+1))}$
\item[(c).] Let $\mac{G}:= \mac{NN}(\mac{W}_g, \mac{D}_g, \mac{S}_g)$ be ReLU DNNs with width $\mac{W}_g \geq m^{d/4(2\beta_{g} +d)}\log(m)$, 
depth $\mac{D}_g \geq m^{d/4(2\beta_{g} +d)}\log(m)$, and size  $\mac{S}_g \geq m^{3d/4(2\beta_{g} +d)}\log^3(m)$.
\end{enumerate}
\begin{theorem}\label{theorem_main}
Under Assumptions~\ref{assum_holder}--\ref{assum_rep} and network choices (a)--(c), 
the expected source and target excess risks coincide and satisfy
\begin{align}\label{eqeb}
    \begin{split}
    \mab{E}_{{D_{\bar{C}}},D_{J}, \bar{D}_{\mac{S}}}\big[\mac{E}_{\hat{f}}(\hat{g}) \big] \leq \mac{O}\left(
      {n}^{-\frac{\beta_f \beta_h}{2\beta_f\beta_h +2(p+1)\beta_h +2(d+1)\beta_f}}\log(n)\right) 
      +\mac{O}\left(m^{-\frac{\beta_g}{2\beta_{g} + d}}\right).    
    \end{split}
    \end{align}
Moreover, suppose $\beta_f = \beta_h = \beta_g = \beta$
and let
$n\geq m^{2(\beta+p+d+2)/(2\beta+d)}$, then 
\begin{equation*}
\begin{split}
  \mab{E}_{{D_{\bar{C}}},D_{J}, \bar{D}_{\mac{S}}}\big[\mac{E}_{\hat{f}}(\hat{g})\big] &\leq  \mac{O}\left(m^{-\frac{\beta}{2\beta + d}}\right).
\end{split}
\end{equation*}
\end{theorem}
\begin{remark}
Theorem~\ref{theorem_main} separates the joint cost of learning representations and domain alignment from the labeled-sample 
cost of downstream prediction. The coupled pretraining rate reflects the combined smoothness and capacity of 
the feature extractor and discriminator: domain discrimination acts on learned features, 
so its statistical complexity depends on both components. This quantifies their interaction within pretraining, 
while the frozen representation enables a separate analysis of predictor fitting.
The resulting bound assigns the ambient-dimensional learning cost to the unlabeled stage and the low-dimensional 
prediction cost to the labeled stage. In the theorem's large-unlabeled-sample regime with common smoothness $\beta\geq1$, 
the bound reduces to $\widetilde{\mathcal O}(m^{-\beta/(2\beta+d)})$, where the tilde suppresses logarithmic factors. 
Thus sufficiently accurate pretraining preserves a downstream label-sample exponent governed by $d$ rather than $p$, 
explicitly linking unlabeled representation learning to low-dimensional prediction efficiency.
\end{remark}
\begin{remark}
The analysis explicitly controls neural approximation errors, allowing the ideal representation, discriminator, 
and predictor to lie outside the chosen network classes. Increasing norm budgets improves the approximation bounds but enlarges 
the statistical bound; the coupled rate results from balancing these effects jointly for the feature extractor and discriminator. 
Under the maximum-row-sum product constraint of Definition~\ref{def_norm_network}, 
the pretraining complexity bound has no explicit width penalty at fixed norm budgets and depths. 
This provides a quantitative basis for coordinating approximation capacity and statistical complexity in contrastive--adversarial learning.
\end{remark}
\paragraph{Discussion.}
JCAR integrates joint-domain contrastive feature learning with explicit adversarial alignment, 
then freezes the representation for source-supervised prediction. Whereas \citet{shen2022connect} 
explain transfer through augmentation-graph connectivity for spectral contrastive learning and a linear head, 
our theory quantifies the estimation and neural-approximation costs of the contrastive--adversarial objective. 
Its distinctive contribution is to connect the coupled complexity of feature learning and domain discrimination 
to downstream label efficiency, with separate roles for unlabeled and labeled sample sizes. 
Detailed comparisons are provided in Appendix~\ref{sec:pro}.
\section{Numerical Experiments}\label{sec_Numer}
To verify the effectiveness of our proposed method, we conduct a series of experiments in this section.
Section \ref{setting} describes the datasets, model architectures and baselines.
Section \ref{evalue} reports from-scratch comparisons with ERM and established domain-adaptation baselines in 
Tables \ref{tab_compare_off31_noweight}--\ref{tab_compare_offHom_noweight}, together with qualitative t-SNE visualizations 
in Figure \ref{fig:four_tsn}. Pretrained comparisons are reported separately in Appendix \ref{add results}.
\subsection{Experimental setup}\label{setting}
We evaluate JCAR on three domain-adaptation benchmarks: Office-31 \citep{saenko2010adapting},
Office-Home \citep{venkateswara2017deep}, and DomainNet \citep{peng2019moment}, using ResNet \citep{he2016deep}
and ViT \citep{dosovitskiy2020image} backbones.
Baselines include DANN \citep{ganin2015unsupervised} and CDAN \citep{long2018conditional}.
Appendix \ref{sup_expe} provides details of the data augmentation, model architectures, and baselines.
\begin{table*}[ht]
\caption{Top-1 accuracy (\%) of the proposed method and baselines trained from scratch on Office-31.}
\label{tab_compare_off31_noweight}
\vskip 0.15in
\centering
\begin{small}
\begin{sc}
\begin{tabular}{lcccccc}
    \toprule
       Methods               &A $\to$ W      &D $\to$ W   &W $\to$ D  & A $\to$ D  &D $\to$ A  &W $\to$ A \\ 
       \midrule
       ERM       &10.1       &28.9      &31.3   & 9.2  & 6.1  & 6.0 \\
         CDAN       & 12.1      & 29.1     & 28.3  & 7.8  & 5.9  & 4.7\\
        DANN       & 16.4      & 63.4       & 67.5  & 10.2  &7.2   &9.6 \\ 
        Ours (ResNet)     & \textbf{16.7}     & 44.6     &51.4   &\textbf{12.1}  &\textbf{10.5}   & \textbf{13.0} \\ \midrule
        ViT-B      & 14.8      & 37.5     & 47.8  &12.5  & 5.9  &9.9 \\
        Ours (ViT)      &\textbf{17.9}      & \textbf{39.2}     &42.9     &11.0   & \textbf{7.1}    &\textbf{11.2} \\
    \bottomrule    
    \end{tabular}
\end{sc}
\end{small}
\vskip -0.1in
\end{table*}
\begin{table*}[h]
\caption{Top-1 accuracy (\%) of the proposed method and baselines trained from scratch on Office-Home.}
\label{tab_compare_offHom_noweight}
\begin{center}
\begin{small}
\begin{tabular}{lcccccc}
    \toprule
      Methods               &Ar $\to$ Rw      &Pr $\to$ Rw   &Cl $\to$ Pr  & Ar $\to$ Cl  &Pr $\to$ Ar  &Cl $\to$ Rw \\ \midrule
      ERM       &10.0       &16.4      &10.9   &4.4   &7.1   &9.5 \\
         CDAN      &9.4       & 16.0     & 12.5  & 3.9  & 5.0  &8.4 \\
        DANN       &9.6       & 16.4     & 10.3  & 4.9  &4.4   &10.2 \\ 
        Ours (ResNet)    & 7.6      &10.4      &10.0  &4.3   & 5.0  & 7.2\\ \midrule
        ViT-B    & 9.2      & 14.3     &3.1   & 5.8  & 7.0  & 9.1 \\
         Ours (ViT)      & \textbf{9.4}     &\textbf{14.5}      &\textbf{8.2}   &4.5  & \textbf{7.8}  & \textbf{10.5} \\
    \bottomrule    
    \end{tabular}
\end{small}
\end{center}
\vskip -0.1in
\end{table*}
\subsection{Evaluation}\label{evalue}
Tables \ref{tab_compare_off31_noweight} and \ref{tab_compare_offHom_noweight} compare randomly initialized models, 
with ResNet-based and ViT-based methods grouped separately. JCAR delivers clear gains in several challenging from-scratch transfers. 
On Office-31, our ResNet variant obtains the best result in its architecture group on four of six transfers, including gains of $1.9$, $3.3$, 
and $3.4$ percentage points over the strongest competing ResNet result on A$\to$D, D$\to$A, and W$\to$A, respectively. 
Our ViT variant improves over ViT-B on four of six Office-31 transfers. On Office-Home, 
it improves over ViT-B on five of six transfers, with a $5.1$-point gain on Cl$\to$Pr. 
These results demonstrate the value of joint-domain contrastive-adversarial pretraining, 
especially for the ViT backbone. The ResNet variant is less effective on Office-Home and on the D$\to$W and W$\to$D Office-31 transfers, 
showing that the gains depend on the backbone and domain pair. 
Tables \ref{tab_compare_off31}--\ref{tab_compare_domain} in the Appendix provide the complementary pretrained evaluation.
Figure \ref{fig:four_tsn} shows t-SNE \citep{maaten2008visualizing} projections generated with scikit-learn \citep{pedregosa2011scikit}; 
green and purple denote source and target samples. For the displayed Office-31 transfer, 
JCAR produces substantial source--target overlap together with a broadly distributed feature geometry, 
providing qualitative evidence that the learned representation combines domain alignment with non-collapsed feature structure. 
Since t-SNE is a nonlinear projection, we use it as supporting visualization rather than a standalone quantitative measure; 
additional transfers are shown in Appendix \ref{add results}.
\begin{figure}[H]
    \centering
    \begin{subfigure}[b]{0.235\textwidth}
        \centering
        \includegraphics[width=\linewidth]{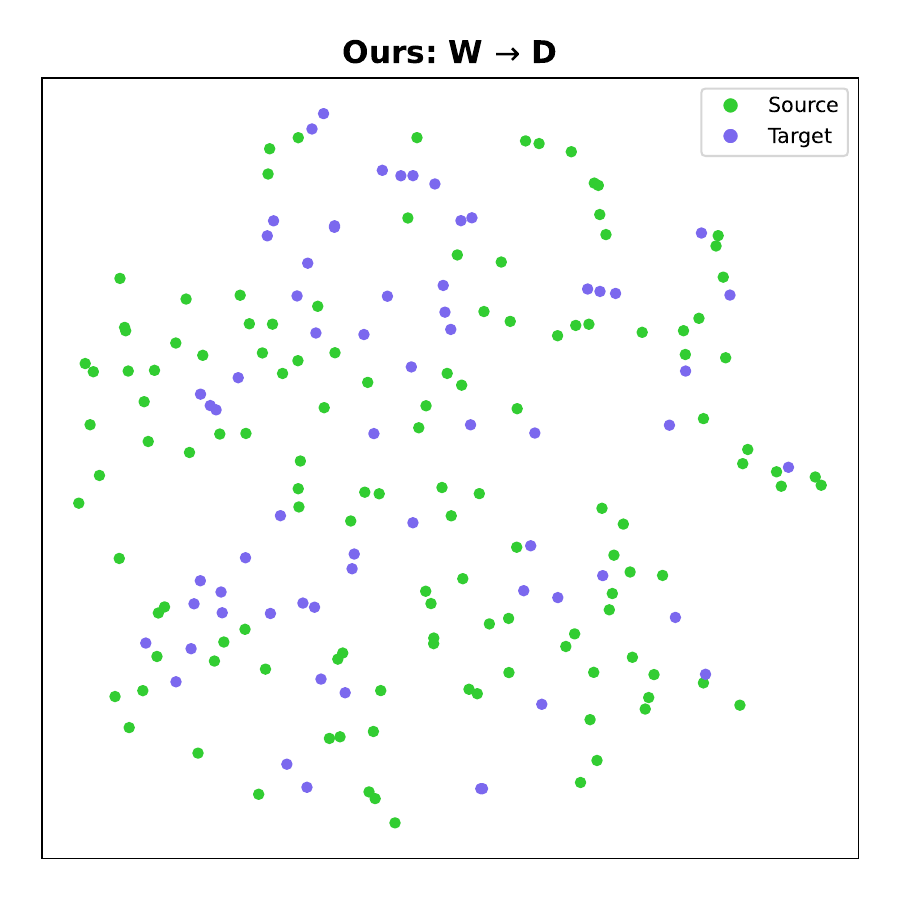}
        \caption{Our method}
        \label{fig:sub1}
    \end{subfigure}
    \hfill
    \begin{subfigure}[b]{0.235\textwidth}
        \centering
        \includegraphics[width=\linewidth]{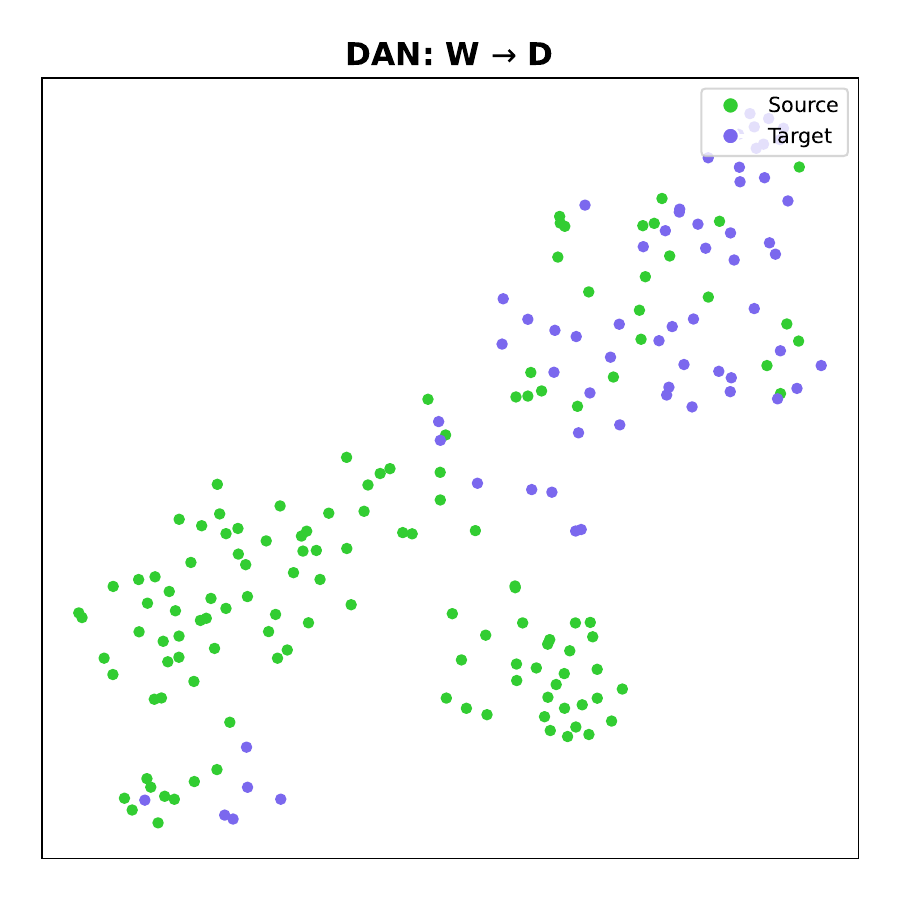}
        \caption{DAN}
        \label{fig:sub2}
    \end{subfigure}
    \hfill
    \begin{subfigure}[b]{0.235\textwidth}
        \centering
        \includegraphics[width=\linewidth]{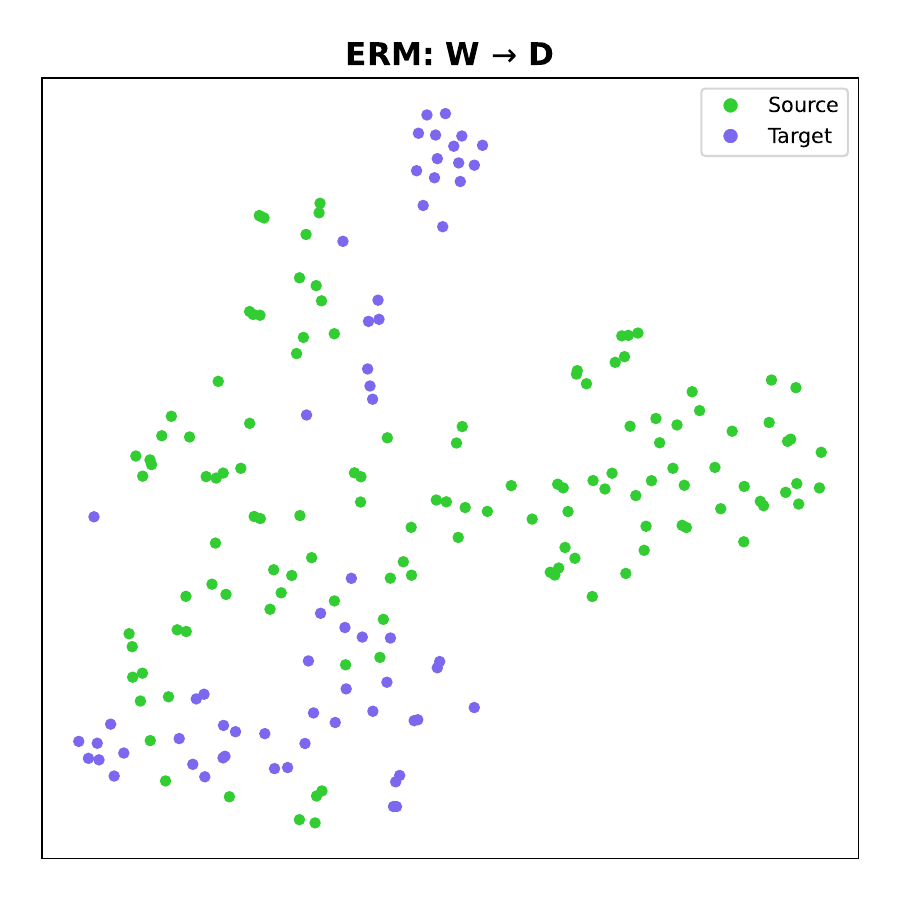}
        \caption{ERM}
        \label{fig:sub3}
    \end{subfigure}
    \hfill
    \begin{subfigure}[b]{0.235\textwidth}
        \centering
        \includegraphics[width=\linewidth]{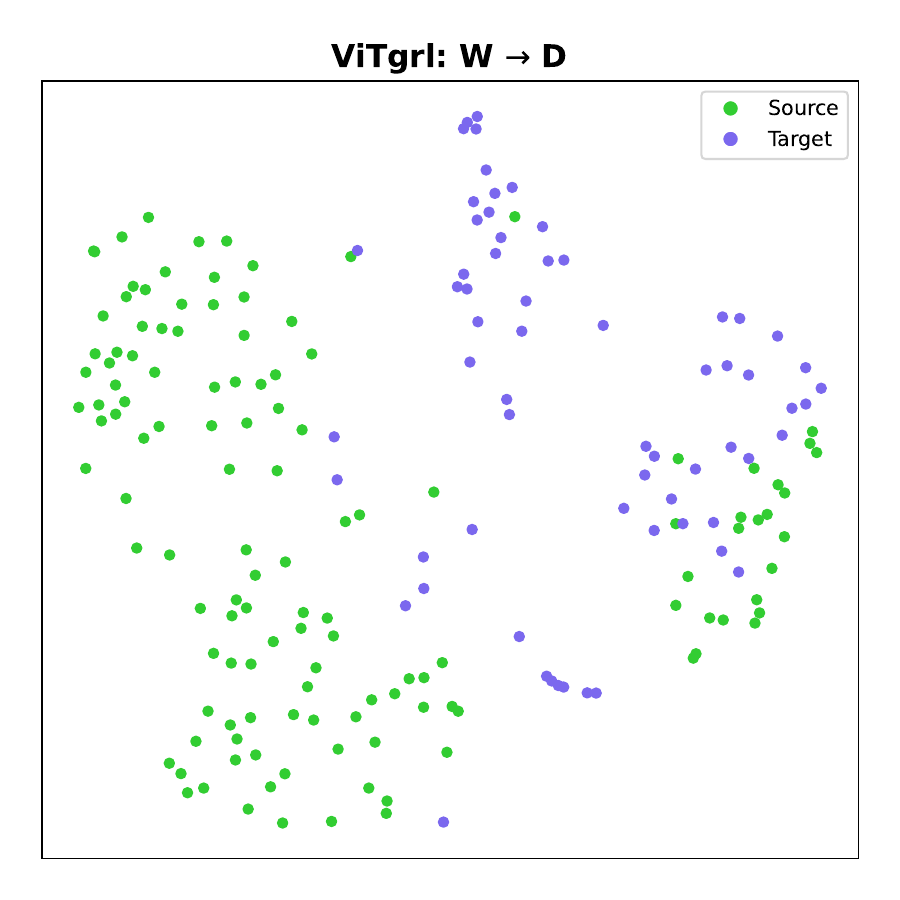}
        \caption{ViT}
        \label{fig:sub4}
    \end{subfigure}
    \caption{Two-dimensional t-SNE projections for the Office-31 W$\to$D task. 
    Source and target samples are shown in green and purple, respectively. 
    Compared with DAN (\ref{fig:sub2}), ERM (\ref{fig:sub3}), and ViT (\ref{fig:sub4}), JCAR (\ref{fig:sub1}) 
    exhibits substantial cross-domain overlap while maintaining a broadly distributed feature geometry.}
    \label{fig:four_tsn}
\end{figure}
\section{Conclusion}\label{sec_Conclu}
We introduced JCAR, a two-stage UDA framework combining joint-domain contrastive learning with adversarial alignment, 
followed by source-supervised prediction on frozen features. Our non-asymptotic analysis balances neural approximation and statistical errors, 
captures the coupled cost of feature learning and domain discrimination, and separates it from downstream prediction. 
Under the stated assumptions, sufficiently accurate unlabeled pretraining yields a downstream polynomial 
rate governed by the latent dimension rather than the ambient dimension. 
This establishes a quantitative link between contrastive--adversarial pretraining and label-sample efficiency. 
The experiments substantiate the practical value of this design, with strong from-scratch 
improvements across multiple Office-31 and Office-Home transfers, particularly for the ViT backbone. 
The current results also identify productive directions for extension: domain-specific augmentation design, 
stronger compatibility with supervised initialization, and component-wise ablations of contrastive learning, 
adversarial alignment, and stage-wise freezing.
\subsection*{AI use statement}
Generative AI tools assisted with \LaTeX{} migration and formatting, language
editing, and checking and refining mathematical derivations and proofs. The
authors independently reviewed and verified all AI-assisted material,
including equations, proofs, citations, and experimental descriptions. The
authors take responsibility for the final content of this work, including
text, claims, and artifacts produced with the aid of generative AI.
\bibliographystyle{iclr2027_conference}
\bibliography{refer}

@article{jiao2021deep,
  title={Deep nonparametric regression on approximate manifolds: Nonasymptotic error bounds with polynomial prefactors},
  author={Jiao, Yuling and Shen, Guohao and Lin, Yuanyuan and Huang, Jian},
  journal={The Annals of Statistics},
  volume={51},
  number={2},
  pages={691--716},
  year={2023},
  publisher={Institute of Mathematical Statistics}
}

@book{anthony2009neural,
  title={Neural network learning: Theoretical foundations},
  author={Anthony, Martin and Bartlett, Peter L},
  year={2009},
  publisher={cambridge university press}
}

@book{shalev2014understanding,
  title={Understanding machine learning: From theory to algorithms},
  author={Shalev-Shwartz, Shai and Ben-David, Shai},
  year={2014},
  publisher={Cambridge university press}
}

@article{dudley1967sizes,
  title={The sizes of compact subsets of Hilbert space and continuity of Gaussian processes},
  author={Dudley, Richard M},
  journal={Journal of Functional Analysis},
  volume={1},
  number={3},
  pages={290--330},
  year={1967},
  publisher={Elsevier}
}

@article{bartlett2019nearly,
  title={Nearly-tight VC-dimension and pseudodimension bounds for piecewise linear neural networks},
  author={Bartlett, Peter L and Harvey, Nick and Liaw, Christopher and Mehrabian, Abbas},
  journal={The Journal of Machine Learning Research},
  volume={20},
  number={1},
  pages={2285--2301},
  year={2019},
  publisher={JMLR. org}
}

@article{jiao2023approximation,
  title={Approximation bounds for norm constrained neural networks with applications to regression and GANs},
  author={Jiao, Yuling and Wang, Yang and Yang, Yunfei},
  journal={Applied and Computational Harmonic Analysis},
  volume={65},
  pages={249--278},
  year={2023},
  publisher={Elsevier}
}

@article{yosinski2014transferable,
  title={How transferable are features in deep neural networks?},
  author={Yosinski, Jason and Clune, Jeff and Bengio, Yoshua and Lipson, Hod},
  journal={Advances in neural information processing systems},
  volume={27},
  year={2014}
}

@inproceedings{glorot2011domain,
  title={Domain adaptation for large-scale sentiment classification: A deep learning approach},
  author={Glorot, Xavier and Bordes, Antoine and Bengio, Yoshua},
  booktitle={Proceedings of the 28th international conference on machine learning (ICML-11)},
  pages={513--520},
  year={2011}
}

@article{mansour2009domain,
  title={Domain adaptation: Learning bounds and algorithms},
  author={Mansour, Yishay and Mohri, Mehryar and Rostamizadeh, Afshin},
  journal={arXiv preprint arXiv:0902.3430},
  year={2009}
}

@article{ajakan2014domain,
  title={Domain-adversarial neural networks},
  author={Ajakan, Hana and Germain, Pascal and Larochelle, Hugo and Laviolette, Fran{\c{c}}ois and Marchand, Mario},
  journal={arXiv preprint arXiv:1412.4446},
  year={2014}
}

@inproceedings{ghifary2014domain,
  title={Domain adaptive neural networks for object recognition},
  author={Ghifary, Muhammad and Kleijn, W Bastiaan and Zhang, Mengjie},
  booktitle={PRICAI 2014: Trends in Artificial Intelligence: 13th Pacific Rim International Conference on Artificial Intelligence, Gold Coast, QLD, Australia, December 1-5, 2014. Proceedings 13},
  pages={898--904},
  year={2014},
  organization={Springer}
}

@article{tzeng2014deep,
  title={Deep domain confusion: Maximizing for domain invariance},
  author={Tzeng, Eric and Hoffman, Judy and Zhang, Ning and Saenko, Kate and Darrell, Trevor},
  journal={arXiv preprint arXiv:1412.3474},
  year={2014}
}

@article{ben2010theory,
  title={A theory of learning from different domains},
  author={Ben-David, Shai and Blitzer, John and Crammer, Koby and Kulesza, Alex and Pereira, Fernando and Vaughan, Jennifer Wortman},
  journal={Machine learning},
  volume={79},
  pages={151--175},
  year={2010},
  publisher={Springer}
}

@article{pan2010domain,
  title={Domain adaptation via transfer component analysis},
  author={Pan, Sinno Jialin and Tsang, Ivor W and Kwok, James T and Yang, Qiang},
  journal={IEEE transactions on neural networks},
  volume={22},
  number={2},
  pages={199--210},
  year={2010},
  publisher={IEEE}
}

@inproceedings{zhang2013domain,
  title={Domain adaptation under target and conditional shift},
  author={Zhang, Kun and Sch{\"o}lkopf, Bernhard and Muandet, Krikamol and Wang, Zhikun},
  booktitle={International conference on machine learning},
  pages={819--827},
  year={2013},
  organization={Pmlr}
}

@article{wang2014flexible,
  title={Flexible transfer learning under support and model shift},
  author={Wang, Xuezhi and Schneider, Jeff},
  journal={Advances in Neural Information Processing Systems},
  volume={27},
  year={2014}
}

@inproceedings{ganin2015unsupervised,
  title={Unsupervised domain adaptation by backpropagation},
  author={Ganin, Yaroslav and Lempitsky, Victor},
  booktitle={International conference on machine learning},
  pages={1180--1189},
  year={2015},
  organization={PMLR}
}

@article{ganin2016domain,
  title={Domain-adversarial training of neural networks},
  author={Ganin, Yaroslav and Ustinova, Evgeniya and Ajakan, Hana and Germain, Pascal and Larochelle, Hugo and Laviolette, Fran{\c{c}}ois and March, Mario and Lempitsky, Victor},
  journal={Journal of machine learning research},
  volume={17},
  number={59},
  pages={1--35},
  year={2016}
}

@inproceedings{tzeng2015simultaneous,
  title={Simultaneous deep transfer across domains and tasks},
  author={Tzeng, Eric and Hoffman, Judy and Darrell, Trevor and Saenko, Kate},
  booktitle={Proceedings of the IEEE international conference on computer vision},
  pages={4068--4076},
  year={2015}
}

@inproceedings{tzeng2017adversarial,
  title={Adversarial discriminative domain adaptation},
  author={Tzeng, Eric and Hoffman, Judy and Saenko, Kate and Darrell, Trevor},
  booktitle={Proceedings of the IEEE conference on computer vision and pattern recognition},
  pages={7167--7176},
  year={2017}
}

@InProceedings{pmlr-v70-arora17a,
  title = 	 {Generalization and Equilibrium in Generative Adversarial Nets ({GAN}s)},
  author =       {Sanjeev Arora and Rong Ge and Yingyu Liang and Tengyu Ma and Yi Zhang},
  booktitle = 	 {Proceedings of the 34th International Conference on Machine Learning},
  pages = 	 {224--232},
  year = 	 {2017},
  editor = 	 {Precup, Doina and Teh, Yee Whye},
  volume = 	 {70},
  series = 	 {Proceedings of Machine Learning Research},
  month = 	 {06--11 Aug},
  publisher =    {PMLR},
  url = 	 {https://proceedings.mlr.press/v70/arora17a.html},
}

@article{HoffmanWYD16,
  author       = {Judy Hoffman and
                  Dequan Wang and
                  Fisher Yu and
                  Trevor Darrell},
  title        = {FCNs in the Wild: Pixel-level Adversarial and Constraint-based Adaptation},
  journal      = {CoRR},
  volume       = {abs/1612.02649},
  year         = {2016},
  url          = {http://arxiv.org/abs/1612.02649},
  eprinttype    = {arXiv},
  eprint       = {1612.02649},

}

@InProceedings{Chen_2017_ICCV,
author = {Chen, Yi-Hsin and Chen, Wei-Yu and Chen, Yu-Ting and Tsai, Bo-Cheng and Frank Wang, Yu-Chiang and Sun, Min},
title = {No More Discrimination: Cross City Adaptation of Road Scene Segmenters},
booktitle = {Proceedings of the IEEE International Conference on Computer Vision (ICCV)},
month = {Oct},
year = {2017}
}

@inproceedings{Song2009,
author = {Song, Le and Huang, Jonathan and Smola, Alex and Fukumizu, Kenji},
title = {Hilbert space embeddings of conditional distributions with applications to dynamical systems},
year = {2009},
isbn = {9781605585161},
publisher = {Association for Computing Machinery},
address = {New York, NY, USA},
url = {https://doi.org/10.1145/1553374.1553497},
doi = {10.1145/1553374.1553497},
booktitle = {Proceedings of the 26th Annual International Conference on Machine Learning},
pages = {961–968},
numpages = {8},
location = {Montreal, Quebec, Canada},
series = {ICML '09}
}

@inproceedings{zhao2019learning,
  title={On learning invariant representations for domain adaptation},
  author={Zhao, Han and Des Combes, Remi Tachet and Zhang, Kun and Gordon, Geoffrey},
  booktitle={International conference on machine learning},
  pages={7523--7532},
  year={2019},
  organization={PMLR}
}

@inproceedings{long2015learning,
  title={Learning transferable features with deep adaptation networks},
  author={Long, Mingsheng and Cao, Yue and Wang, Jianmin and Jordan, Michael},
  booktitle={International conference on machine learning},
  pages={97--105},
  year={2015},
  organization={PMLR}
}

@inproceedings{long2017deep,
  title={Deep transfer learning with joint adaptation networks},
  author={Long, Mingsheng and Zhu, Han and Wang, Jianmin and Jordan, Michael I},
  booktitle={International conference on machine learning},
  pages={2208--2217},
  year={2017},
  organization={PMLR}
}

@article{long2018conditional,
  title={Conditional adversarial domain adaptation},
  author={Long, Mingsheng and Cao, Zhangjie and Wang, Jianmin and Jordan, Michael I},
  journal={Advances in neural information processing systems},
  volume={31},
  year={2018}
}

@article{long2016unsupervised,
  title={Unsupervised domain adaptation with residual transfer networks},
  author={Long, Mingsheng and Zhu, Han and Wang, Jianmin and Jordan, Michael I},
  journal={Advances in neural information processing systems},
  volume={29},
  year={2016}
}

@ARTICLE{Long2016Regularization,
  author={Long, Mingsheng and Wang, Jianmin and Ding, Guiguang and Pan, Sinno Jialin and Yu, Philip S.},
  journal={IEEE Transactions on Knowledge and Data Engineering}, 
  title={Adaptation Regularization: A General Framework for Transfer Learning}, 
  year={2014},
  volume={26},
  number={5},
  pages={1076-1089},
  doi={10.1109/TKDE.2013.111}}

@inproceedings{chen2019transferability,
  title={Transferability vs. discriminability: Batch spectral penalization for adversarial domain adaptation},
  author={Chen, Xinyang and Wang, Sinan and Long, Mingsheng and Wang, Jianmin},
  booktitle={International conference on machine learning},
  pages={1081--1090},
  year={2019},
  organization={PMLR}
}

@inproceedings{cao2018partial,
  title={Partial transfer learning with selective adversarial networks},
  author={Cao, Zhangjie and Long, Mingsheng and Wang, Jianmin and Jordan, Michael I},
  booktitle={Proceedings of the IEEE conference on computer vision and pattern recognition},
  pages={2724--2732},
  year={2018}
}

@inproceedings{cao2019learning,
  title={Learning to transfer examples for partial domain adaptation},
  author={Cao, Zhangjie and You, Kaichao and Long, Mingsheng and Wang, Jianmin and Yang, Qiang},
  booktitle={Proceedings of the IEEE/CVF conference on computer vision and pattern recognition},
  pages={2985--2994},
  year={2019}
}

@inproceedings{liu2019transferable,
  title={Transferable adversarial training: A general approach to adapting deep classifiers},
  author={Liu, Hong and Long, Mingsheng and Wang, Jianmin and Jordan, Michael},
  booktitle={International conference on machine learning},
  pages={4013--4022},
  year={2019},
  organization={PMLR}
}

@article{long2016deep,
  title={Deep learning of transferable representation for scalable domain adaptation},
  author={Long, Mingsheng and Wang, Jianmin and Cao, Yue and Sun, Jiaguang and Yu, Philip S},
  journal={IEEE Transactions on Knowledge and Data Engineering},
  volume={28},
  number={8},
  pages={2027--2040},
  year={2016},
  publisher={IEEE}
}

@inproceedings{yao2020unsupervised,
  title={Unsupervised transfer learning for spatiotemporal predictive networks},
  author={Yao, Zhiyu and Wang, Yunbo and Long, Mingsheng and Wang, Jianmin},
  booktitle={International Conference on Machine Learning},
  pages={10778--10788},
  year={2020},
  organization={PMLR}
}

@book{wainwright2019high,
  title={High-dimensional statistics: A non-asymptotic viewpoint},
  author={Wainwright, Martin J},
  volume={48},
  year={2019},
  publisher={Cambridge university press}
}

@article{bengio2013representation,
  title={Representation learning: A review and new perspectives},
  author={Bengio, Yoshua and Courville, Aaron and Vincent, Pascal},
  journal={IEEE transactions on pattern analysis and machine intelligence},
  volume={35},
  number={8},
  pages={1798--1828},
  year={2013},
  publisher={IEEE}
}

@article{krizhevsky2012imagenet,
  title={Imagenet classification with deep convolutional neural networks},
  author={Krizhevsky, Alex and Sutskever, Ilya and Hinton, Geoffrey E},
  journal={Advances in neural information processing systems},
  volume={25},
  year={2012}
}

@inproceedings{he2016deep,
  title={Deep residual learning for image recognition},
  author={He, Kaiming and Zhang, Xiangyu and Ren, Shaoqing and Sun, Jian},
  booktitle={Proceedings of the IEEE conference on computer vision and pattern recognition},
  pages={770--778},
  year={2016}
}

@article{sugiyama2007direct,
  title={Direct importance estimation with model selection and its application to covariate shift adaptation},
  author={Sugiyama, Masashi and Nakajima, Shinichi and Kashima, Hisashi and Buenau, Paul and Kawanabe, Motoaki},
  journal={Advances in neural information processing systems},
  volume={20},
  year={2007}
}

@article{huang2006correcting,
  title={Correcting sample selection bias by unlabeled data},
  author={Huang, Jiayuan and Gretton, Arthur and Borgwardt, Karsten and Sch{\"o}lkopf, Bernhard and Smola, Alex},
  journal={Advances in neural information processing systems},
  volume={19},
  year={2006}
}

@ARTICLE{pan2010,
  author={Pan, Sinno Jialin and Yang, Qiang},
  journal={IEEE Transactions on Knowledge and Data Engineering}, 
  title={A Survey on Transfer Learning}, 
  year={2010},
  volume={22},
  number={10},
  pages={1345-1359},
  doi={10.1109/TKDE.2009.191}}

@article{sejdinovic2013equivalence,
  title={Equivalence of distance-based and RKHS-based statistics in hypothesis testing},
  author={Sejdinovic, Dino and Sriperumbudur, Bharath and Gretton, Arthur and Fukumizu, Kenji},
  journal={The annals of statistics},
  pages={2263--2291},
  year={2013},
  publisher={JSTOR}
}

@article{gretton2012optimal,
  title={Optimal kernel choice for large-scale two-sample tests},
  author={Gretton, Arthur and Sejdinovic, Dino and Strathmann, Heiko and Balakrishnan, Sivaraman and Pontil, Massimiliano and Fukumizu, Kenji and Sriperumbudur, Bharath K},
  journal={Advances in neural information processing systems},
  volume={25},
  year={2012}
}

@inproceedings{kato2021non,
  title={Non-negative bregman divergence minimization for deep direct density ratio estimation},
  author={Kato, Masahiro and Teshima, Takeshi},
  booktitle={International Conference on Machine Learning},
  pages={5320--5333},
  year={2021},
  organization={PMLR}
}

@inproceedings{chen2020simple,
  title={A simple framework for contrastive learning of visual representations},
  author={Chen, Ting and Kornblith, Simon and Norouzi, Mohammad and Hinton, Geoffrey},
  booktitle={International conference on machine learning},
  pages={1597--1607},
  year={2020},
  organization={PMLR}
}

@book{quinonero2022dataset,
  title={Dataset shift in machine learning},
  author={Qui{\~n}onero-Candela, Joaquin and Sugiyama, Masashi and Schwaighofer, Anton and Lawrence, Neil D},
  year={2022},
  publisher={Mit Press}
}

@inproceedings{yu2012analysis,
  title={Analysis of kernel mean matching under covariate shift},
  author={Yu, Yao-Liang and Szepesv{\'a}ri, Csaba},
  booktitle={Proceedings of the 29th International Coference on International Conference on Machine Learning},
  pages={1147--1154},
  year={2012}
}

@article{sugiyama2008direct,
  title={Direct importance estimation for covariate shift adaptation},
  author={Sugiyama, Masashi and Suzuki, Taiji and Nakajima, Shinichi and Kashima, Hisashi and Von B{\"u}nau, Paul and Kawanabe, Motoaki},
  journal={Annals of the Institute of Statistical Mathematics},
  volume={60},
  pages={699--746},
  year={2008},
  publisher={Springer}
}

@article{golowich2020size,
  title={Size-independent sample complexity of neural networks},
  author={Golowich, Noah and Rakhlin, Alexander and Shamir, Ohad},
  journal={Information and Inference: A Journal of the IMA},
  volume={9},
  number={2},
  pages={473--504},
  year={2020},
  publisher={Oxford University Press}
}

@inproceedings{maurer2016vector,
  title={A vector-contraction inequality for rademacher complexities},
  author={Maurer, Andreas},
  booktitle={Algorithmic Learning Theory: 27th International Conference, ALT 2016, Bari, Italy, October 19-21, 2016, Proceedings 27},
  pages={3--17},
  year={2016},
  organization={Springer}
}

@inproceedings{saenko2010adapting,
  title={Adapting visual category models to new domains},
  author={Saenko, Kate and Kulis, Brian and Fritz, Mario and Darrell, Trevor},
  booktitle={Computer vision--ECCV 2010: 11th European conference on computer vision, Heraklion, Crete, Greece, September 5-11, 2010, proceedings, part iV 11},
  pages={213--226},
  year={2010},
  organization={Springer}
}

@inproceedings{venkateswara2017deep,
  title={Deep hashing network for unsupervised domain adaptation},
  author={Venkateswara, Hemanth and Eusebio, Jose and Chakraborty, Shayok and Panchanathan, Sethuraman},
  booktitle={Proceedings of the IEEE conference on computer vision and pattern recognition},
  pages={5018--5027},
  year={2017}
}

@inproceedings{peng2019moment,
  title={Moment matching for multi-source domain adaptation},
  author={Peng, Xingchao and Bai, Qinxun and Xia, Xide and Huang, Zijun and Saenko, Kate and Wang, Bo},
  booktitle={Proceedings of the IEEE/CVF international conference on computer vision},
  pages={1406--1415},
  year={2019}
}

@article{dosovitskiy2020image,
  title={An image is worth 16x16 words: Transformers for image recognition at scale},
  author={Dosovitskiy, Alexey and Beyer, Lucas and Kolesnikov, Alexander and Weissenborn, Dirk and Zhai, Xiaohua and Unterthiner, Thomas and Dehghani, Mostafa and Minderer, Matthias and Heigold, Georg and Gelly, Sylvain and others},
  journal={arXiv preprint arXiv:2010.11929},
  year={2020}
}

@misc{tllib,
    author = {Junguang Jiang and Baixu Chen and Bo Fu and Mingsheng Long},
    title = {Transfer-Learning-library},
    year = {2020},
    publisher = {GitHub},
    journal = {GitHub repository},
    howpublished = {\url{https://github.com/thuml/Transfer-Learning-Library}},
}

@article{bartlett2002rademacher,
  title={Rademacher and gaussian complexities: Risk bounds and structural results},
  author={Bartlett, Peter L and Mendelson, Shahar},
  journal={Journal of Machine Learning Research},
  volume={3},
  number={Nov},
  pages={463--482},
  year={2002}
}

@book{wasserman2006all,
  title={All of nonparametric statistics},
  author={Wasserman, Larry},
  year={2006},
  publisher={Springer Science \& Business Media}
}

@book{gyorfi2002distribution,
  title={A distribution-free theory of nonparametric regression},
  author={Gy{\"o}rfi, L{\'a}szl{\'o} and Kohler, Michael and Krzyzak, Adam and Walk, Harro and others},
  volume={1},
  year={2002},
  publisher={Springer}
}

@inproceedings{shen2022connect,
  title={Connect, not collapse: Explaining contrastive learning for unsupervised domain adaptation},
  author={Shen, Kendrick and Jones, Robbie M and Kumar, Ananya and Xie, Sang Michael and HaoChen, Jeff Z and Ma, Tengyu and Liang, Percy},
  booktitle={International conference on machine learning},
  pages={19847--19878},
  year={2022},
  organization={PMLR}
}

@book{ledoux2013probability,
  title={Probability in Banach Spaces: isoperimetry and processes},
  author={Ledoux, Michel and Talagrand, Michel},
  year={2013},
  publisher={Springer Science \& Business Media}
}

@article{maaten2008visualizing,
  title={Visualizing data using t-SNE},
  author={Maaten, Laurens van der and Hinton, Geoffrey},
  journal={Journal of machine learning research},
  volume={9},
  number={Nov},
  pages={2579--2605},
  year={2008}
}

@article{pedregosa2011scikit,
  title={Scikit-learn: Machine learning in Python},
  author={Pedregosa, Fabian and Varoquaux, Ga{\"e}l and Gramfort, Alexandre and Michel, Vincent and Thirion, Bertrand and Grisel, Olivier and Blondel, Mathieu and Prettenhofer, Peter and Weiss, Ron and Dubourg, Vincent and others},
  journal={the Journal of machine Learning research},
  volume={12},
  pages={2825--2830},
  year={2011},
  publisher={JMLR. org}
}

@article{madrid2022risk,
  title={Risk bounds for quantile trend filtering},
  author={Madrid Padilla, Oscar Hernan and Chatterjee, Sabyasachi},
  journal={Biometrika},
  volume={109},
  number={3},
  pages={751--768},
  year={2022},
  publisher={Oxford University Press}
}

@article{belloni2011,
  title={$L_1$-penalized quantile regression in high-dimensional sparse models},
  author={Belloni, Alexandre and Chernozhukov, Victor},
  journal={The Annals of Statistics},
  pages={82--130},
  year={2011},
  publisher={JSTOR}
}

@article{bartlett2006convexity,
  title={Convexity, classification, and risk bounds},
  author={Bartlett, Peter L and Jordan, Michael I and McAuliffe, Jon D},
  journal={Journal of the American Statistical Association},
  volume={101},
  number={473},
  pages={138--156},
  year={2006},
  publisher={Taylor \& Francis}
}
\appendix
\section*{Appendix}
Implementation and methodological comparisons appear first in Appendix~\ref{sec:pro}, followed by broader context and technical definitions in Appendices~\ref{related_works} and~\ref{definitions}. Appendices~\ref{sec:ps}--\ref{prof_pat3} develop the proofs, Appendix~\ref{Au_results} collects auxiliary results, and Appendices~\ref{sup_expe}--\ref{add results} supply the experimental details and additional results.
\section{The proposed algorithm}\label{sec:pro}
\begin{algorithm}[H]
    \caption{Domain adaptation with common representation}\label{alg:al1}
    \begin{algorithmic}[1]
        \REQUIRE {
   Unlabeled  joint dataset from   source and target domains: $D_C=\{X_k\}_{k=1}^{n}$, the augmentation dataset: $D_{\bar{C}} = \{(X_k, X_k^+, X_k^-)\}_{k =1}^n$, and domain  tags $D_{Z}=\{Z_k\}_{k = 1}^n $.  Labeled source dataset $\bar{D}_{\mac{S}}=\{(X_{j}^\mac{S}, Y_{j}^\mac{S})\}_{j=1}^m$.}
        \STATE 
{\bf 
Obtain representation function  $\hat{f}$:
}
Learn the representation function by minimizing the contrastive domain-adaptive loss, i.e., 
\begin{equation*}
       \hat{f}
       \in \arg\min_{f \in \mac{F}} \max_{h \in \mac{H}} \hat{\mac{R}}_{ad}(f, h).
\end{equation*}
\STATE Freeze the learned representation $\hat f$; its parameters are not updated during downstream predictor training.
\STATE 
{\bf
Obtain predictor $\hat{g}$ for downstream task: 
}  
Learn the predictor  through  ERM using labeled source-domain data $\bar{D}_{\mac{S}}$, i.e.,
\begin{equation*}
    \hat{g} = \arg\min_{g \in \mac{G}} \frac{1}{m}\sum_{j =1 }^m \ell(g\circ\hat{f}(X^\mac{S}_j), Y^\mac{S}_j). 
\end{equation*}
\RETURN $\hat{g}$ and the frozen representation $\hat{f}$.
\end{algorithmic}    
\end{algorithm}
\subsection{Novelty and Key Distinctions from Prior Frameworks}
Subsection~\ref{math_for} places the learning objectives side by side. Subsections~\ref{dis_tra} and~\ref{dis_shen} then contrast JCAR with conventional UDA and the framework of \citet{shen2022connect}, respectively.
\subsubsection{Unified Loss Function Formulations}\label{math_for}
Let $G_f(\cdot;\theta_f)$ denote the feature extractor, $G_y(\cdot;\theta_y)$ the downstream predictor, and $G_d(\cdot;\theta_d)$ the domain discriminator; $G_f$ produces $D$-dimensional representations. Let $\{(x_i^s,y_i^s)\}_{i=1}^n$ be the labeled source data and $\{x_i^t\}_{i=1}^N$ the unlabeled target data. We use $\mathcal L_y$, $\mathcal L_d$, and $\mathcal L_{\mathrm{infoNCE}}$ for the supervised, domain, and contrastive losses, respectively.
This subsection presents the multi-negative mini-batch objective used in the experiments. 
It generalizes the pairwise logistic contrastive mechanism in Definition \ref{cl_loss}, 
whose one-negative form enables the non-asymptotic analysis developed in Section \ref{sec_Error}.
\begin{itemize}
    \item Traditional Adversarial UDA (e.g., DANN) [Single-Stage Coupled Objective]:
    \begin{align*}
        \mathcal{L}_{\text{Trad}}(\theta_f, \theta_y, \theta_d) &= \underbrace{\frac{1}{n} \sum_{i=1}^{n} \mathcal{L}_y\left(G_y(G_f(\mathbf{x}_i^s; \theta_f); \theta_y), y_i\right)}_{\text{Supervised Prediction Loss}}\\
        &- \lambda \underbrace{\left( \frac{1}{n} \sum_{i=1}^{n} \mathcal{L}_d^i(G_f; \theta_d) + \frac{1}{N} \sum_{i=n+1}^{N+n} \mathcal{L}_d^i(G_f; \theta_d) \right)}_{\text{Domain Discrimination Loss}}
    \end{align*}
    \item \cite{shen2022connect} [Two-Stage Decoupled Objective]:
    Stage 1 (Representation Pre-training):
    \begin{equation*}
       \mathcal{L}_{\text{Shen}}(\theta_f) = \frac{1}{n+N} \sum_{i=1}^{n+N} \mathcal{L}_{\text{infoNCE}}\left(G_f(\mathbf{x}_i; \theta_f), G_f(\mathbf{x}_i^+; \theta_f), G_f(\mathbf{x}_i^-; \theta_f)\right) 
    \end{equation*}
    Stage 2 (Downstream Task Training):
    \begin{equation*}
        \mathcal{L}_{\text{Downstream}}(\theta_y) = \frac{1}{n} \sum_{i=1}^{n} \mathcal{L}_y\left(G_y\left(\hat{G}_f(\mathbf{x}_i^s); \theta_y\right), y_i\right) \quad \text{where } \hat{G}_f \text{ is frozen.}
    \end{equation*}
    \item JCAR [Two-Stage Decoupled Objective with Joint Contrastive Alignment]:
    Stage 1 (Contrastive Domain-Adaptive Pre-training):
    \begin{align*}
        \mathcal{L}_{\text{JCAR}}(\theta_f, \theta_d) &= \underbrace{\frac{1}{n+N} \sum_{i=1}^{n+N} \mathcal{L}_{\text{infoNCE}}\left(G_f(\mathbf{x}_i; \theta_f), G_f(\mathbf{x}_i^+; \theta_f), G_f(\mathbf{x}_i^-; \theta_f)\right)}_{\text{Unsupervised Joint Contrastive Loss}}\\
        &- \lambda \underbrace{\left( \frac{1}{n} \sum_{i=1}^{n} \mathcal{L}_d^i(G_f; \theta_d) + \frac{1}{N} \sum_{i=n+1}^{N+n} \mathcal{L}_d^i(G_f; \theta_d) \right)}_{\text{Domain Discrimination Loss}}
    \end{align*}
    Stage 2 (Downstream Task Training):
    \begin{equation*}
        \mathcal{L}_{\text{Downstream}}(\theta_y) = \frac{1}{n} \sum_{i=1}^{n} \mathcal{L}_y\left(G_y\left(\hat{G}_f(\mathbf{x}_i^s); \theta_y\right), y_i\right) \quad \text{where } \hat{G}_f \text{ is frozen.}
    \end{equation*}
    \end{itemize}
    \subsubsection{{Distinctions from Traditional UDA Methods}}\label{dis_tra}
    The formulations above highlight two differences between JCAR and standard single-stage adversarial UDA:
    \begin{itemize}
    \item[(A).] Source-Supervised vs. Joint-Domain Contrastive Feature Learning.
    In standard DANN-style training, task-discriminative supervision for $G_f$ is supplied by labeled source examples, 
    while unlabeled target examples contribute through the domain objective. Consequently, 
    source labels remain the primary driver of feature construction even though target data participates in alignment.
    JCAR instead applies the same self-supervised contrastive objective to the combined unlabeled dataset ($D_S \cup D_T$), 
    so both domains contribute symmetrically and directly to feature construction. 
    The adversarial term complements this instance-level objective by aligning the two feature distributions, 
    yielding a representation shaped jointly by within-sample consistency and cross-domain agreement.
    \item[(B).] Separating Representation Learning from Downstream Fitting.
    Standard adversarial UDA jointly optimizes $\mathcal{L}_y$ and $\mathcal{L}_d$ in one minimax loop. 
    JCAR separates joint-domain representation pretraining from fitting the downstream label predictor, 
    isolating the task-agnostic representation objective from source-label supervision.
    JCAR completely decouples these two objectives:
    \textbf{Stage 1 (Task-agnostic representation learning):}
    We learn the representation using $\mathcal{L}_{\text{infoNCE}}$ and $\mathcal{L}_d$ without task labels. 
    The representation is therefore constructed from the geometry of both domains and their domain-level alignment, 
    independently of the source classification objective.
    \textbf{Stage 2 (Flexible task learning):}
    After pretraining, $\hat{G}_f$ is frozen and the downstream predictor $G_y$ is trained on top of it. 
    This decoupling prevents downstream fitting from overwriting the learned common representation and permits the representation 
    to be reused across downstream tasks.
    \end{itemize}
    \subsubsection{\texorpdfstring{Distinctions from \citet{shen2022connect}}{Distinctions from Shen et al. (2022)}}\label{dis_shen}
    Both JCAR and \citet{shen2022connect} use contrastive pretraining followed by downstream classification. 
    JCAR incorporates an explicit adversarial alignment objective and analyzes its joint statistical and approximation 
    costs with feature learning:
    \begin{itemize}
    \item[(A).]Representation Learning and Transfer Mechanisms. Under their augmentation-connectivity model, 
    \citet{shen2022connect} show that contrastive representations can support target prediction while retaining domain information.
    JCAR instead includes domain distinguishability in the pretraining objective, making feature alignment an 
    explicit optimization target alongside contrastive learning.
    \item[(B).] Mathematical Formulation and Explicit Alignment.
    The contrastive objective in \citet{shen2022connect} has no explicit adversarial alignment term. 
    JCAR couples instance-level contrastive learning with domain discrimination through a minimax objective; 
    its analysis therefore controls the composite class of discriminators acting on learned features.
    \item[(C).] Implementation of Joint Pretraining.
    At the implementation level, JCAR combines symmetric source--target feature processing with a progressive gradient-reversal schedule:
    \textbf{Symmetric Feature Processing:} The same feature extractor and contrastive objective are applied 
    to source and target mini-batch examples, allowing both domains to contribute uniformly to the representation 
    rather than assigning target samples only an alignment role.
    \textbf{Progressive Gradient Reversal (Warm-Start GRL):} We use a warm-start schedule that gradually increases the adaptation 
    factor $\lambda(t)$:
  $$\lambda(t) = \frac{2 \cdot (\text{hi} - \text{lo})}{1 + e^{-\alpha \cdot t / T}} - (\text{hi} - \text{lo}) + \text{lo}$$
  This schedule introduces the adversarial signal gradually while contrastive features are being formed, 
  avoiding an abrupt full-strength alignment objective at the start of training.
  \item[(D).] Generalization and Theoretical Guarantees.
   Whereas \citet{shen2022connect} derive target-error guarantees for spectral contrastive learning through augmentation-graph connectivity, 
   our analysis explicitly tracks neural-network estimation and approximation errors for a contrastive--adversarial objective. 
   Theorem \ref{theorem_main} jointly balances feature-extractor and discriminator complexity, 
   while treating downstream predictor fitting separately. Under its assumptions, common smoothness $\beta\geq1$, 
   and sufficiently many unlabeled samples, the bound reduces to $\widetilde{\mathcal O}(m^{-\beta/(2\beta+d)})$. 
   This quantifies how unlabeled data support high-dimensional pretraining while labeled data govern low-dimensional prediction.
\end{itemize}
Together, joint-domain contrastive learning, explicit adversarial alignment, and frozen-feature prediction define JCAR's methodological contribution. The non-asymptotic analysis connects this unlabeled representation-learning stage to downstream statistical efficiency.
\section{Related Work}\label{related_works}
Distribution shifts between training and test data can degrade predictive performance \citep{quinonero2022dataset}. Domain adaptation addresses such shifts to improve performance on downstream domains. \citet{huang2006correcting, sugiyama2007direct} used reweighting techniques that estimate density ratios between domains through nonparametric methods such as kernel estimators. These methods, however, can be sensitive to kernel selection and difficult to apply to high-dimensional unstructured data \citep{sugiyama2007direct, sugiyama2008direct, yu2012analysis, cao2019learning}.
Furthermore, 
Maximum mean discrepancy (MMD) \citep{gretton2012optimal,sejdinovic2013equivalence} is widely used to align distributions. \citet{long2015learning,long2016unsupervised,long2016deep} studied feature transferability in deep convolutional networks and used MMD in a reproducing kernel Hilbert space to align deeper-layer representations across unlabeled source and target data. \citet{long2017deep,long2018conditional} extended this approach to joint and conditional distribution matching, respectively. Motivated by spectral analyses of feature representations, \citet{chen2019transferability} penalized dominant singular values to strengthen less dominant directions and improve feature discriminability.
Domain-adversarial methods provide another approach to distribution alignment \citep{tzeng2017adversarial,long2018conditional}. Deep domain-adaptation methods embed adaptation modules in neural networks to learn transferable representations while matching feature distributions across domains \citep{Long2016Regularization,tzeng2015simultaneous,ganin2015unsupervised,ganin2016domain,long2016unsupervised,tzeng2017adversarial}. \citet{cao2018partial} relaxed the shared-label-space assumption by allowing the target label space to be a subset of the source label space; their method suppresses outlier source classes and aligns distributions within the shared label space. Adversarial alignment can nevertheless impair transfer when suppressing domain-specific variation distorts useful feature structure. \citet{liu2019transferable} instead generated examples that are adversarial to both the category classifier and the domain discriminator without altering the feature representations.
When data distributions embody complex multimodal structures, adversarial adaptation methods may fail to capture these structures for a discriminative alignment of distributions, potentially leading to model mismatch. This risk cannot be tackled by aligning distributions of features and classes via separate domain discriminators as \citet{HoffmanWYD16, Chen_2017_ICCV}, because there is no guarantee that two distributions are sufficiently similar \citep{song2009, pmlr-v70-arora17a} even when the discriminator is fully confused.
DNNs have achieved significant success, and they are widely believed to be able to learn more transferable features \citep{zhao2019learning} since they disentangle explanatory factors of variations. 
\citet{ben2010theory, pan2010domain, zhang2013domain, wang2014flexible} studied feature representations for transfer between source and target domains.
Several recent efforts have focused on learning domain-invariant features in the context of neural networks  \citep{ajakan2014domain, ghifary2014domain, tzeng2014deep}. 
\citet{yosinski2014transferable} showed that learned deep representations capture factors shared across populations and can transfer from the original tasks to related tasks. DNNs can also learn more transferable representations than traditional hand-crafted features \citep{bengio2013representation}.
DNNs can reduce domain discrepancy \citep{glorot2011domain}, but residual dataset shift can increase target-domain risk \citep{mansour2009domain, ben2010theory}.
\citet{yao2020unsupervised} noticed that DNNs may suffer from the problem of long-tail data distribution in the target domain. Furthermore, not all knowledge of the pretrained models can be directly applied to the target task due to the discrepancy of various domains.
\section{Definitions}\label{definitions}
\begin{definition}\label{def_holderclass}
The set $\mho^{\beta}([0,1]_d^p,B)$ represents the class of H\"older-continuous functions mapping $[0,1]^p$ to $\mathbb{R}^d$,
where $p,d \in \mathbb{N}^+$, 
$\beta > 0$, with $\beta = \bar{\beta} + \tilde{r}$, $\bar{\beta} \in \mathbb{N}_0$, $\tilde{r} \in (0,1]$, and 
$B>0$ is a constant.
Specifically,  this class is  defined as
\begin{align*}
    \mho^{\beta}([0,1]_d^p,B) := \bigg\{
    &f=\{f_1,\ldots, f_d\}; f_i:[0,1]^p \to \mathbb{R}, \\
    &\max_{\|s\|_{1}\leq \bar{\beta}}\|\partial^{s}f_i\|_{\infty}\leq B,
    \max_{\|s\|_{1} = \bar{\beta}}\sup_{x \neq y}\dfrac{|\partial^{s}f_i(x) - \partial^{s}f_i(y)|}{\|x-y\|_2^{\tilde{r}}} \leq B,~ i=1,\ldots, d  \bigg\},
\end{align*}
where 
$s = (s_1,\ldots,s_p) \in \mathbb{N}^p$ and $\|s\|_1 = \sum_{j=1}^p s_j$, $\partial^{s}f_i(x) = \partial^{s_1}\cdots \partial^{s_p}f_i(x_1,\ldots,x_p)$ for $i = 1,\ldots, d$.
For convenience, 
we use $\mho^{\beta}([0,1]^p,B)$ to denote the function class for which $d=1$. 
\end{definition}
\begin{definition}[ReLU DNNs]\label{df:rednns}
    Let $\mac{G} = \mac{NN}(\mac{W}, \mac{D}, \mac{S}, \mac{B})$ be the class of DNNs with parameters $\mac{W, D, S, B}$ and ReLU activation function. Any $g \in \mac{G}$ can be parameterized as 
    \begin{equation*}
        \begin{split}
            g_0(x) &= x, \\
            g_{i+1}(x) &= \varrho(W_ig_i(x) + b_i), \quad i = 0, \dots, \mac{D} -1, \\
            g(x) &= W_{\mac{D}}g_{\mac{D}}(x)+b_{\mac{D}}.
        \end{split}
    \end{equation*}
Here, $\varrho(x):=\max\{x,0\}$ is applied element-wise. For $i=0,\ldots,\mac D-1$, $W_i\in\mathbb R^{d_{i+1}\times d_i}$ 
and $b_i\in\mathbb R^{d_{i+1}}$ parameterize hidden layer $i+1$; $W_{\mac D}$ and $b_{\mac D}$ parameterize the final affine output map. 
The symbol $\mac D$ denotes the number of hidden layers, $\mac W$ the maximum hidden-layer width, 
$\mac S$ the total number of trainable scalar parameters, and $\mac B>0$ a uniform output bound, $\|g\|_\infty\leq\mac B$.
\end{definition}
\begin{definition}[Norm-constrained DNNs]\label{def_norm_network}
    Denote $
    \mac{NN}(\mac{W, D, K})$ as the norm-constrained DNNs with width $\mac{W}$, depth $\mac{D}$ and weight constraint $\mac{K}$. Specifically, for any $f \in \mac{NN}(\mac{W, D, K})$, the network can be parameterized as
    \begin{equation*}
        \begin{split}
            f_{0} &=x, \\
            f_{i+1}(x) &= \varrho(W_{i}f_{i}(x) + b_i), \\
            f(x) &= W_{\mac{D}}f_{\mac{D}}(x).
        \end{split}
    \end{equation*}
This definition is similar to the DNN class in Definition~\ref{df:rednns}, but imposes the following norm constraint on the weights of $f$:
\begin{equation*}
    \|W_{\mac{D}}\| \Pi_{i =0}^{\mac{D} -1}\max\{\|(W_i, b_i)\|, 1\} \leq \mac{K},
\end{equation*}
where $\|W\| := \sup_{\|x\|_{\infty}\leq 1}\|Wx\|_{\infty}$ and $\mac{K} > 0$.
\end{definition}
\begin{definition}[Pseudo Dimension]
    Let $\mac{H}$ be a set of real-valued functions defined on   $\mac{X}$. 
    The set $A = \{x_i\}_{i =1}^n$ is pseudo-shattered by $\mac{H}$ if there exists a fixed threshold vector $w \in \mathbb{R}^n$ such that, for every $b \in \{\pm 1\}^n$, there exists $h \in \mac{H}$ satisfying $b_i(h(x_i)-w_i)>0$ for all $i=1,\dots,n$. The pseudo dimension is defined by
    \begin{equation*}
        Pdim(\mac{H}):= \sup\{|A|: A\subseteq\mac{X} \text{ is finite and pseudo-shattered by }\mac{H}\},
    \end{equation*}
    where $|\cdot|$ denotes cardinality; the supremum may be infinite.
\end{definition}
\begin{definition}[$\epsilon$-covering and Covering number]\label{epsilon cover}
Let $(\mathcal{F}, q)$
be a metric space, where the metric $q: \mathcal{F} \times \mathcal{F} \to \mathbb{R}^+$. For any $\epsilon > 0 $, 
 we define $\mathcal{A}:=\{f_i \in \mathcal{F};  i = 1, \dots, |\mathcal{A}|\}$ as an $\epsilon$-cover of $\mathcal{F}$,
 if for each $f \in \mathcal{F}$, there exists a $f_i \in \mathcal{A}$ such that $q(f, f_i) \leq \epsilon.$
The $\epsilon$-covering number of $\mathcal{F}$ is the minimal cardinality of an $\epsilon$-cover:
    \begin{equation*}
        \mathcal{C}(\epsilon, \mathcal{F}, q) = \min \{|\mathcal{A}|: \text{ $\mathcal{A}$ is an $\epsilon$-cover of $\mathcal{F}$}\}.
    \end{equation*}   
\end{definition}
\begin{definition}[Uniform covering number]
 Given $x = (x_1,\dots,x_n) \in \mathcal{X}^n$, denote $\mathcal{F}_{x} = \{(f(x_1), \dots,f(x_n)) : f \in \mathcal{F} \}$
 as a subset of $\mathbb{R}^n$.
Let $\mathcal{C}(\epsilon, \mathcal{F}_x, \|\cdot\|_{\infty})$ be the covering number of $\mac{F}_{x}$
with radius $\epsilon >0$ and metric $\|\cdot\|_{\infty}$.
Then, the uniform covering number $\mathcal{C}_{n}(\epsilon, \mathcal{F}, \|\cdot\|_{\infty})$ is defined as
    \begin{equation*}
       \mathcal{C}_n(\epsilon, \mathcal{F}, \|\cdot\|_{\infty}):= \max\{{\mathcal{C}(\epsilon, \mathcal{F}_x, \|\cdot\|_{\infty})} :x \in \mathcal{X}^n\}. 
    \end{equation*}   
\end{definition}
\section{Proof sketch of Theorem \ref{theorem_main}}\label{sec:ps}
To bound $\mac{E}_{\hat{f}}(\hat{g})$, 
we first decompose it into three terms as follows:
\begin{equation*}
    \begin{split}
      \mac{R}_{\hat{f}}(\hat{g}) - \mac{R}_{f^*}(g^*)
       &\leq \underbrace{2\sup_{g \in \mac{G}}|\hat{\mac{R}}_{\hat{f}}(g) - \mac{R}_{\hat{f}}(g)|}_{\mbox{(I)}} + \underbrace{\inf_{g\in \mac{G}}|\mac{R}_{\hat{f}}(g) - \mac{R}_{\hat{f}}(g^*)|}_{\mbox{(II)}}
      +\underbrace{\mac{R}_{\hat{f}}(g^*) - \mac{R}_{f^*}(g^*)}_{\mbox{(III)}}.
    \end{split}
\end{equation*}
Appendix~\ref{err_decompos} derives this decomposition; the following lemmas bound its three terms.
~~\\
\noindent
\textbf{Bound (I).} 
Term (I) is the predictor's statistical error for the learned representation $\hat f$. Lemma~\ref{lem_sta_g} bounds it by symmetrization and contraction.
\begin{lemma}\label{lem_sta_g}
Suppose that Assumption \ref{assum_holder} holds.
Let $\mac{G}$  be ReLU DNNs $\mac{NN}(\mac{W}_g, \mac{D}_g, \mac{S}_g)$ bounded by a  positive constant $B_{\mac{G}}$.
If $m \geq Pdim(\mac{G})$, then
    \begin{equation*}
    \mab{E}_{\bar{D}_{\mac{S}}}\bigg[\sup_{g \in \mac{G}}|\hat{\mac{R}}_{\hat{f}}(g) - \mac{R}_{\hat{f}}(g)| \bigg]\leq
B_{\mac{G}}\sqrt{\frac{\mac{S}_g\mac{D}_g \log(\mac{S}_g)\log(2meB_{\mac{G}})}{m}}. 
    \end{equation*}
\end{lemma}
This bound separates the effect of the labeled sample size $m$ from the complexity of the predictor class $\mathcal G$ and decreases as $m$ grows.
~~\\
\noindent
 \textbf{Bound (II).} 
Term (II) represents the approximation error of $g^*$ when using the DNN class $\mac{G}$ as the hypothesis class. More
precisely, we have
\begin{equation*}
\begin{split}
   |\mac{R}_{\hat{f}}(g) - \mac{R}_{\hat{f}}(g^*)| &= |\mab{E}_{P_{T}(X,Y)}[\ell(g\circ\hat{f}(X), Y) - \ell(g^*\circ\hat{f}(X), Y)]| \\
   & \leq \mab{E}_{P_{T}(X,Y)}[|\ell(g\circ\hat{f}(X), Y) - \ell(g^*\circ\hat{f}(X), Y)|]\\
   & \leq c_{\ell}\mab{E}_{P_{T}}[|g(\hat{f}(X)) - g^*(\hat{f}(X))|]\\
   & \leq c_{\ell} 
   \|g - g^*\|_{\infty},
\end{split}
\end{equation*}
where the second inequality follows from the 
$c_{\ell}$-Lipschitz continuity of $\ell(\cdot,\cdot)$. By taking infimum on both sides of the above inequality 
Taking the infimum and applying the approximation result of \citet{jiao2021deep}, stated in Lemma~\ref{lem_appro}, bounds term (II).
~~\\
\noindent
\textbf{Bound (III).}
Term (III) transfers the downstream representation error to the excess pretraining risk, which we then decompose into statistical and approximation terms.
By the loss's Lipschitz property and the compact-domain H\"older modulus $C_g$ of $g^*$, Assumption~\ref{assum_rep} gives
\[
\mathcal R_{\hat f}(g^*)-\mathcal R_{f^*}(g^*)
\leq\frac{C_gC_{\rm cal}}{\alpha}
\{\mathcal R_{\mathcal M_h}(\hat f)-\mathcal R_{\mathcal M_h}(f^*)\}.
\]
Here $\alpha$ is the source mixture proportion; the constants are fixed independently of the sample sizes.
Thus, it suffices to bound the excess pretraining risk $\mac{R}_{\mac{M}_h}(\hat{f})-\mac{R}_{\mac{M}_h}(f^*)$, which decomposes into the following three terms:
\begin{equation*}
 \begin{split}
&\mac{R}_{\mac{M}_h}(\hat{f}) - \mac{R}_{\mac{M}_h}(f^*)\\
&\leq \underbrace{2\sup_{f\in \mac{F}}|\mac{R}_{\mac{H}}(f) - \hat{\mac{R}}_{\mac{H}}(f)|}_{\text{statistical error:~}  \mac{E}_{sta}}
+ \underbrace{2\sup_{f\in \mac{F}}|\mac{R}_{\mac{H}}(f) - \mac{R}_{\mac{M}_h}(f)|}_{\text{approximation error:~}  \mac{E}_{h}}\\
&\quad + \underbrace{\inf_{f \in \mac{F}}\mac{R}_{\mac{M}_h}(f) - \mac{R}_{\mac{M}_h}(f^*)}_{\text{approximation error:~} \mac{E}_{f}}.
\end{split}   
\end{equation*}
Details regarding the error transformation and decomposition can be found in Appendix \ref{err_trans_deom_III}. 
For the statistical error  $\mac{E}_{sta}$,  it can be bounded by the sum  of the two parts  below: 
\begin{equation*}
\sup_{f\in \mac{F}}|\mac{R}_{\mac{H}}(f) - \hat{\mac{R}}_{\mac{H}}(f)| 
\leq \sup_{f\in \mac{F}}|\hat{\mac{R}}_{cl}(f) - \mac{R}_{cl}(f) | +  \sup_{f\in \mac{F}, h \in \mac{H}}|  \hat{\mac{R}}_{d}(f, h) - \mac{R}_{d}(f,h)|. 
\end{equation*}
Appendix~\ref{prof_pat3} derives this decomposition, and Lemmas~\ref{lem_sta_f}--\ref{lem_sta_h} bound its two terms.
\begin{lemma}\label{lem_sta_f}
For independent sampled triples and the bounded representation class specified in Section~\ref{sec_consis},
\[
\mathbb E\sup_{f\in\mathcal F}
|\mathcal R_{cl}(f)-\hat{\mathcal R}_{cl}(f)|
\lesssim d \mathcal K_f
\sqrt{\frac{\mathcal D_f+2+\log(p+1)}{n}}.
\]
\end{lemma}
\begin{lemma}\label{lem_sta_h}
For independent domain-labeled anchors and the norm-constrained classes in Section~\ref{sec_consis},
\[
\mathbb E\sup_{f\in\mathcal F,h\in\mathcal H}
|\mathcal R_d(f,h)-\hat{\mathcal R}_d(f,h)|
\lesssim  \frac{ \mac{K}_f \mac{K}_h d\sqrt{2(\mac{D}_f + 2 + \log(p+1))}}{\sqrt{n}} .
\]
\end{lemma}
Lemmas \ref{lem_sta_f} and \ref{lem_sta_h} control the contrastive term over $\mathcal F$ 
and the domain term over the composite class $\mathcal H\circ\mathcal F$, respectively. 
Both use the joint unlabeled sample size $n$. The domain term depends on both network classes, 
which couples feature-extractor and discriminator complexity in the pretraining bound.
For the approximation errors $\mac E_h$ and $\mac E_f$, Appendix~\ref{prof_pat3} shows that
\[
\mac E_h\lesssim
\sup_{h^*\in\mac M_h}\inf_{h\in\mac H}\|h-h^*\|_\infty,
\qquad
\mac E_f\lesssim\inf_{f\in\mac F}\|f-f^*\|_\infty.
\]
Consequently,
Lemma~\ref{lem_norm_constrain_appro}, based on the approximation results of \citet{jiao2023approximation}, bounds these two errors.
Therefore, we can now bound the excess risk $\mac{R}_{\mac{M}_h}(\hat{f}) - \mac{R}_{\mac{M}_h}(f^*) $.
\begin{lemma}\label{lem_pat3}
Under Assumptions~\ref{assum_holder} and independent representation sampling, and the choices (a)--(b),
\[
\mathbb E\{\mathcal R_{\mathcal M_h}(\hat f)-\mathcal R_{\mathcal M_h}(f^*)\}
\lesssim 
\frac{d \mathcal K_f}{\sqrt{n}} + \mathcal K_f\mathcal K_h
\sqrt{\frac{\mathcal D_h+2+\log(d+1)}{n}}+\mathcal K_f^{-\frac{\beta_f}{p+1}}+\mathcal K_h^{-\frac{\beta_h}{d+1}}.
\]
This representation-stage result isolates the statistical and approximation costs of contrastive--adversarial pretraining before the downstream calibration step.
\end{lemma}
Lemma \ref{lem_pat3} combines the joint statistical complexity of $\mathcal F$ and $\mathcal H$ with their approximation errors. 
At fixed network budgets and depths, the statistical terms scale as $n^{-1/2}$; balancing these terms with approximation yields 
the coupled pretraining rate. Thus the bound quantifies the role of the joint unlabeled sample size without assuming that the ideal 
functions are exactly realizable by the chosen networks.
By combining the results from \textbf{Bound (I)}--\textbf{Bound (III)}, we obtain the desired result in Theorem \ref{theorem_main}.
\section{Error Decompositions}\label{err_decompos}
Given the learned representation function $\hat{f}$, the excess risk for any $g \in \mac{G}$ can be decomposed as follows:
\begin{equation*}
    \begin{split}
      \mac{R}_{\hat{f}}(\hat{g}) - \mac{R}_{f^*}(g^*) &=  \mac{R}_{\hat{f}}(\hat{g}) - \hat{\mac{R}}_{\hat{f}}(\hat{g})
      +\hat{\mac{R}}_{\hat{f}}(\hat{g}) - \hat{\mac{R}}_{\hat{f}}(g)
      + \hat{\mac{R}}_{\hat{f}}(g) - \mac{R}_{\hat{f}}(g)\\
      &+\mac{R}_{\hat{f}}(g) - \mac{R}_{\hat{f}}(g^*)
      +\mac{R}_{\hat{f}}(g^*)  - \mac{R}_{f^*}(g^*)\\
       & \leq \mac{R}_{\hat{f}}(\hat{g}) - \hat{\mac{R}}_{\hat{f}}(\hat{g})
      + \hat{\mac{R}}_{\hat{f}}(g) - \mac{R}_{\hat{f}}(g)
      +\mac{R}_{\hat{f}}(g) - \mac{R}_{\hat{f}}(g^*)\\
      &+\mac{R}_{\hat{f}}(g^*)  - \mac{R}_{f^*}(g^*)\\
       &\leq \underbrace{2\sup_{g \in \mac{G}}|\hat{\mac{R}}_{\hat{f}}(g) - \mac{R}_{\hat{f}}(g)|}_{\mbox{(I)}} + \underbrace{\inf_{g\in \mac{G}}|\mac{R}_{\hat{f}}(g) - \mac{R}_{\hat{f}}(g^*)|}_{\mbox{(II)}}
      +\underbrace{\mac{R}_{\hat{f}}(g^*) - \mac{R}_{f^*}(g^*)}_{\mbox{(III)}}.
    \end{split}
\end{equation*}
Here, the first inequality follows from the definition of $\hat{g}$ in \eqref{def_g}. 
\section{Proof of Lemma \ref{lem_risk4ST}}\label{pro_risk4ST}
\begin{proof}
Let 
$f$ be the common representation function. Write $U^{\mac{S}}:=f(X^{\mac{S}})$ and $U^{\mac{T}}:=f(X^{\mac{T}})$ for the source and target representations. Let $P_S(U^{\mac{S}},Y^{\mac{S}})$ and $P_T(U^{\mac{T}},Y^{\mac{T}})$ denote their joint distributions, with marginals $P_S(U^{\mac{S}})$ and $P_T(U^{\mac{T}})$. Then, we have 
\begin{equation*}
    \begin{split}
    \mab{E}_{P_{T}(U^{\mac{T}}, Y^{\mac{T}})}[\ell(g(U^{\mac{T}}), Y^{\mac{T}})] &= \int \ell(g(U^{\mac{T}}), Y^{\mac{T}})P_{T}(U^{\mac{T}}, Y^{\mac{T}})dU^{\mac{T}}dY^{\mac{T}}\\
        &  = \int \ell(g(U^{\mac{T}}), Y^{\mac{T}})P_{T}(Y^{\mac{T}}|U^{\mac{T}})P_T(U^{\mac{T}}) dU^{\mac{T}}dY^{\mac{T}}\\
        &= \int \ell(g(U^{\mac{S}}), Y^{\mac{S}})P_{S}(Y^{\mac{S}}\mid U^{\mac{S}})P_S(U^{\mac{S}})\,dU^{\mac{S}}dY^{\mac{S}}\\
        &=\mab{E}_{P_S(U^{\mac{S}}, Y^{\mac{S}})}[\ell(g(U^{\mac{S}}), Y^{\mac{S}})],
    \end{split}
    \end{equation*}
    where the third equality follows 
    from
    $P_{S}(U^{\mac{S}}) = P_{T}(U^{\mac{T}})$ and $P_{S}(Y^{\mac{S}}|U^{\mac{S}}) = P_{T}(Y^{\mac{T}}|U^{\mac{T}})$. Then, 
    we have $g^{**} = g^*$.
\end{proof}
\section{Proof of Theorem \ref{theorem_main}}\label{prof_theorem_main}
\begin{proof}
We set the network architectures of $\mac{F,H, G}$ as follows:
\begin{enumerate}
\item[(a).] Let $\mac{F} := \mac{NN}(\mac{W}_f, \mac{D}_f, \mac{K}_f)$ be the norm-constrained DNNs with width
$\mathcal W_f\gtrsim d\mathcal K_f^{(2p+\beta_f)/(2p+2)}$, depth
$\mathcal D_f\geq2\lceil\log_2(p+\bar\beta_f)\rceil+3$, and norm constraint $\mac{K}_f \asymp n^{\beta_h(p+1)/(2\beta_{f}\beta_h +2\beta_f(d+1) +2\beta_h(p+1))}$,
\item[(b).] Let $\mac{H}:= \mac{NN}(\mac{W}_h, \mac{D}_h, \mac{K}_h)$
    be the norm-constrained DNNs with width
$\mathcal W_h\gtrsim\mathcal K_h^{(2d+\beta_h)/(2d+2)}$, depth
$\mathcal D_h\geq2\lceil\log_2(d+\bar\beta_h)\rceil+2$, and norm constraint $\mac{K}_h \asymp n^{\beta_f(d+1)/(2\beta_{f}\beta_h +2\beta_f(d+1) +2\beta_h(p+1))}$
\item[(c).] Let $\mac{G}:= \mac{NN}(\mac{W}_g, \mac{D}_g, \mac{S}_g)$ be ReLU DNNs with width $\mac{W}_g \geq m^{d/4(2\beta_{g} +d)}\log(m)$, 
depth $\mac{D}_g \geq m^{d/4(2\beta_{g} +d)}\log(m)$, and size  $\mac{S}_g \geq m^{3d/4(2\beta_{g} +d)}\log^3(m)$.
\end{enumerate}
By combining Lemmas  \ref{lem_sta_g}, \ref{lem_pat3}, and \ref{lem_appro}, we have
\begin{align*}
 \mathbb E\mathcal E_{\hat f}(\hat g) &\lesssim B_{\mac{G}}\sqrt{\frac{\mac{S}_g\mac{D}_g \log(\mac{S}_g)\log(2meB_{\mac{G}})}{m}} + 19B_{g}(\lfloor \beta_{g}  \rfloor +1)^2d^{\lfloor \beta_{g}  \rfloor + (\beta_{g} \vee 1)/2}(MN)^{-2\beta_{g}/d}\\
&+ \frac{d \mathcal K_f}{\sqrt{n}} + \mathcal K_f\mathcal K_h
\sqrt{\frac{\mathcal D_h+2+\log(d+1)}{n}}+\mathcal K_f^{-\frac{\beta_f}{p+1}}+\mathcal K_h^{-\frac{\beta_h}{d+1}},
\end{align*}
we then obtain the results in Theorem \ref{theorem_main} under the (a)--(c) network architectures.
\end{proof}
\section{Proof of Lemma \ref{lem_sta_g}.}\label{prof_sta_g}
\begin{proof}
Let $\bar{D}_{S}:= \{(X_i^{\mac{S}}, Y_i^{\mac{S}})\}_{i=1}^m$  and $\bar{D}_{T}:= \{(X_j^{\mac{T}}, Y_j^{\mac{T}})\}_{j=1}^m$ be the  
labeled datasets from the source and target domains, respectively.
Let $\bar{D}_{S^{\prime}}:= \{({X_i^{\mac{S}}}^{\prime}, {Y_i^{\mac{S}}}^{\prime})\}_{i=1}^m$ and $\bar{D}_{T^{\prime}}:= \{({X_j^{\mac{T}}}^{\prime}, {Y_j^{\mac{T}}}^{\prime})\}_{j=1}^m$ be independent ghost samples of $\bar{D}_{S}$ and $\bar{D}_{T}$. Let $\epsilon:=\{\epsilon_j\}_{j=1}^m$ be independent Rademacher variables, each taking values in $\{-1,1\}$ with equal probability. Taking the expectation of $\sup_{g \in \mac{G}}\{\mac{R}_{\hat{f}}(g)-\hat{\mac{R}}_{\hat{f}}(g)\}$ with respect to $\bar{D}_T$ yields
we have
\begin{equation*}
    \begin{split}
&\mab{E}_{\bar{D}_{T}}\bigg[\sup_{g \in \mac{G}} \mac{R}_{\hat{f}}(g)-\hat{\mac{R}}_{\hat{f}}(g) \bigg]\\
&= \mab{E}_{\bar{D}_{T}}\bigg[\sup_{g \in \mac{G}} \big(\mab{E}_{P_{T}}[\ell(g\circ \hat{f}(X^{\mac{T}}), Y^{\mac{T}})]- \frac{1}{m}\sum_{j=1}^m\ell(g\circ \hat{f}(X_j^{\mac{T}}), Y_j^{\mac{T}}) \big)\bigg]\\
        &= \mab{E}_{\bar{D}_{T}}\bigg[\sup_{g \in \mac{G}} \big(\mab{E}_{\bar{D}_{T^\prime}}[\frac{1}{m}\sum_{j =1}^m\ell(g\circ \hat{f}({X_j^{\mac{T}}}^\prime), {Y_j^{\mac{T}}}^\prime)]- \frac{1}{m}\sum_{j =1}^m\ell(g\circ \hat{f}(X_j^\mac{T}), Y_j^\mac{T}) \big)\bigg]\\
        &= \mab{E}_{\bar{D}_{T}}\bigg[\sup_{g \in \mac{G}} \mab{E}_{\bar{D}_{T^\prime}}\big[\frac{1}{m}\sum_{j =1}^m\ell(g\circ \hat{f}({X_j^\mac{T}}^\prime), {Y_j^\mac{T}}^\prime)- \frac{1}{m}\sum_{j=1}^m\ell(g\circ \hat{f}(X_j^\mac{T}), Y_j^\mac{T})\big] \bigg|\bar{D}_{T} \bigg]\\
        &\leq \mab{E}_{\bar{D}_{T},\bar{D}_{T^\prime}}\bigg[\sup_{g \in \mac{G}} \frac{1}{m}\sum_{j =1}^m \big(\ell(g\circ \hat{f}({X_j^\mac{T}}^\prime), {Y_j^\mac{T}}^\prime)- \ell(g\circ \hat{f}(X_j^\mac{T}), Y_j^\mac{T})\big) \bigg|\bar{D}_{T} \bigg] \quad \text{(Jensen's inequality)}\\
        &= 2\mab{E}_{\epsilon, \bar{D}_{T}}\bigg[\sup_{g \in \mac{G}} \frac{1}{m}\sum_{j =1}^m \epsilon_{j} \ell(g\circ \hat{f}(X_j^\mac{T}), Y_j^\mac{T})  \bigg] \quad \text{(symmetrization)}\\
        &\leq 2c_{\ell}\mab{E}_{\epsilon, \bar{D}_{T}}\bigg[\sup_{g \in \mac{G}} \frac{1}{m}\sum_{j =1}^m \epsilon_{j} g\circ \hat{f}(X_j^\mac{T})   \bigg]\\
        &:=2c_{\ell}\mathfrak{R}_{\bar{D}_{T}}(\mac{G} \circ \hat{f}).
    \end{split}
\end{equation*}
The last inequality follows from  Talagrand inequality \citet[Theorem 4.12]{ledoux2013probability}, since $\ell$ is $1$-Lipschitz.  Let $(U^{\mac{T}},Y^{\mac{T}}) = (\hat{f}(X^{\mac{T}}), Y^{\mac{T}})$ 
and $\bar{D}_{U^{\mac{T}}} = \{(U^{\mac{T}}_j, Y^{\mac{T}}_j)\}_{j=1}^m$ with distribution $P_{\bar{D}_{U^{\mac{T}}}}(U^\mac{T}, Y^\mac{T})$. Since $\hat{f}$ is fixed,
\begin{equation*}
 \mathfrak{R}_{\bar{D}_{T}}(\mac{G} \circ \hat{f}) =  \mab{E}_{\epsilon, \bar{D}_{T}}\bigg[\sup_{g \in \mac{G}} \frac{1}{m}\sum_{j =1}^m \epsilon_{j} g\circ \hat{f}(X_j) \bigg] \leq   \mab{E}_{\epsilon, \bar{D}_{U^{\mac{T}}}}\bigg[\sup_{g \in \mac{G}} \frac{1}{m}\sum_{j =1}^m \epsilon_{j} g(U_j) \bigg]:=\mathfrak{R}_{\bar{D}_{U^{\mac{T}}}}(\mac{G}).
\end{equation*}
Because the dataset $\bar{D}_{T}$ is unavailable, we cannot calculate the risk $\mathfrak{R}_{\bar{D}_{U^{\mac{T}}}}(\mac{G})$ in practice. However, by Lemma \ref{lem_risk4ST} and its proof, $P_{\bar{D}_T}(U^\mac{T}, Y^\mac{T}) = P_{\bar{D}_S}(U^\mac{S}, Y^\mac{S})$, and consequently $\mathfrak{R}_{\bar{D}_{U^{\mac{T}}}}(\mac{G}) = \mathfrak{R}_{\bar{D}_{U^{\mac{S}}}}(\mac{G})$.
Let $\mathcal{C}_{m}(\delta, \mathcal{G}, \|\cdot\|_{\infty})$ denote the uniform covering number of $\mac{G}$, where $\|g\|_{\infty}:=\max\{|g(U_j)|:j=1,\ldots,m\}$. Since $\|g\|_{\infty}\leq B_{\mac{G}}$ for $B_{\mac{G}}>0$, Dudley's entropy integral \citep{anthony2009neural} gives
\begin{equation}\label{eq_dudley}
    \mathfrak{R}_{\bar{D}_{U^{\mac{S}}}}(\mac{G}) \leq \mab{E}\inf_{0<\delta <B_{\mac{G}}} \bigg(4\delta + \frac{12}{\sqrt{m}}\int_{\delta}^{B_{\mac{G}}} \sqrt{\log(2\mathcal{C}_m(t, \mathcal{G}, \|\cdot\|_{\infty}))} \, dt\bigg).
\end{equation}
Let $Pdim(\mac{G})$ denote the pseudo dimension of $\mac{G}$. For $1\leq Pdim(\mac{G})<m$ and $0<\delta\leq B_{\mac{G}}$, applying \citet[Theorem 12.2]{anthony2009neural} to the shifted class $\{g+B_{\mac{G}}:g\in\mac{G}\}$, whose range lies in $[0,2B_{\mac{G}}]$, gives
\begin{equation}\label{eq_psu}
    \mathcal{C}_m(\delta,  \mathcal{G}, \|\cdot\|_{\infty}) \leq \left(\dfrac{2eB_{\mac{G}}m}{\delta Pdim(\mac{G})} \right)^{Pdim(\mac{G})}.
\end{equation}
Moreover, the ReLU network pseudodimension bound of \citet{bartlett2019nearly} gives, for $\mac{S}_g\geq2$ and $\mac{D}_g\geq1$, a universal constant $C>0$ such that
\begin{equation*}
    Pdim(\mac{G}) \leq C\cdot\mac{S}_g \mac{D}_g\log(\mac{S}_g).
\end{equation*}
Combining these bounds proves Lemma~\ref{lem_sta_g}:
\begin{equation*}
   \mathfrak{R}_{\bar{D}_{U^{\mac{S}}}}(\mac{G}) \leq B_{\mac{G}}\sqrt{\frac{\mac{S}_g\mac{D}_g\log(\mac{S}_g)\log(2meB_{\mac{G}})}{m}}.
\end{equation*} 
\end{proof}
\section{Error transform and decomposition}\label{err_trans_deom_III}
\begin{proof}
First, we bound $\mac{R}_{\hat{f}}(g^*) - \mac{R}_{f^*}(g^*)$ by $\mac{R}_{\mac{M}_h}(\hat{f}) - \mac{R}_{\mac{M}_h}(f^*)$:
\begin{equation*}
    \begin{split}
     \mac{R}_{\hat{f}}(g^*) - \mac{R}_{f^*}(g^*)& = \mab{E}_{P_S}[\ell(g^*\circ \hat{f}(X), Y) - \ell(g^*\circ f^*(X), Y)]\\
     &\leq c_{\ell}\mab{E}_{P_S}[|g^*\circ \hat{f}(X) - g^*\circ f^*(X)|]\\
     &\leq c_{\ell} C_{g}\mab{E}_{P_S}[\|\hat{f}(X) - f^*(X)\|_2^{\tilde{r}_g}]\\
     &\leq
     \frac{C_g }{\alpha}c_{\ell}\mab{E}_{P_{C}(X)}[\|\hat{f}(X) - f^*(X)\|_2^{\tilde{r}_g}]
     \\
     &\leq \frac{C_g C_{\rm cal}}{\alpha}c_{\ell} \big(\mac{R}_{\mac{M}_h}(\hat{f}) - \mac{R}_{\mac{M}_h}(f^*)\big),
    \end{split}
\end{equation*}
The first inequality uses the $c_{\ell}$-Lipschitz continuity of $\ell$, the second uses the H\"older modulus of $g^*$ with constant $C_g$, the third follows from $P_C(X)=\alpha P_S(X)+(1-\alpha)P_T(X)\geq\alpha P_S(X)$, and the last follows from Assumption~\ref{assum_rep}. We therefore obtain the following decomposition:
\begin{equation*}
\begin{split}
&\mac{R}_{\mac{M}_h}(\hat{f}) - \mac{R}_{\mac{M}_h}(f^*) \\ & =\mac{R}_{\mac{M}_h}(\hat{f}) - \mac{R}_{\mac{H}}(\hat{f})
   + \mac{R}_{\mac{H}}(\hat{f}) - \hat{\mac{R}}_{\mac{H}}(\hat{f})
   + \hat{\mac{R}}_{\mac{H}}(\hat{f}) - \hat{\mac{R}}_{\mac{H}}(f)\\
   &+ \hat{\mac{R}}_{\mac{H}}(f) - \mac{R}_{\mac{H}}(f)
   + \mac{R}_{\mac{H}}(f) - \mac{R}_{\mac{M}_h}(f) 
   + \mac{R}_{\mac{M}_h}(f) - \mac{R}_{\mac{M}_h}(f^*)\\
   & \leq \mac{R}_{\mac{M}_h}(\hat{f}) - \mac{R}_{\mac{H}}(\hat{f})
   + \mac{R}_{\mac{H}}(\hat{f}) - \hat{\mac{R}}_{\mac{H}}(\hat{f})\\
   &+ \hat{\mac{R}}_{\mac{H}}(f) - \mac{R}_{\mac{H}}(f)
   + \mac{R}_{\mac{H}}(f) - \mac{R}_{\mac{M}_h}(f) 
   + \mac{R}_{\mac{M}_h}(f) - \mac{R}_{\mac{M}_h}(f^*)\\
   &\leq \underbrace{2\sup_{f\in \mac{F}}|\mac{R}_{\mac{H}}(f) - \hat{\mac{R}}_{\mac{H}}(f)|}_{\text{statistical error:} \mac{E}_{sta}}
   + \underbrace{2\sup_{f\in \mac{F}}|\mac{R}_{\mac{H}}(f) - \mac{R}_{\mac{M}_h}(f)|}_{\text{approximation error:} \mac{E}_{h}}\\
   &\quad + \underbrace{\inf_{f \in \mac{F}}\mac{R}_{\mac{M}_h}(f) - \mac{R}_{\mac{M}_h}(f^*)}_{\text{approximation error:} \mac{E}_{f}}.
\end{split}
\end{equation*}
\end{proof}
\section{Proof of Lemma \ref{lem_sta_f}}\label{prof_sta_f}
\begin{proof}
Considering the expectation of $\sup_{f \in \mac{F}}(\mac{R}_{cl}(f) - \hat{\mac{R}}_{cl}(f))$, we obtain
  \begin{equation*}
      \begin{split}
       &\mab{E}_{D_{\bar{C}}}[\sup_{f \in \mac{F}} (\mac{R}_{cl}(f) - \hat{\mac{R}}_{cl}(f))] =  \mab{E}_{D_{\bar{C}}}\bigg[\sup_{f \in \mac{F}} \big(\mab{E}_{P_{C}}\big[\ell_{cl}(X,f)\big] - \frac{1}{n}\sum_{i =1}^n\ell_{cl}(X_i, f)\big)\bigg] \\
       &= \mab{E}_{D_{\bar{C}}}\bigg[\sup_{f \in \mac{F}} \big(\mab{E}_{D_{\bar{C}^\prime}}\big[\frac{1}{n}\sum_{i=1}^n\ell_{cl}(X_i^\prime,f)\big] - \frac{1}{n}\sum_{i =1}^n\ell_{cl}(X_i, f)\big)\bigg]\\
       &= \mab{E}_{D_{\bar{C}}}\bigg[\sup_{f \in \mac{F}} \big(\mab{E}_{D_{\bar{C}^\prime}}\big[\frac{1}{n}\sum_{i=1}^n\ell_{cl}(X_i^\prime,f) - \frac{1}{n}\sum_{i =1}^n\ell_{cl}(X_i, f)\big)\big]\bigg|D_{\bar{C}}\bigg]\\
       &= \mab{E}_{D_{\bar{C}}}\bigg[\sup_{f \in \mac{F}} \mab{E}_{D_{\bar{C}^\prime}}\big[\frac{1}{n}\sum_{i=1}^n \big(\ell_{cl}(X_i^\prime,f) - \ell_{cl}(X_i, f)\big)\big]\bigg|D_{\bar{C}}\bigg]\\
       &\leq \mab{E}_{D_{\bar{C}},D_{\bar{C}^\prime}}\bigg[\sup_{f \in \mac{F}} \frac{1}{n}\sum_{i=1}^n \big(\ell_{cl}(X_i^\prime,f) - \ell_{cl}(X_i, f)\big)\bigg|D_{\bar{C}}\bigg] \quad \text{(Jensen's inequality)}\\
       &= \mab{E}_{\epsilon, D_{\bar{C}},D_{\bar{C}^\prime}}\bigg[\sup_{f \in \mac{F}} \frac{1}{n}\sum_{i=1}^n \epsilon_i \big(\ell_{cl}(X_i^\prime,f) - \ell_{cl}(X_i, f)\big)\bigg|D_{\bar{C}}\bigg] \quad \text{(Symmetrization)}\\
       &= 2\mab{E}_{\epsilon, D_{\bar{C}}}\bigg[\sup_{f \in \mac{F}} \frac{1}{n}\sum_{i=1}^n \epsilon_i  \ell_{cl}(X_i, f)\bigg]\\
       & = 2\mathfrak{R}_{\bar{C}}(\ell_{cl}\circ\mac{F}),
      \end{split}
  \end{equation*}   
where $D_{\bar{C}}=\{(X_i, X_i^+, X_i^-)\}_{i =1}^{n}$, 
 $D_{\bar{C}^\prime}:=\{(X^\prime_i, {X_i^\prime}^+, {X_i^\prime}^-)\}_{i =1}^{n}$ is the ghost sample of $D_{\bar{C}}$, 
and $\epsilon_i$ are independent Rademacher variables, each taking values in $\{-1,1\}$ with equal probability.
Since $f:\mab{R}^p\to\mab{R}^d$ and $\ell_{cl}$ is Lipschitz continuous,
\begin{equation*}
      \mathfrak{R}_{\bar{C}}(\ell_{cl}\circ\mac{F})\leq \ell_{cl}\mathfrak{R}_{\bar{C}}(\mac{F}) \leq  d \mac{K}_{f}\sqrt{\frac{\mathcal D_f+2+\log(p+1)}{n}}. 
\end{equation*}
The first inequality uses the Lipschitz continuity of $\ell_{cl}$, and the second follows from Lemma~\ref{lem_elemet_radermachec}.
\end{proof}
\section{Proof of Lemma \ref{lem_sta_h}}\label{prof_sta_h}
\begin{proof}
  We consider the expectation of  $\sup_{f\in \mac{F},h \in \mac{H}}\big(\mac{R}_{d}(f,h) - \hat{\mac{R}}_{d}(f,h)\big)$, i.e.,
  \begin{equation*}
\begin{split}
&\mab{E}_{D_{J}}[\sup_{f\in \mac{F},h \in \mac{H}}(\mac{R}_{d}(f,h) - \hat{\mac{R}}_{d}(f, h))]\\
&= \mab{E}_{D_{J}}\bigg[\sup_{f\in \mac{F},h \in \mac{H}}\big(\mab{E}_{P_{\mu}}[\ell_{d}(X, Z,f,h)] - \frac{1}{n}\sum_{k =1}^n\ell_{d}(X_k, Z_k,f, h)\big) \bigg]\\
       & = \mab{E}_{D_{J}}\bigg[\sup_{f\in \mac{F},h \in \mac{H}}\big(\mab{E}_{D_{J^\prime}}[\frac{1}{n}\sum_{k =1}^n \ell_{d}(X_k^\prime, Z_k^{\prime},f, h)] - \frac{1}{n}\sum_{k =1}^n\ell_{d}(X_k, Z_k,f, h)\big) \bigg| D_{J}\bigg]\\
        & \leq \mab{E}_{D_{J},D_{J^\prime}}\bigg[\sup_{f\in \mac{F},h \in \mac{H}}\frac{1}{n}\sum_{k =1}^n \big(\ell_{d}(X_k^\prime,Z_k^{\prime},f, h) - \ell_{d}(X_k, Z_k,f, h)\big) \bigg| D_{J}\bigg] \quad \text{(Jensen's inequality)}\\
        &= \mab{E}_{\epsilon, D_{J},D_{J^\prime}}\bigg[\sup_{f\in \mac{F},h \in \mac{H}}\frac{1}{n}\sum_{k =1}^n \epsilon_{k} \big(\ell_{d}(X_k^\prime, Z_k^{\prime},f, h) - \ell_{d}(X_k, Z_k,f, h)\big) \bigg| D_{J}\bigg] \quad \text{(Symmetrization)}\\
        & =2\mab{E}_{\epsilon, D_{J}}\bigg[\sup_{f\in \mac{F},h \in \mac{H}}\frac{1}{n}\sum_{k =1}^n \epsilon_{k} \ell_{d}(X_k, Z_k,f, h) \bigg]\\
        &=2\mathfrak{R}_{C}(\ell_{d}\circ(\mac{H} \circ \mac{F}))\\
        &\leq 2\mathfrak{R}_{C}(\mac{H} \circ \mac{F}).
      \end{split}
  \end{equation*}
  The last inequality follows from Talagrand inequality \citep[Theorem 4.12]{ledoux2013probability}, since $\ell_{d}$ is $1$-Lipschitz. 
Let $\mac{C}(\delta, \mac{H}\circ \mac{F}, \|\cdot\|_{2})$ denote the covering number of $\mac{H}\circ\mac{F}$ with respect to the empirical norm $\|h\circ f\|_2:=\sqrt{n^{-1}\sum_{k=1}^n(h\circ f(X_k))^2}$. Since $\|h\|_2\leq B_{\mac{H}}$, the chaining argument of \citet[Lemma 27.4]{shalev2014understanding} and Dudley's entropy integral \citep{dudley1967sizes} give
\begin{equation}\label{eq_rade_covering}
   \mathfrak{R}_{C}(\mac{H}\circ \mac{F}) \leq   
\mathbb{E}   
   \inf_{0 \leq \delta \leq B_{\mac{H}}}\bigg(4\delta + \frac{12}{\sqrt{n}} \int_{\delta}^{B_{\mac{H}}} \sqrt{\log(2\mac{C}(t, \mac{H}\circ\mac{F}, \|\cdot\|_{2}))}\, dt \bigg). 
  \end{equation}
We bound the entropy integral in \eqref{eq_rade_covering} by constructing a cover of $\mac H\circ\mac F$ from covers of $\mac H$ and $\mac F$. Recall that $D_C=\{X_i\}_{i=1}^n$.
Let $\mac{C}_{\mac{F}}(\delta_1, \mac{F}, \|\cdot\|_{2})$ be a $\delta_1$-cover of $\mac F$ on $D_C$. For each $f\in\mac{C}_{\mac F}$, let $\mac{C}_{\mac{H}_f}(\delta_2,\mac H,\|\cdot\|_2)$ be a $\delta_2$-cover of $\mac H$ on the feature inputs generated by $f$. Then $\mac{C}_{\mac H\circ\mac F}:=\bigcup_{f\in\mac{C}_{\mac F}}\mac{C}_{\mac{H}_f}$ is a $(L(\mac H)\delta_1+\delta_2)$-cover of $\mac H\circ\mac F$ on $D_C$. Indeed, for arbitrary $f\in\mac F$ and $h\in\mac H$, choose $f^\prime\in\mac{C}_{\mac F}$ within $\delta_1$ of $f$ and $h^\prime\in\mac{C}_{\mac{H}_{f^\prime}}$ within $\delta_2$ of $h$ on the inputs $\{f^\prime(X_i):X_i\in D_C\}$. Writing $\mathbf X=(X_1,\ldots,X_n)$, the triangle inequality gives
\begin{equation*}
\begin{split}
   \|h\circ f(\textbf{X}) - h^\prime \circ f^\prime(\textbf{X})\|_{2} 
   &\leq \|h\circ f(\textbf{X}) - h \circ f^\prime(\textbf{X})\|_{2}  + \|h \circ f^\prime(\textbf{X}) - h^\prime \circ f^\prime(\textbf{X})\|_{2}  \\
   & \leq L(\mac{H})\cdot\delta_1 + \delta_2, 
\end{split}
\end{equation*}
The second inequality uses the uniform Lipschitz property of $\mac H$. Let $\mathbf U:=\{U_i=f(X_i):i=1,\ldots,n\}$. Then $|\mac{C}_{\mac H\circ\mac F}|\leq |\mac{C}_{\mac F}|\max_{f\in\mac{C}_{\mac F}}|\mac{C}_{\mac H_f}|\leq |\mac{C}_{\mac F}|\mac{C}_n(\delta_2,\mac H,\|\cdot\|_2)$.
Using this covering-number bound with $\delta_1=\delta/(2L(\mac H))$ and $\delta_2=\delta/2$, together with the subadditivity of the square root, gives
\begin{equation*}
\begin{split}
     \int_{\delta}^{B_{\mac{H}}} \sqrt{\log \mac{C}(t, \mac{H} \circ \mac{F}, \|\cdot\|_{2})}\, dt &\leq \int_{\delta}^{B_{\mac{H}}} \sqrt{\log \mac{C}(t/{2L(\mac{H})}, \mac{F}, \|\cdot\|_{2})}\, dt \\
     &+\int_{\delta}^{B_{\mac{H}}} \sqrt{\log \mac{C}_n(t/{2}, \mac{H}, \|\cdot\|_{2})}\, dt.
\end{split}
\end{equation*}
The Sudakov minoration theorem for Gaussian processes \citep[Theorem 5.30]{wainwright2019high}, together with the relation between packing and covering numbers at the same scale, gives
\begin{equation*}
    \log \mac{C}_n(\delta, \mac{H}, \|\cdot\|_{2}) \leq 4\bigg(\frac{\sqrt{n} \mathfrak{G}_{\mathbf{U}}(\mac{H})}{\delta}\bigg)^2, \quad \forall \delta >0,
\end{equation*}
where $\mathfrak{G}_{\mathbf{U}}(\mac{H})$ is the Gaussian process indexed by $\mac{H}$ on $\mathbf{U}$, and $h(U) \in [-B_{\mac{H}}, B_{\mac{H}}]$ for $U \in \mathbf{U}$.
Consequently,
\begin{align*}
\mac{R}_{C}(\mac{H} \circ \mac{F}) 
&\leq \frac{12}{\sqrt{n}}\mathbb{E}
    \inf_{0 \leq \delta \leq B_{\mac{H}}}
    \Bigg(4\delta+\int_{\delta}^{B_{\mac{H}}}
      \sqrt{\log \mac{C}(t/{2L(\mac{H})},\mac{F},\|\cdot\|_{2})}\,dt\\
    &\quad+\int_{\delta}^{B_{\mac{H}}}
      \sqrt{\log \mac{C}_n(t/2,\mac{H},\|\cdot\|_{2})}\,dt
    \Bigg).
\end{align*}
By the relation between the covering number and Gaussian process, we have 
\begin{equation}\label{eq_covering}
   \log \mac{C}_n(t/2, \mac{H}, \|\cdot\|_{2}) \leq 4\left(\frac{2\sqrt{n}\mathfrak{G}_{\mathbf{U}}(\mac{H})}{t} \right)^2 \leq 4\left(\frac{4\ln(n)\sqrt{n}\mathfrak{R}_{\mathbf{U}}(\mac{H})}{t} \right)^2,
\end{equation}
where the second inequality holds by \citet[Lemma 4]{bartlett2002rademacher}, and $\mathfrak{R}_{X}(\mac{H})$ is the empirical Rademacher complexity. Similarly, we have 
\begin{equation}\label{eq_coverF}
   \log \mac{C}(t/({2L(\mac{H})}), \mac{F}, \|\cdot\|_{2}) \leq 4\left(\frac{2L(\mac{H})\sqrt{n}\mathfrak{G}_{\mathbf{X}}(\mac{F})}{t} \right)^2 \leq 4\left(\frac{4L(\mac{H})\ln(n)\sqrt{n}\mathfrak{R}_{\mathbf{X}}(\mac{F})}{t} \right)^2.
\end{equation}
Combining Lemmas \ref{lem_elemet_radermachec}-\ref{lem_rademachec_nomalconstrain} with \eqref{eq_rade_covering}, \eqref{eq_covering}, and \eqref{eq_coverF}, we obtain  
\begin{equation*}
    \mac{R}_{C}(\mac{H} \circ \mac{F}) \leq \frac{ \mac{K}_f \mac{K}_h d\sqrt{2(\mac{D}_f + 2 + \log(p+1))}}{\sqrt{n}} + \frac{\mac{K}_h\sqrt{2(\mac{D}_h + 2 + \log(d+1))}}{\sqrt{n}}.
\end{equation*}
\end{proof}
\section{Proof of Lemma \ref{lem_pat3}}\label{prof_pat3}
\begin{proof}
We bound the three terms in the decomposition of $\mac{R}_{\mac{M}_h}(\hat f)-\mac{R}_{\mac{M}_h}(f^*)$ from Appendix~\ref{err_trans_deom_III}.
~~\\
\textbf{For the statistical error $\mac{E}_{sta}$}, we have
\begin{equation*}
\begin{split}
 \sup_{f\in \mac{F}}|\mac{R}_{\mac{H}}(f) - \hat{\mac{R}}_{\mac{H}}(f)|  & = \sup_{f\in \mac{F}}|\sup_{h \in \mac{H}}\hat{\mac{R}}_{ad}(f, h) - \sup_{h \in \mac{H}}\mac{R}_{ad}(f,h)|  \\
 & \leq \sup_{f\in \mac{F}, h \in \mac{H}}|\hat{\mac{R}}_{ad}(f, h) - \mac{R}_{ad}(f,h)|\\
 & = \sup_{f\in \mac{F}, h \in \mac{H}}|\hat{\mac{R}}_{cl}(f) - \mac{R}_{cl}(f) +  \hat{\mac{R}}_{d}(f, h) - \mac{R}_{d}(f,h)|\\
 &\leq \sup_{f\in \mac{F}}|\hat{\mac{R}}_{cl}(f) - \mac{R}_{cl}(f) | +  \sup_{f\in \mac{F}, h \in \mac{H}}|  \hat{\mac{R}}_{d}(f, h) - \mac{R}_{d}(f,h)|.
\end{split}
\end{equation*}
Using Lemmas \ref{lem_sta_f} and \ref{lem_sta_h}, we bound the statistical error as follows: 
\begin{equation*}
\begin{split}
\mathbb{E}\sup_{f\in \mac{F}}|\mac{R}_{\mac{H}}(f) - \hat{\mac{R}}_{\mac{H}}(f)| &\lesssim  d \mathcal K_f
\sqrt{\frac{\mathcal D_f+2+\log(p+1)}{n}} + \mathcal K_f\mathcal K_h
\sqrt{\frac{\mathcal D_f+\mathcal D_h+2+\log(p+1)}{n}}.\\
\end{split}
\end{equation*}
~~\\
\textbf{For the approximation error $\mac{E}_{h}$}, we have
\begin{equation*}
\begin{split}
&\sup_{f\in \mac{F}}|\mac{R}_{\mac{H}}(f) - \mac{R}_{\mac{M}_h}(f)|\\
&=  \sup_{f\in \mac{F}}|\sup_{h_2 \in \mac{M}_h} \mac{R}_{ad}(f, h_2) - \sup_{h_1 \in \mac{H}}\mac{R}_{ad}(f, h_1) |  \\
&=\sup_{f\in \mac{F}}|\sup_{h_2 \in \mac{M}_h}\inf_{h_1 \in \mac{H}} \left\{\mac{R}_{ad}(f, h_2) -\mac{R}_{ad}(f, h_1) \right\}|\\
&=\sup_{f\in \mac{F}}|\sup_{h_2 \in \mac{M}_h}\inf_{h_1 \in \mac{H}} \left\{\mac{R}_{cl}(f) +\mac{R}_{d}(f, h_2) - (\mac{R}_{cl}(f) +\mac{R}_{d}(f, h_1))\right\}|\\
&\leq \sup_{f\in \mac{F}}|\sup_{h_2 \in \mac{M}_h}\inf_{h_1 \in \mac{H}} \left\{\mac{R}_{d}(f, h_2) - \mac{R}_{d}(f, h_1)\right\}|\\
&=  \sup_{f\in \mac{F}}|\sup_{h_2 \in \mac{M}_h}\inf_{h_1 \in \mac{H}}\mab{E}_{P_{\mu}}[\mac{\ell}_{d}(f, h_2) - \mac{\ell}_{d}(f, h_1)]|\\
&\lesssim   \sup_{f\in \mac{F}}\sup_{h_2 \in \mac{M}_h}\inf_{h_1 \in \mac{H}} \mab{E}_{P_C}[|h_2\circ f(X) - h_1\circ f(X)|] \quad \text{(Lipschitz  continuity)}\\
&\leq \sup_{h_2 \in \mac{M}_h}\inf_{h_1 \in \mac{H}} \mab{E}_{P_C(U)}[|h_2(U) - h_1(U)|]\\
&\leq \sup_{h_2 \in \mac{M}_h}\inf_{h_1 \in \mac{H}} \|h_2 - h_1\|_{\infty}.
\end{split}
\end{equation*}
Lemma~\ref{lem_norm_constrain_appro} bounds this approximation error:
\begin{equation*}
 \sup_{h_2 \in \mac{M}_h}\inf_{h_1 \in \mac{H}} \|h_2 - h_1\|_{\infty} \leq  \mac{K}_h^{-\beta_h /(d+1)}.  
\end{equation*}
~~\\
\textbf{For the approximation error of $\mac{E}_f$}, we have
\begin{equation*}
\begin{split}
& \inf_{f \in \mac{F}}\mac{R}_{\mac{M}_h}(f) - \mac{R}_{\mac{M}_h}(f^*)\\
&\leq \inf_{f \in \mac{F}} \sup_{h \in \mac{M}_h} \big(\mac{R}_{ad}(f, h) - \mac{R}_{ad}(f^*, h) \big)\\
 & \leq \inf_{f \in \mac{F}} \sup_{h \in \mac{M}_h} |\mac{R}_{cl}(f) +  \mac{R}_{d}(f, h)  - (\mac{R}_{cl}(f^*) +  \mac{R}_{d}(f^*, h))| \\
 &\leq  \inf_{f \in \mac{F}} \sup_{h \in \mac{M}_h} \bigg[ |\mac{R}_{cl}(f)  - \mac{R}_{cl}(f^*) | +   |\mac{R}_{d}(f,h)  - \mac{R}_{d}(f^*,h)| \bigg]\\
 &= \inf_{f \in \mac{F}} \sup_{h \in \mac{M}_h}  \bigg[|\mab{E}_{P_C\times P_A}[\ell_{cl}(f) - \ell_{cl}(f^*)]| +   |\mab{E}_{P_{\mu}}[\ell_{d}(f,h)  - \ell_{d}(f^*,h)]| \bigg]\\
&\leq c \inf_{f \in \mac{F}} \mab{E}_{P_C}[\|f(X) - f^*(X)\|]  \quad \text{(Lipschitz continuity of $\ell_{cl}, \ell_{d}$ and $h$)} \\
 &\lesssim  \inf_{f \in \mac{F}}  \|f-f^*\|_{\infty}.
\end{split}
\end{equation*}
Applying Lemma~\ref{lem_norm_constrain_appro} gives
\begin{equation*}
   \inf_{f \in \mac{F}} \sup_{h \in \mac{M}_h} \|f-f^*\|_{\infty} \lesssim  \mac{K}_f^{-\beta_f /(p+1)}.
\end{equation*}    
\end{proof}
\section{Auxiliary Results}\label{Au_results}
\begin{lemma}[Corollary 3.1 of \citet{jiao2021deep}]
\label{lem_appro}
For any $N, M \in \mathbb{N}^+$, there exists a function $g \in \mathcal{G}$ with width $\mathcal{W}=38(\lfloor \beta_{g} \rfloor +1)^23^dd^{\lfloor \beta_{g} \rfloor +1}N \lceil \log_{2}(8N) \rceil$ and depth $\mathcal{D} = 21(\lfloor \beta_{g}  \rfloor +1)^2M \lceil \log_{2}(8M) \rceil+2d$, such that
\begin{equation*}
    |g(x)  - g^*(x)| \leq 19B_{g}(\lfloor \beta_{g}  \rfloor +1)^2d^{\lfloor \beta_{g}  \rfloor + (\beta_{g} \vee 1)/2}(MN)^{-2\beta_{g}/d},
\end{equation*}
for all $x \in [0,1]^d$.    
\end{lemma}
\begin{lemma}[Corollary 3.2 of \citet{jiao2023approximation}]\label{lem_norm_constrain_appro}
Let $d\in \mab{N}$ and $\beta_f = r + \alpha >0$, where $r \in \mab{N}_0, \alpha \in (0, 1]$. There exists $c >0$ such that for any $\mac{K} \geq 1$, any $\mac{W} \geq c\mac{K}^{(2d+\beta_f)/(2d +2)}$ and $\mac{D} \geq 2\lceil \log_{2}(d+r)\rceil +2$,
\begin{equation*}
    \mac{E}(\mac{H}^{\alpha}, \mac{NN}(\mac{W}, \mac{D}, \mac{K})) \lesssim \mac{K}^{-\beta_f/(d+1)}.
\end{equation*}
\end{lemma}
\begin{lemma}[Corollary 4 of \citet{maurer2016vector}]\label{lem_vector_rademacherc}
Let $\mathcal V$ be a class of maps from a set $\mathcal Z$ to $\mathbb R^q$, and let $\ell:\mathbb R^q\to\mathbb R$ be $L$-Lipschitz in the Euclidean norm. For fixed $z_1,\ldots,z_n$,
\[
\mathbb E_\epsilon\sup_{v\in\mathcal V}\sum_{i=1}^n\epsilon_i\ell(v(z_i))
\leq\sqrt2 L\,\mathbb E_{\epsilon_{it}}\sup_{v\in\mathcal V}
\sum_{i=1}^n\sum_{t=1}^q\epsilon_{it}v_t(z_i).
\]
\end{lemma}
For the contrastive application, use $q=3d$ and
$v_f(z)=(f(x),f(x^+),f(x^-))$; 
\begin{lemma}[Coordinatewise extension of the scalar row-sum-norm bound]\label{lem_elemet_radermachec}
Let $z_1,\ldots,z_n\in[-B,B]^p$, $B_0=\max\{B,1\}$, and let $\mathcal F$ be a $d$-output class of networks with the norm constraint in Definition~\ref{def_norm_network}. Then
\[
\mathbb E_\epsilon \sup_{f\in\mathcal F} \frac{1}{n}
\sum_{i=1}^n\sum_{t=1}^d\epsilon_{it}f_t(z_i)
\leq dB_0\mathcal K_f
\sqrt{\frac{2(\mathcal D_f+2+\log(p+1))}{n}}.
\]
\end{lemma}
\begin{proof}
Let $\mathcal F_t=\{f_t:f\in\mathcal F\}$. Selecting one output row does not increase the maximum row-sum norm, so each $\mathcal F_t$ is contained in the corresponding scalar network class. Thus
\[
\mathbb E_\epsilon  \sup_f \frac{1}{n}\sum_t\sum_i\epsilon_{it}f_t(z_i)
\leq\sum_t\mathbb E_\epsilon\sup_{g\in\mathcal F_t}
\frac{1}{n}\sum_i\epsilon_{it}g(z_i),
\]
by Lemma~\ref{lem_rademachec_nomalconstrain}, \citet{golowich2020size} and \citet{jiao2023approximation}, we obtain the stated bound.
\end{proof}
\begin{lemma}[Theorem 3.2 of \citet{golowich2020size}]\label{lem_rademachec_nomalconstrain}
 For any $x_1, \dots, x_n \in [-B, B]^p$ with $1 \leq B$, let $\mac{F}$ be the class of norm-constrained DNNs $\mac{NN}(\mac{W,D,K})$, then 
 \begin{equation*}
     \hat{\mathfrak{R}}_{n}(\mac{F}) \leq \frac{1}{n}\mac{K}\sqrt{2(\mac{D} + 2 + \log(p +1))}\max_{1 \leq j \leq p+1}\sqrt{\sum_{i = 1}^n x_{i,j}^2} \leq \frac{B\mac{K}\sqrt{2(\mac{D} + 2 + \log(p +1)})}{\sqrt{n}}.
 \end{equation*}
\end{lemma}
\section{Experiment Details}\label{sup_expe}
\textbf{Datasets and Processes.}
We use Office-31 \citep{saenko2010adapting}, Office-Home \citep{venkateswara2017deep}, and DomainNet \citep{peng2019moment}. 
For the reported protocol, each domain is randomly partitioned into an $80\%$ training split and a $20\%$ held-out test split. 
Representation pretraining uses images from the source and target training splits without target class labels, and downstream predictor 
fitting uses labels only from the source training split. The target test split is reserved for final evaluation. 
We apply the same split-based protocol to every compared method, enabling controlled comparisons across learning objectives and architectures.
\begin{table}[H]
    \centering
    \caption{The dataset information for domain adaptation experiments.}
    \label{data_set}
    \begin{tabular}{lccc}
        \toprule
        Data & Domain & Class  & Total sample   \\ \midrule
        Office-31  &3  & 31  & 4,625    \\
        Office-Home  &4  &65  & 15,588    \\
        DomainNet  &6  &345  & 586,753    \\
         \bottomrule
    \end{tabular}
\end{table}
We then combine the training data from both the source and target domains and apply a series of augmentations for contrastive learning. The augmentation procedures follow the same setup as \citep{chen2020simple}, including:
\begin{itemize}
\item Randomly cropping the image to a size of $224 \times 224$ pixels.
\item Randomly applying a horizontal flip to the image with a probability of $0.5$.
\item Randomly applying color jitter to the image with a probability of $0.8$.
\item Converting the image to grayscale with a probability of $0.2$.
\item Applying a custom Gaussian blur with a specified kernel size.
All of the above augmentations are applied to two randomized copies of the original image $x$ to generate a positive pair consisting of two augmented views, denoted by $x_{(1)}^+$  and $x_{(2)}^+$, respectively. Other images within the same training mini-batch serve as negative samples.
\end{itemize}
\textbf{Model Architectures.}
We use ResNet \citep{he2016deep} and ViT \citep{dosovitskiy2020image} backbones. 
The domain discriminator and downstream predictor are fully connected networks with ReLU activations, 
and a shallow perceptron connects the backbone to the discriminator. During Stage 1, 
a mini-batch contains unlabeled source and target training examples. Their features are used by both the contrastive 
loss and the domain discriminator. The discriminator minimizes its domain-classification loss, while a gradient-reversal 
layer sends the reversed domain gradient to the feature extractor; the contrastive gradient is propagated normally. 
During Stage 2, the feature extractor is frozen and only the downstream source-supervised predictor is fitted. 
DANN also uses labeled source and unlabeled target examples, but couples source classification with adversarial alignment; 
our Stage 1 replaces source classification with joint-domain self-supervised contrastive learning.
\begin{figure}[ht]
    \centering
    \includegraphics[width=0.7\linewidth]{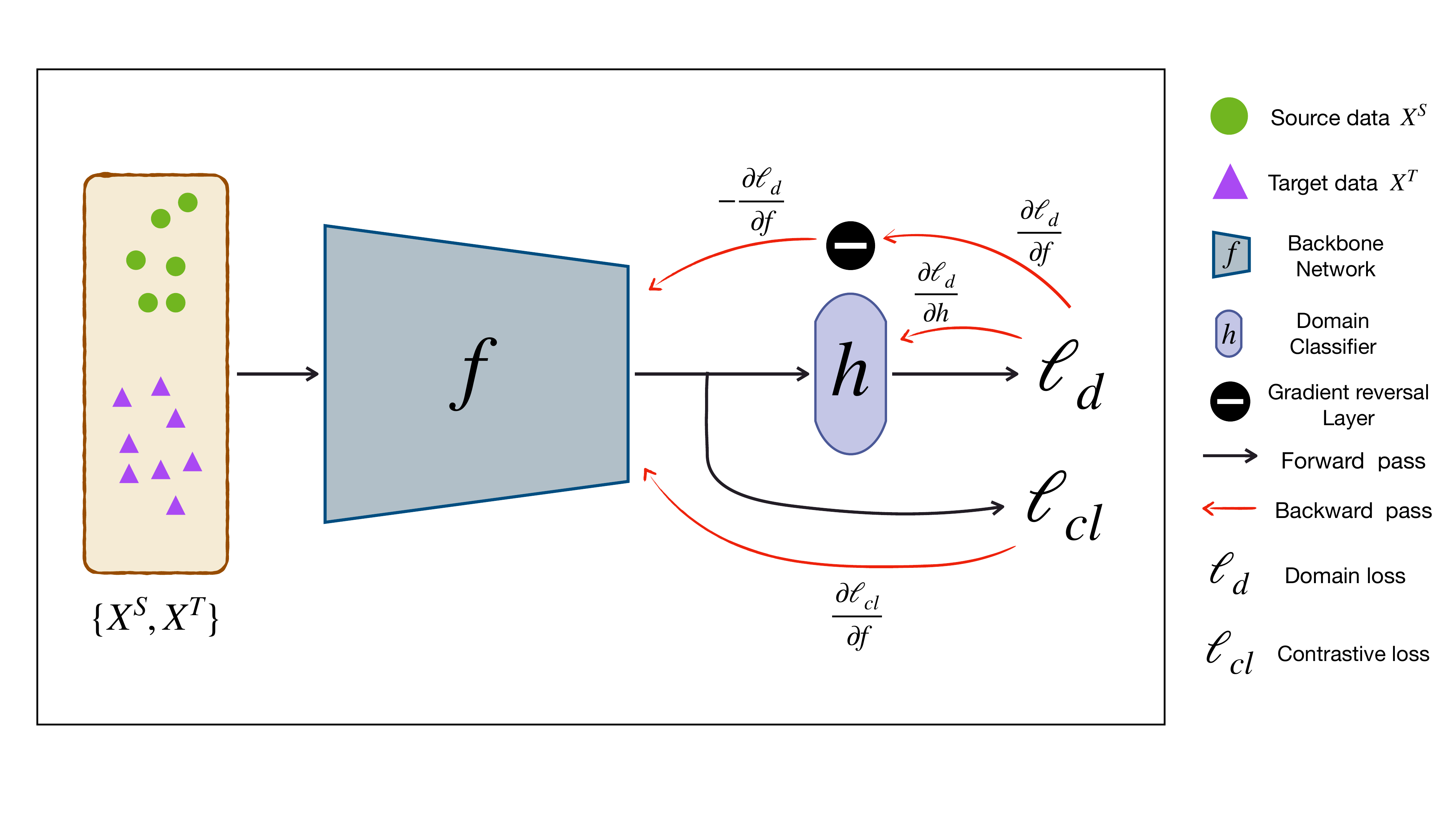}
    \caption{The framework for representation learning.}
    \label{frame_work}
\end{figure}
\textbf{Hyperparameter Settings.}
Table \ref{hyper} summarizes the hyperparameter settings used in our experiments. For each dataset, we train the backbone network for $800$ epochs using stochastic gradient descent (SGD) as the optimizer. The initial learning rates are set to $0.01$, $0.1$, and $0.1 (1.0)$ for the backbone, bottleneck, and discriminator, respectively. The value in parentheses indicates that the learning rate for the discriminator is set to $1.0$ when training without pretrained weights. The batch size is set to $16$ to accommodate hardware constraints and is reduced to $8$ when training without pretrained weights. We set the feature dimension to $512$ for Office-Home and DomainNet, which is larger than that used for Office-31.
\begin{table}[h]
    \caption{Hyperparameter settings}
    \label{hyper}
    \centering
    \begin{tabular}{lccccc}
    \toprule
        Data & Epoch & Batch Size  & Learning rate & Feature dimension \\ \midrule
         Office-31&     800&    16(8)&   0.01, 0.1, 0.1(1.0)&  256 &   \\
         Office-Home&     800&   16(8)&   0.01, 0.1, 0.1(1.0)&   512& \\
         DomainNet&     800&   16(8)&   0.01, 0.1, 0.1(1.0)&   512& \\
         \bottomrule
    \end{tabular}
\end{table}
\textbf{Baselines.}
We use ERM, DANN \citep{ganin2015unsupervised}, and CDAN \citep{long2018conditional} as reference baselines, using their implementations in the Transfer Learning Library \citep{tllib}. 
Together, these methods distinguish the contribution of source-only empirical risk minimization from adversarial and conditional 
domain adaptation. ViT-B \citep{dosovitskiy2020image} provides an architecture-matched reference for the ViT experiments.
\section{Additional experiment results}\label{add results}
Tables \ref{tab_compare_off31} and \ref{tab_compare_off31M} report pretrained results on Office-31; Tables \ref{tab_compare_offHom} 
and \ref{tab_compare_offHomM} report Office-Home; and Table \ref{tab_compare_domain} reports DomainNet. 
These experiments initialize the backbone from supervised pretraining on a large labeled dataset such as 
ImageNet \citep{krizhevsky2012imagenet}, testing JCAR in a complementary regime to the from-scratch evaluation. On Office-31, 
our ViT variant remains within three percentage points of ViT-B on five of six single-source transfers. 
ViT-B is nevertheless higher across the pretrained single-source tables, indicating that JCAR's strongest 
empirical advantage currently arises when the representation is learned from scratch and that integrating 
joint contrastive pretraining with supervised initialization is an important direction for further improvement.
\begin{table*}[ht]
\caption{The Top-1 accuracy (\%) comparison of proposed method with others on the Office-31 dataset using pretrained model weights.}
\label{tab_compare_off31}
\vskip 0.15in
\begin{center}
\begin{small}
\begin{sc}
\begin{tabular}{lcccccc}
    \toprule
        Methods              &A $\to$ W      &D $\to$ W   &W $\to$ D  & A $\to$ D  &D $\to$ A  &W $\to$ A \\ \midrule
        ERM       & 74.8      &95.0      &93.9   &69.7   &43.2   &55.4  \\ 
         CDAN       & 61.6      &95.0      &97.0   &66.7   &53.3   &52.8 \\
        DANN      & 71.7      &95.0     &97.0   & 70.7  &52.8   &53.5 \\ 
        ViT-B      &95.0       &99.2      &99.9   &93.9   &81.8   &81.4 \\
        Ours (ViT)      &92.0       &98.9      & 99.5    & 82.0    &79.1     & 79.8 \\
    \bottomrule    
    \end{tabular}
\end{sc}
\end{small}
\end{center}
\vskip -0.1in
\end{table*}
\begin{table}[H]
\caption{The Top-1 accuracy (\%) comparison of proposed method with others on the Office-31 dataset using pretrained model weights.}
\label{tab_compare_off31M}
\begin{center}
\begin{small}
\begin{tabular}{lccc}
    \toprule
Methods     & \{D,W\} $\to$ A         &\{A, D \} $\to$ W   &\{A, W\}$\to$ D   \\ \midrule
        ERM     & 56.5         & 91.8     &93.9     \\
         CDAN   &  56.7        & 92.5     & 89.9    \\
        DANN    &56.5          &81.1      & 90.9    \\ 
         ViT-B   & 81.4       & 99.2    & 100.0    \\
        Ours (ViT)   &78.5        & 98.4    & 98.9    \\
    \bottomrule    
    \end{tabular}
\end{small}
\end{center}
\vskip -0.1in
\end{table}
\begin{table}[H]
\caption{The Top-1 accuracy (\%) comparison of proposed method with others on the Office-Home dataset using pretrained model weights.}
\label{tab_compare_offHom}
\begin{center}
\begin{small}
\begin{tabular}{lcccccc}
    \toprule
        Methods              &Ar $\to$ Rw      &Pr $\to$ Rw   &Cl $\to$ Pr  & Ar $\to$ Cl  &Pr $\to$ Ar  &Cl $\to$ Rw \\ \midrule
        ERM       & 66.5      & 63.5     &56.6   &39.5   & 41.2  &55.3 \\ 
         CDAN      & 64.3      & 65.4     &54.3   &36.4   &41.6   &50.2 \\
        DANN       & 65.0      &67.3      &57.4   &39.1   & 41.2  &51.7 \\ 
        ViT-B      & 86.8      & 87.8     &82.8   &62.3   &78.2   &85.6 \\
        Ours (ViT)      &84.5      & 82.9     & 76.2  & 58.5  & 67.0  &78.9  \\
    \bottomrule    
    \end{tabular}
\end{small}
\end{center}
\vskip -0.1in
\end{table}
\begin{table}[H]
\caption{The Top-1 accuracy (\%) comparison of proposed method with others on the Office-Home dataset using pretrained model weights.}
\label{tab_compare_offHomM}
\begin{center}
\begin{small}
\begin{tabular}{lcccc}
    \toprule
 \multirow{3}{*}{Methods}     & \{Cl,Pr, Rw\}   &\{Ar, Rw, Pr \}    &\{Ar, Cl, Pr \}  &\{Ar, Cl, Rw\}   \\ 
~ &  $\downarrow$  & $\downarrow$      & $\downarrow$   & $\downarrow$  \\ 
~ &  Ar  & Cl      & Rw   & Pr   \\ \midrule
        ERM     & 63.3    &46.8     &77.7      &76.1     \\
        CDAN  &67.4     &51.5     &77.2      &73.8     \\
        DANN   &62.7    &48.3      & 75.5     &73.3     \\ 
        ViT-B   & 83.7   &64.4    & 90.0    & 89.6    \\
        Ours (ViT)   &78.6    &71.6    & 88.5    &  88.0   \\
    \bottomrule    
    \end{tabular}
\end{small}
\end{center}
\vskip -0.1in
\end{table}
\begin{table}[H]
\caption{The Top-1 accuracy (\%) comparison of proposed method with others on the DomainNet dataset using pretrained model weights.}
\label{tab_compare_domain}
\begin{center}
\begin{small}
\begin{tabular}{lcccccc}
    \toprule
        Methods              &r $\to$ c      &r $\to$ p   &p $\to$ c  &c $\to$ s  &s $\to$ p  &r $\to$ s \\ \midrule
        ERM       &36.5       &40.7      &29.4   &32.1   & 27.5  &26.2 \\
         CDAN       &51.4       &50.1      &41.1   &40.1   &41.0   &40.7 \\
        DANN     & 48.9      & 48.8     &35.9   & 41.1  &40.2   & 39.8\\ 
        ViT-B      & 47.4      & 53.5     &37.8   &44.7   &47.1   &34.8 \\
        Ours (ViT)      & 44.2      &47.8      &  30.9 & 31.2  &33.7   &34.7 \\
    \bottomrule    
    \end{tabular}
\end{small}
\end{center}
\vskip -0.1in
\end{table}
\begin{figure}[H]
    \centering
    \begin{minipage}[b]{0.22\textwidth}
        \centering
        \includegraphics[width=\textwidth]{figures/tsne/res_office31/W--D.pdf}
    \end{minipage}
    \hfill
    \begin{minipage}[b]{0.22\textwidth}
        \centering
        \includegraphics[width=\textwidth]{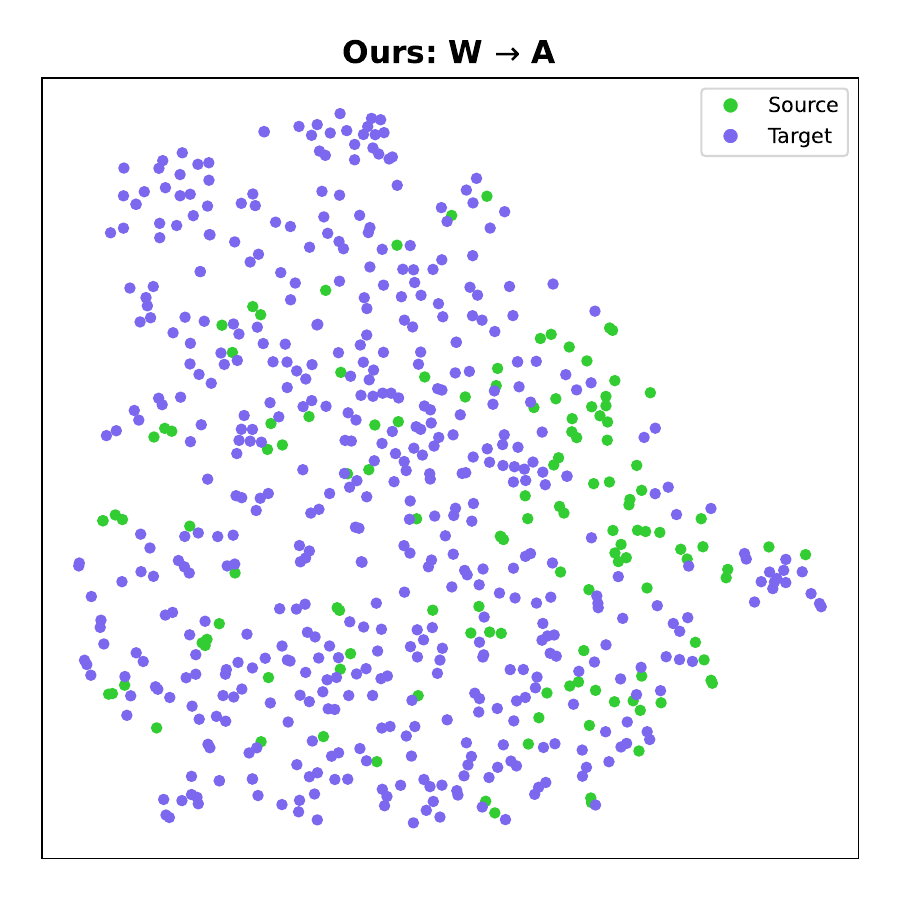}
    \end{minipage}
    \hfill
    \begin{minipage}[b]{0.22\textwidth}
        \centering
        \includegraphics[width=\textwidth]{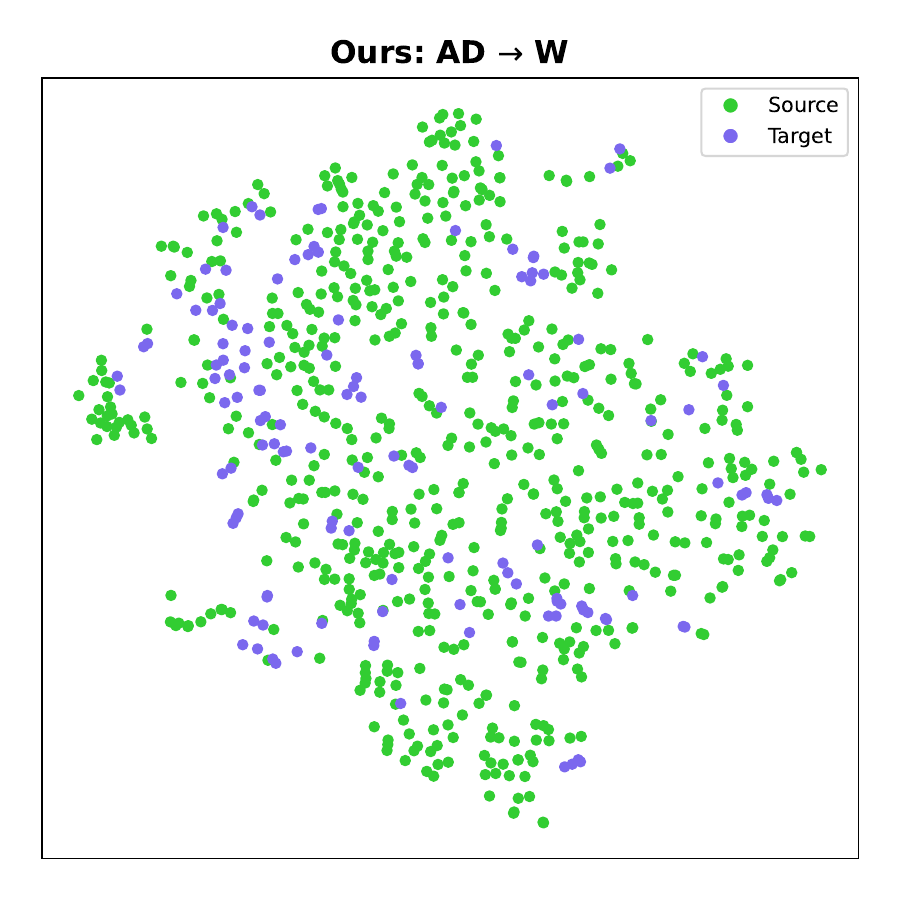}
    \end{minipage}
    \hfill
    \begin{minipage}[b]{0.22\textwidth}
        \centering
        \includegraphics[width=\textwidth]{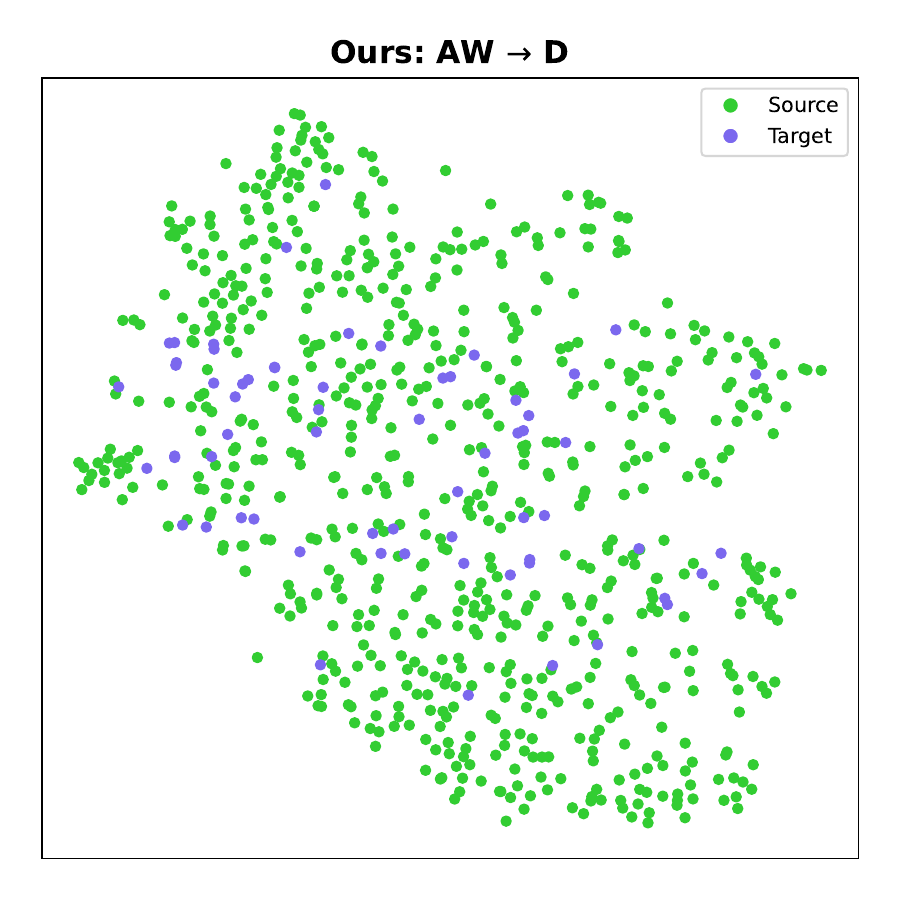}
    \end{minipage}
    \begin{minipage}[b]{0.22\textwidth}
        \centering
        \includegraphics[width=\textwidth]{figures/tsne/res_office31/DAN_W--D.pdf}
    \end{minipage}
    \hfill
    \begin{minipage}[b]{0.22\textwidth}
        \centering
        \includegraphics[width=\textwidth]{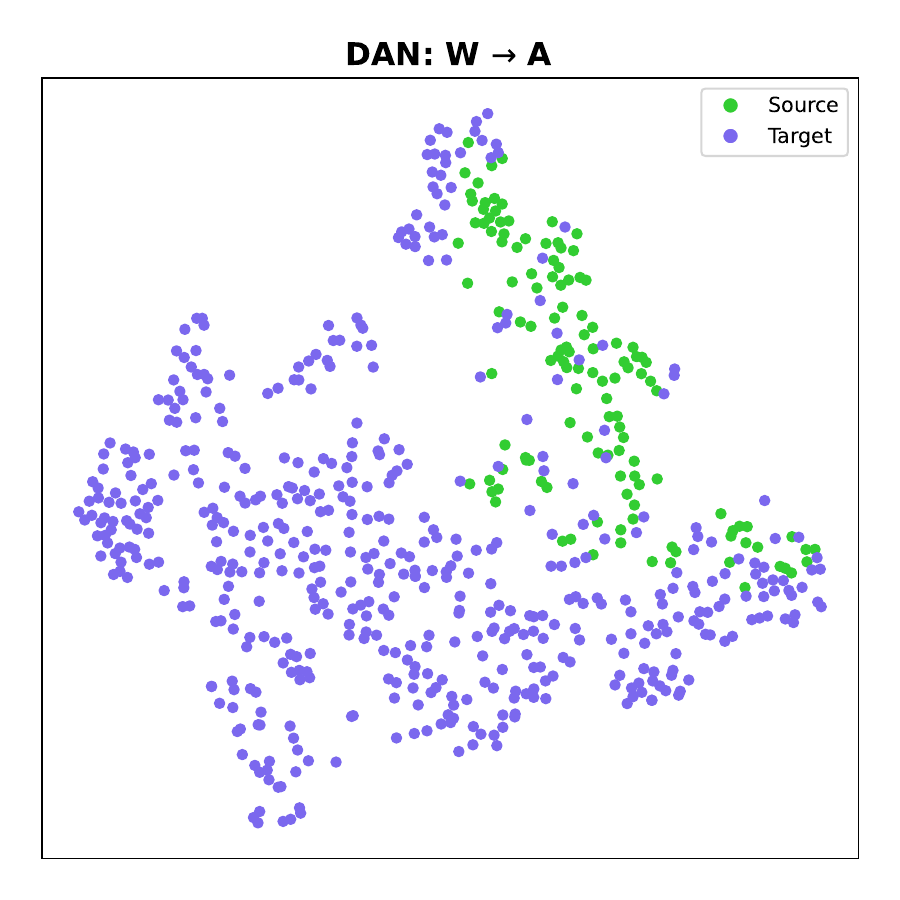}
    \end{minipage}
    \hfill
    \begin{minipage}[b]{0.22\textwidth}
        \centering
        \includegraphics[width=\textwidth]{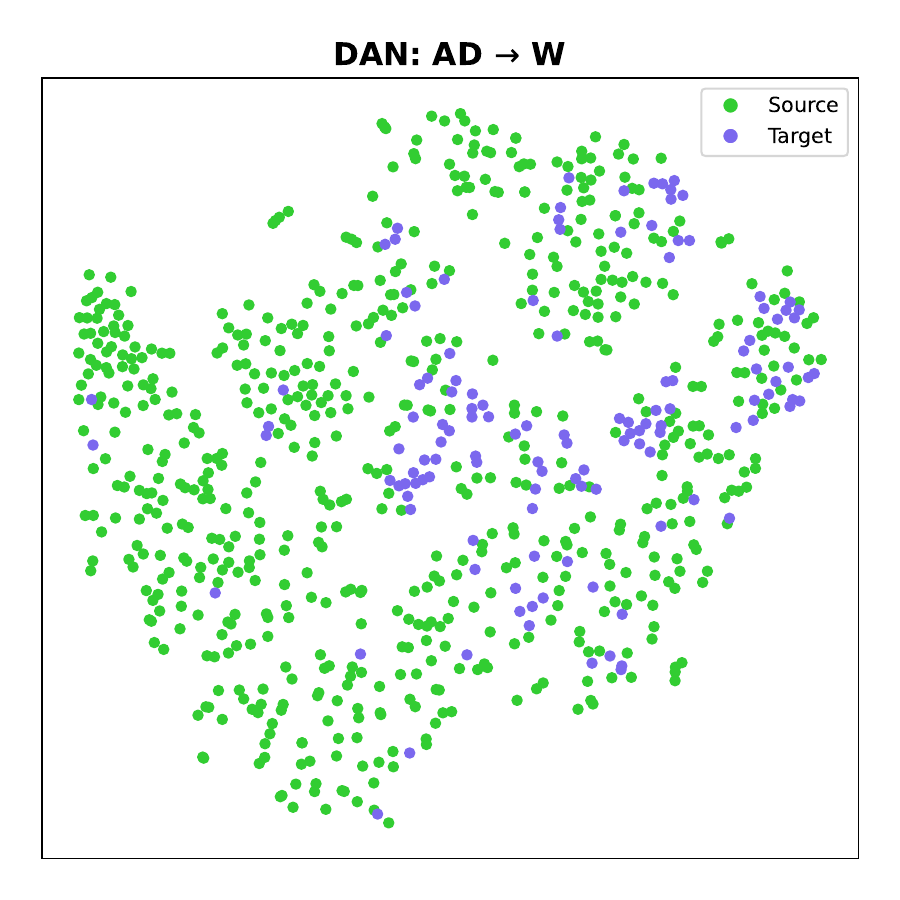}
    \end{minipage}
    \hfill
    \begin{minipage}[b]{0.22\textwidth}
        \centering
        \includegraphics[width=\textwidth]{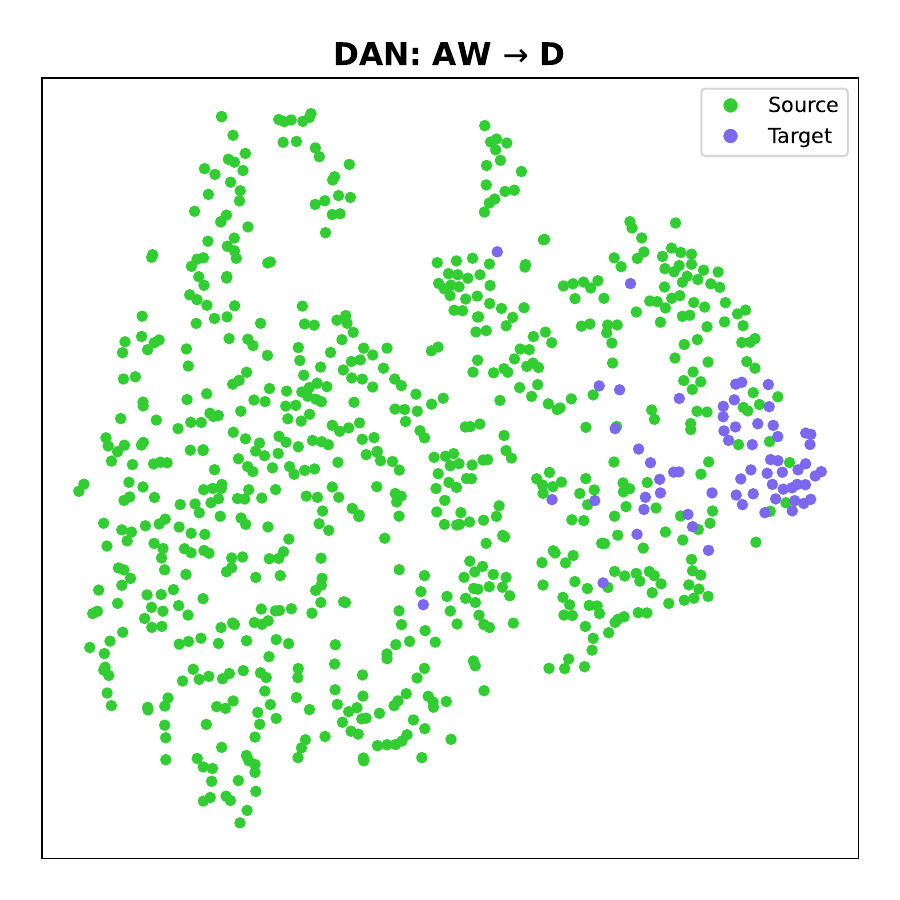}
    \end{minipage}
    \begin{minipage}[b]{0.22\textwidth}
        \centering
        \includegraphics[width=\textwidth]{figures/tsne/res_office31/ERM_W--D.pdf}
    \end{minipage}
    \hfill
    \begin{minipage}[b]{0.22\textwidth}
        \centering
        \includegraphics[width=\textwidth]{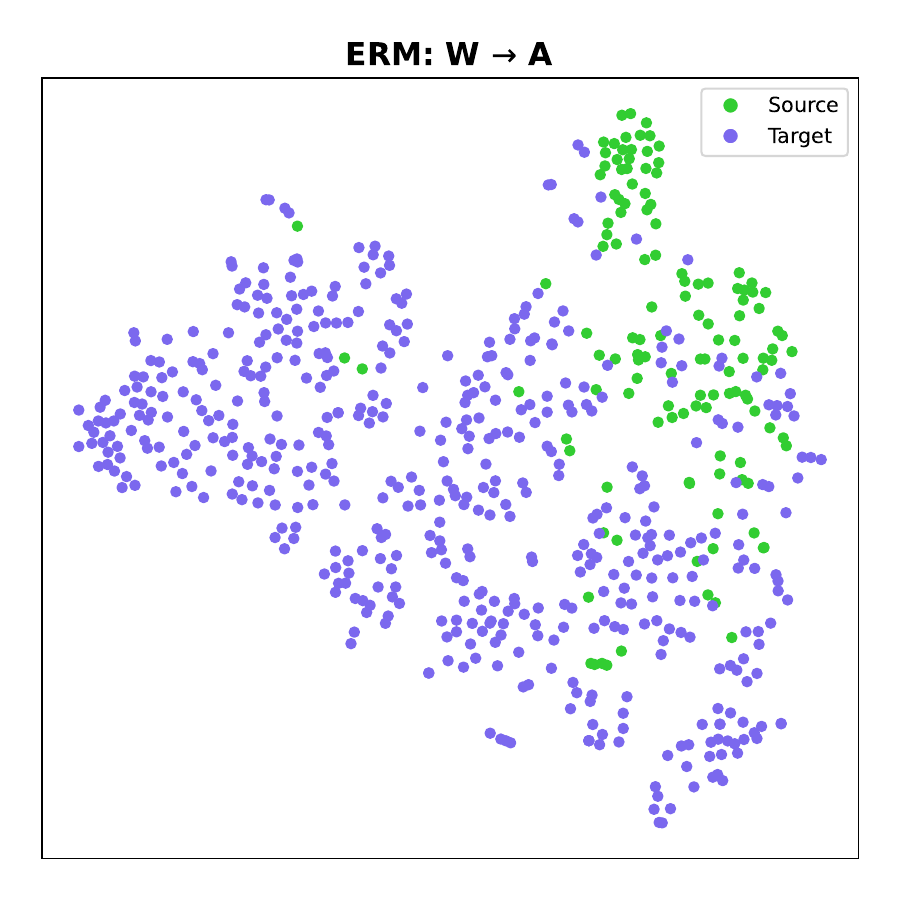}
    \end{minipage}
    \hfill
    \begin{minipage}[b]{0.22\textwidth}
        \centering
        \includegraphics[width=\textwidth]{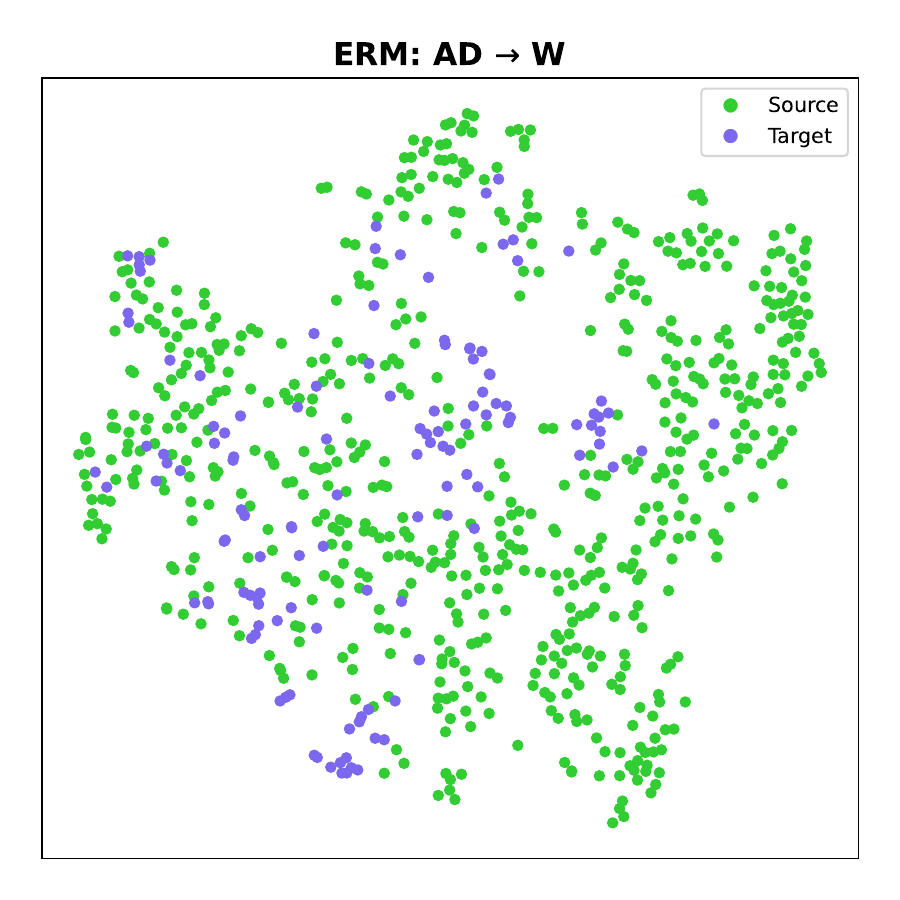}
    \end{minipage}
    \hfill
    \begin{minipage}[b]{0.22\textwidth}
        \centering
        \includegraphics[width=\textwidth]{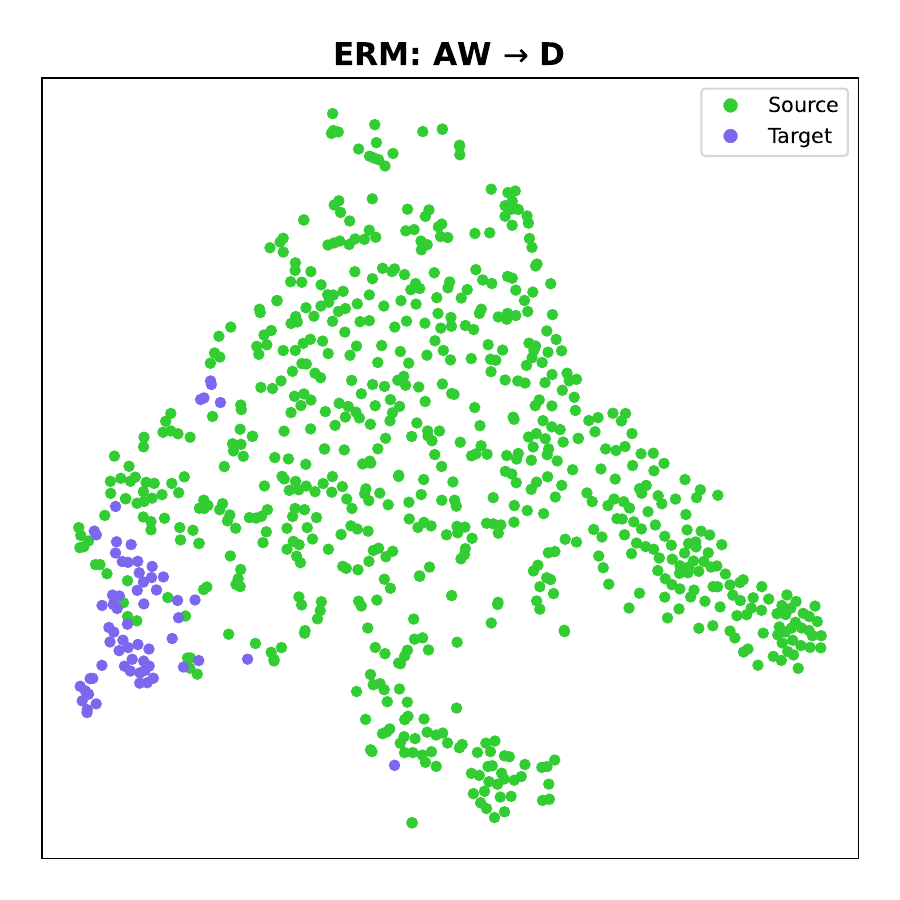}
    \end{minipage}
    \begin{minipage}[b]{0.22\textwidth}
        \centering
        \includegraphics[width=\textwidth]{figures/tsne/ViTgrl_Office31/ViTgrl_W--D.pdf}
    \end{minipage}
    \hfill
    \begin{minipage}[b]{0.22\textwidth}
        \centering
        \includegraphics[width=\textwidth]{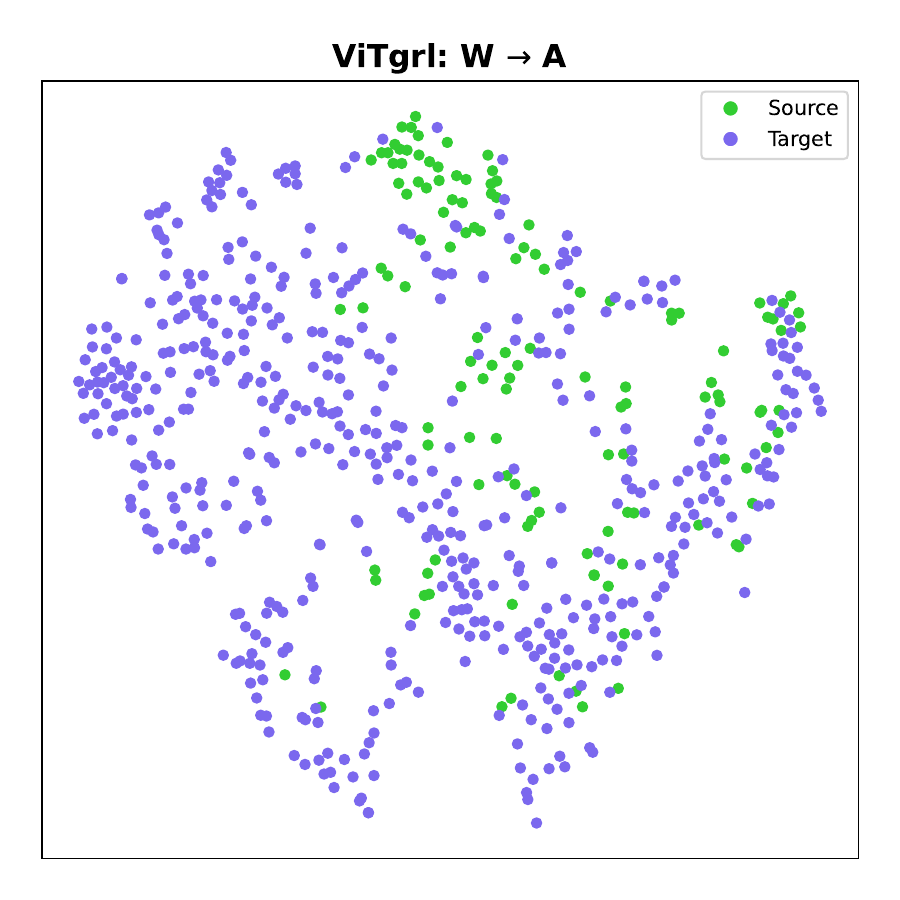}
    \end{minipage}
    \hfill
    \begin{minipage}[b]{0.22\textwidth}
        \centering
        \includegraphics[width=\textwidth]{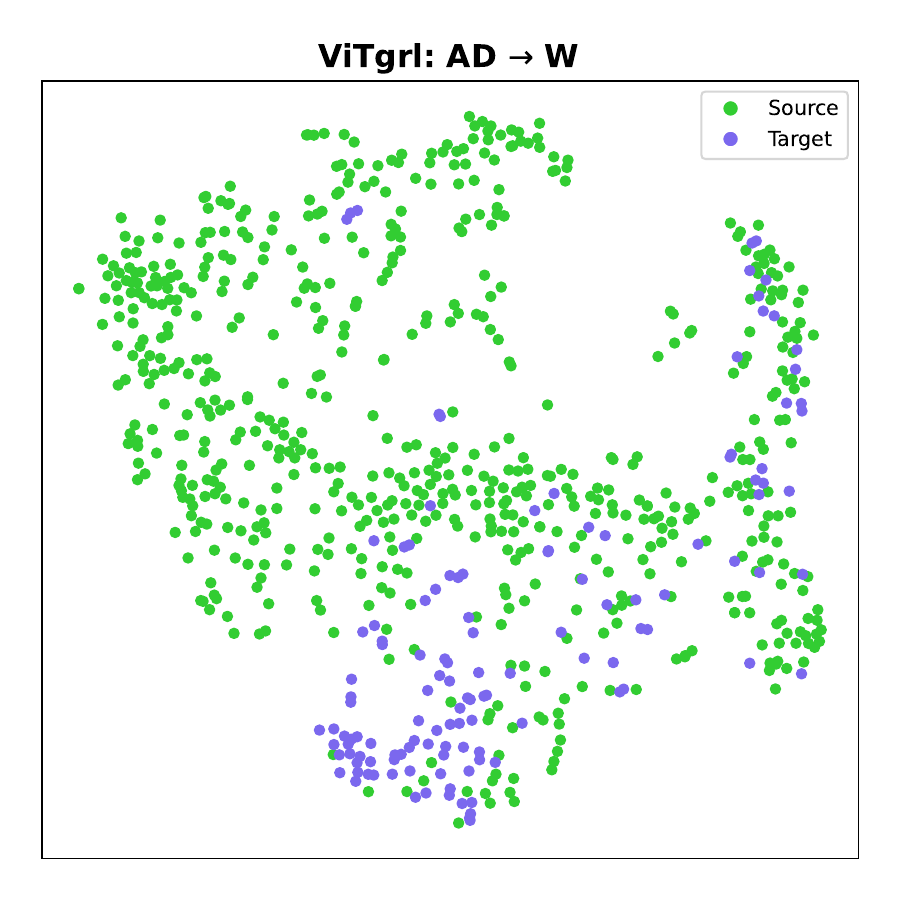}
    \end{minipage}
    \hfill
    \begin{minipage}[b]{0.22\textwidth}
        \centering
        \includegraphics[width=\textwidth]{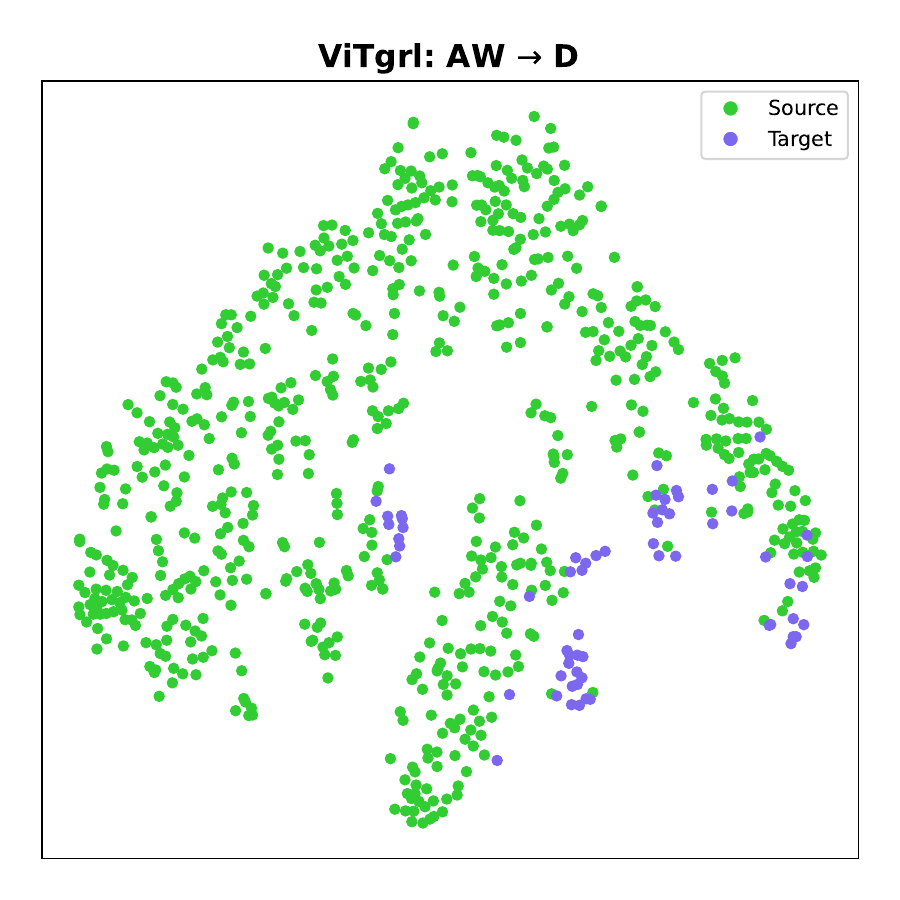}
    \end{minipage}
    \caption{The t-SNE plots for the representation of the source and target domains in Office-31, the top row is our method, and other rows are the baselines. The first two columns correspond to single-source domain scenarios, and the remaining two correspond to multi-source domains.}
    \label{tsen-office31}
\end{figure}
\begin{figure}[H]
    \centering
    \begin{minipage}[b]{0.22\textwidth}
        \centering
        \includegraphics[width=\textwidth]{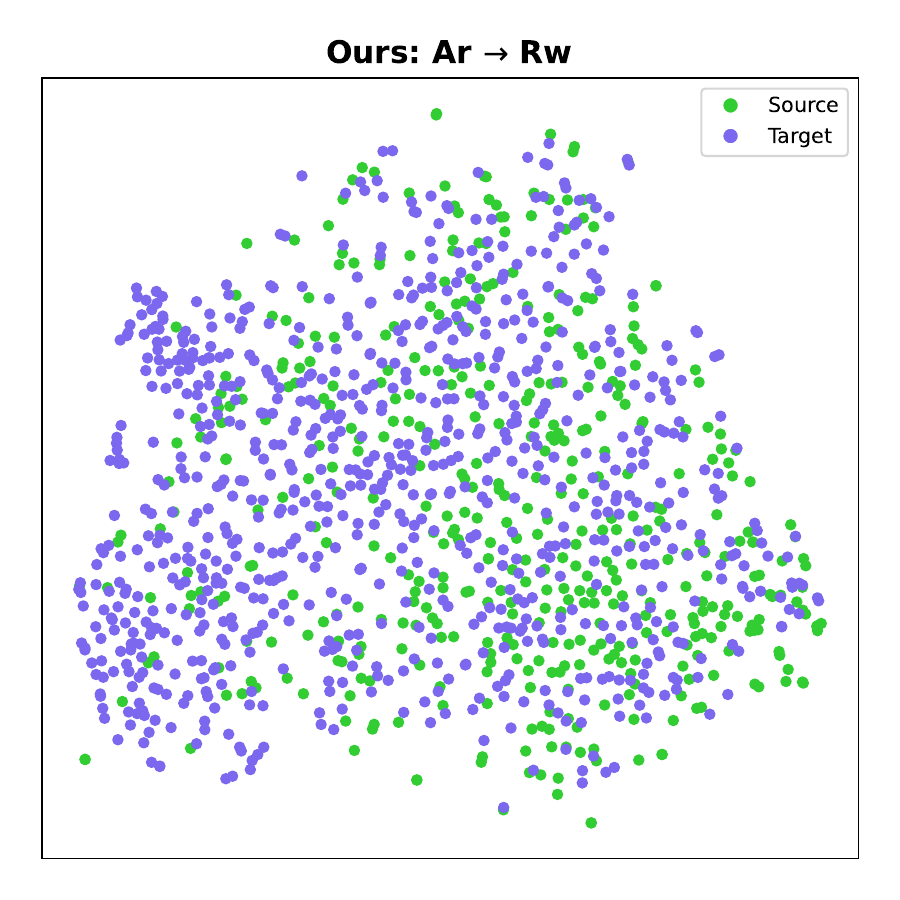}
    \end{minipage}
    \hfill
    \begin{minipage}[b]{0.22\textwidth}
        \centering
        \includegraphics[width=\textwidth]{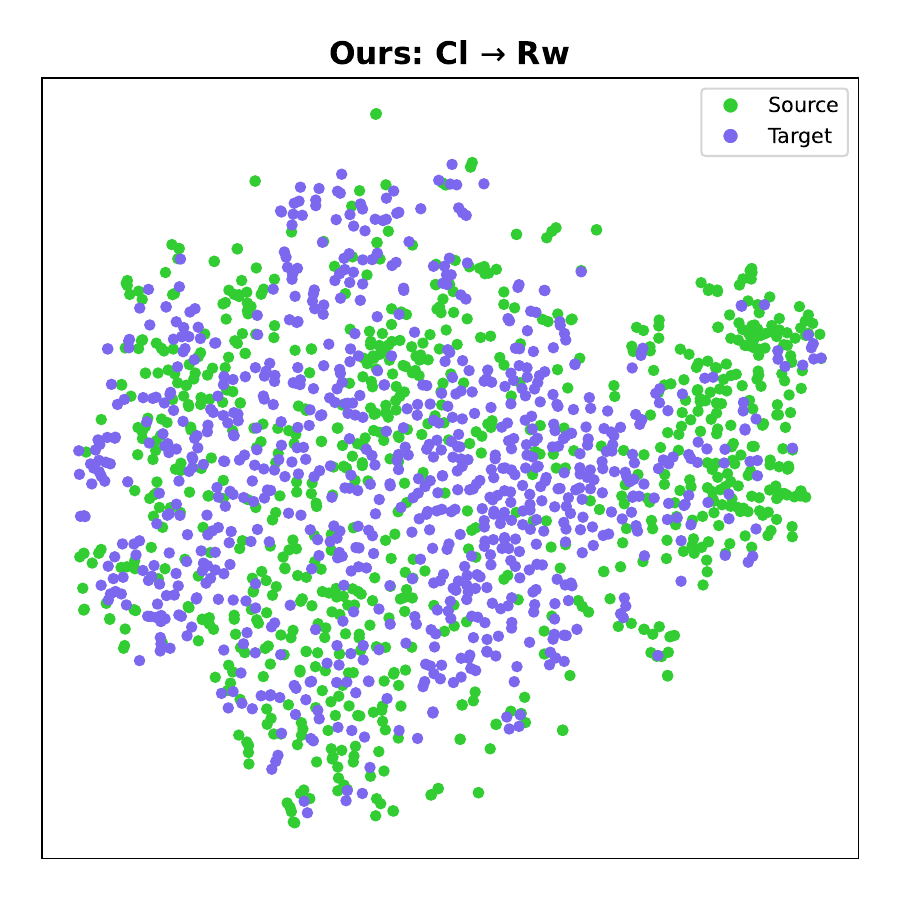}
    \end{minipage}
    \hfill
    \begin{minipage}[b]{0.22\textwidth}
        \centering
        \includegraphics[width=\textwidth]{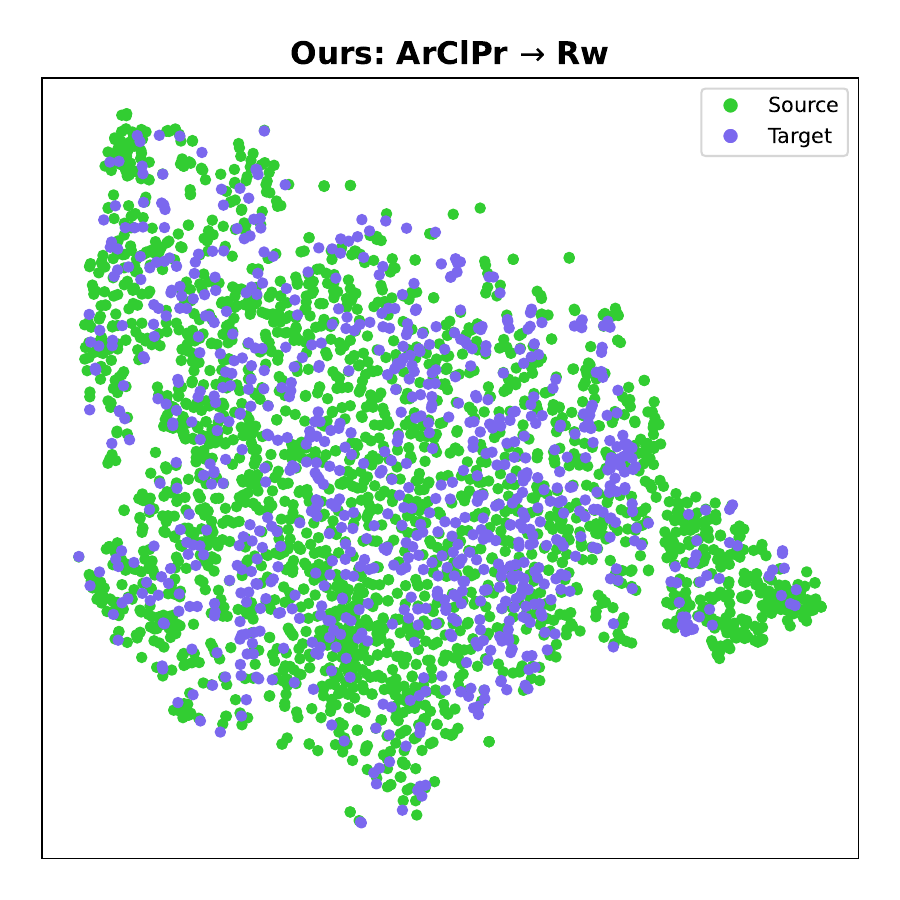}
    \end{minipage}
    \hfill
    \begin{minipage}[b]{0.22\textwidth}
        \centering
        \includegraphics[width=\textwidth]{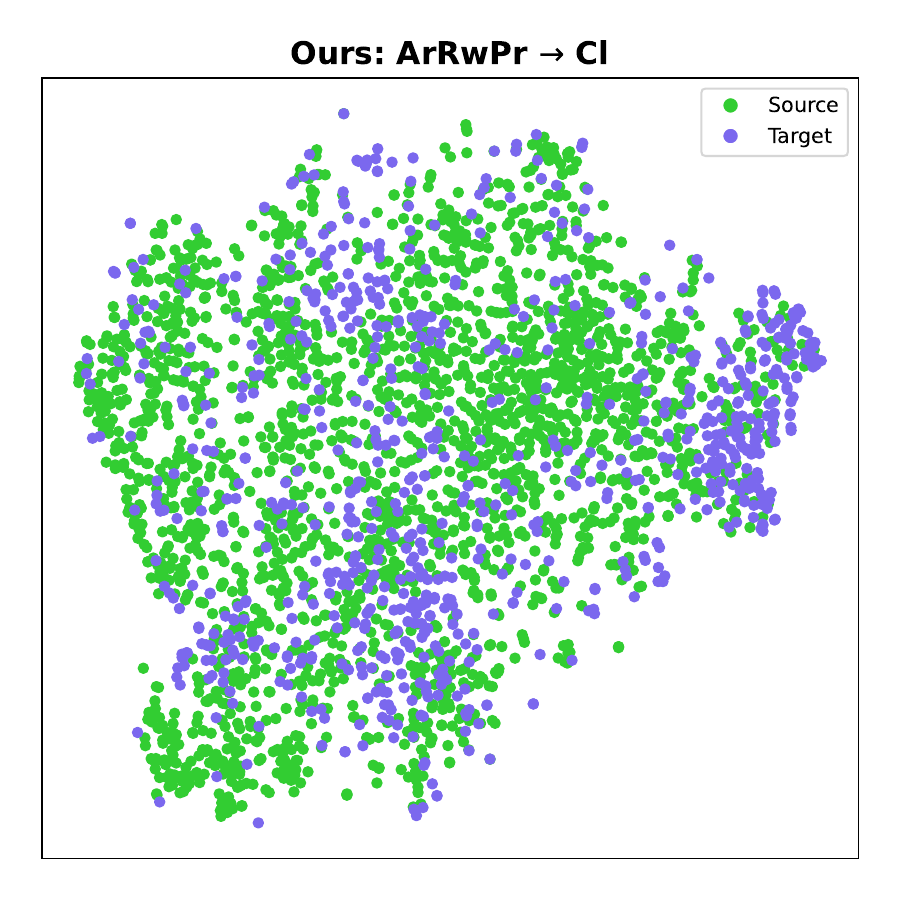}
    \end{minipage}
    \begin{minipage}[b]{0.22\textwidth}
        \centering
        \includegraphics[width=\textwidth]{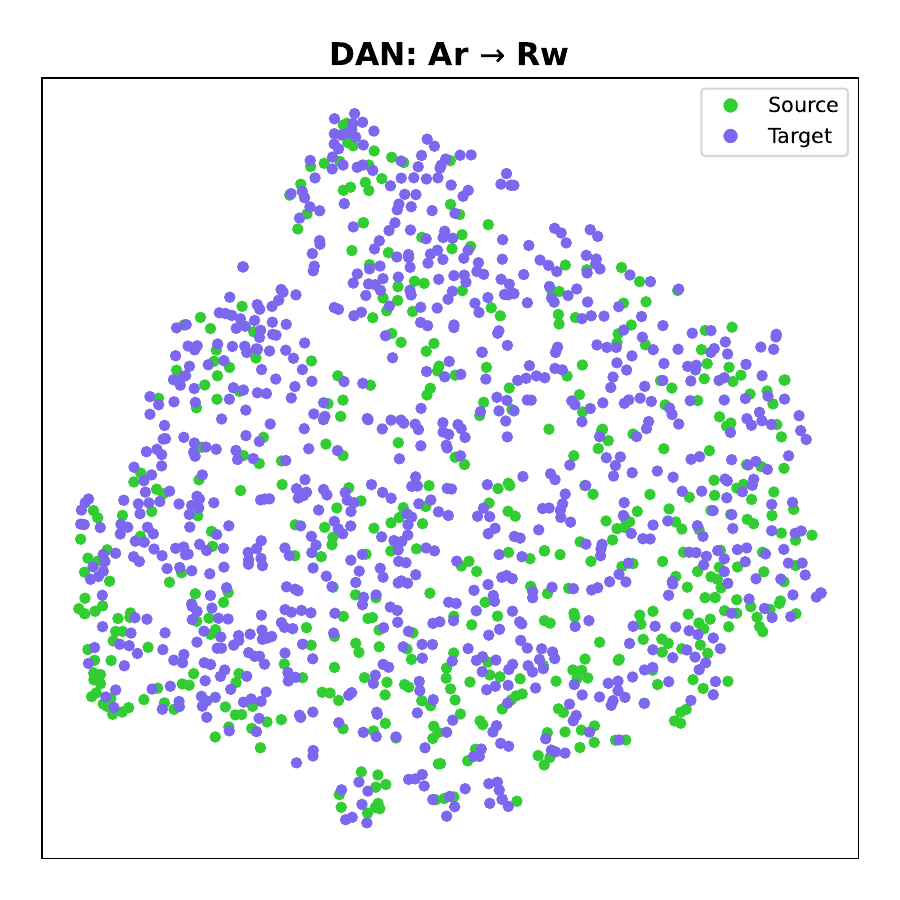}
    \end{minipage}
    \hfill
    \begin{minipage}[b]{0.22\textwidth}
        \centering
        \includegraphics[width=\textwidth]{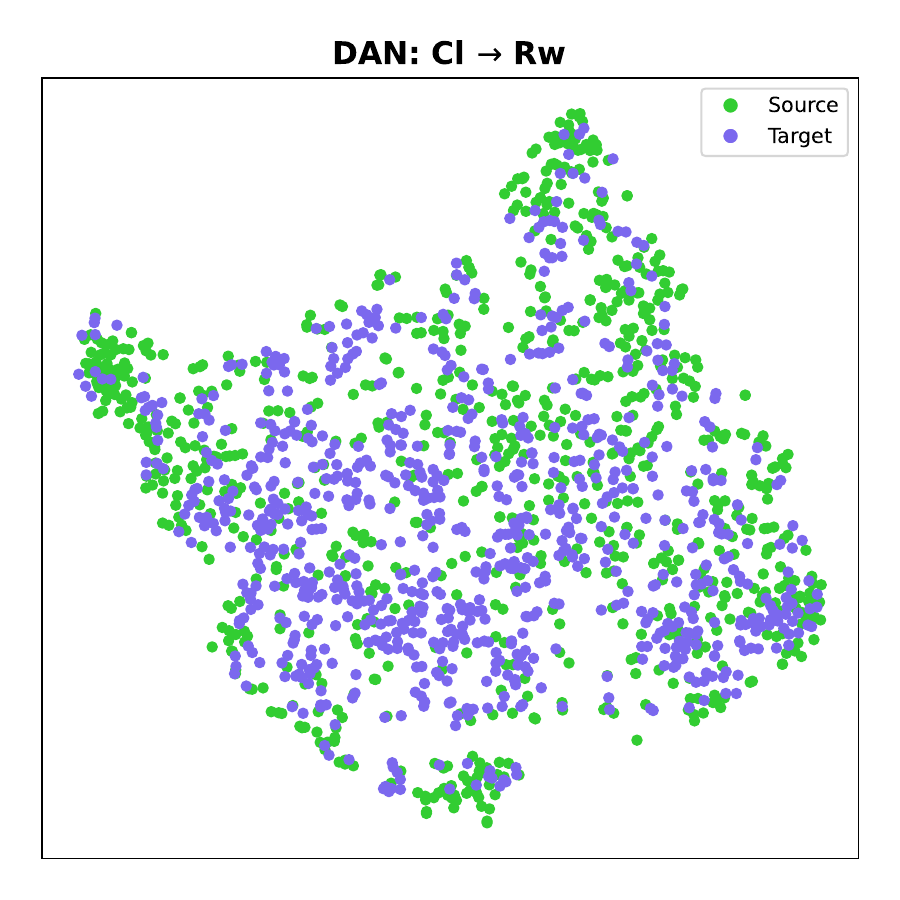}
    \end{minipage}
    \hfill
    \begin{minipage}[b]{0.22\textwidth}
        \centering
        \includegraphics[width=\textwidth]{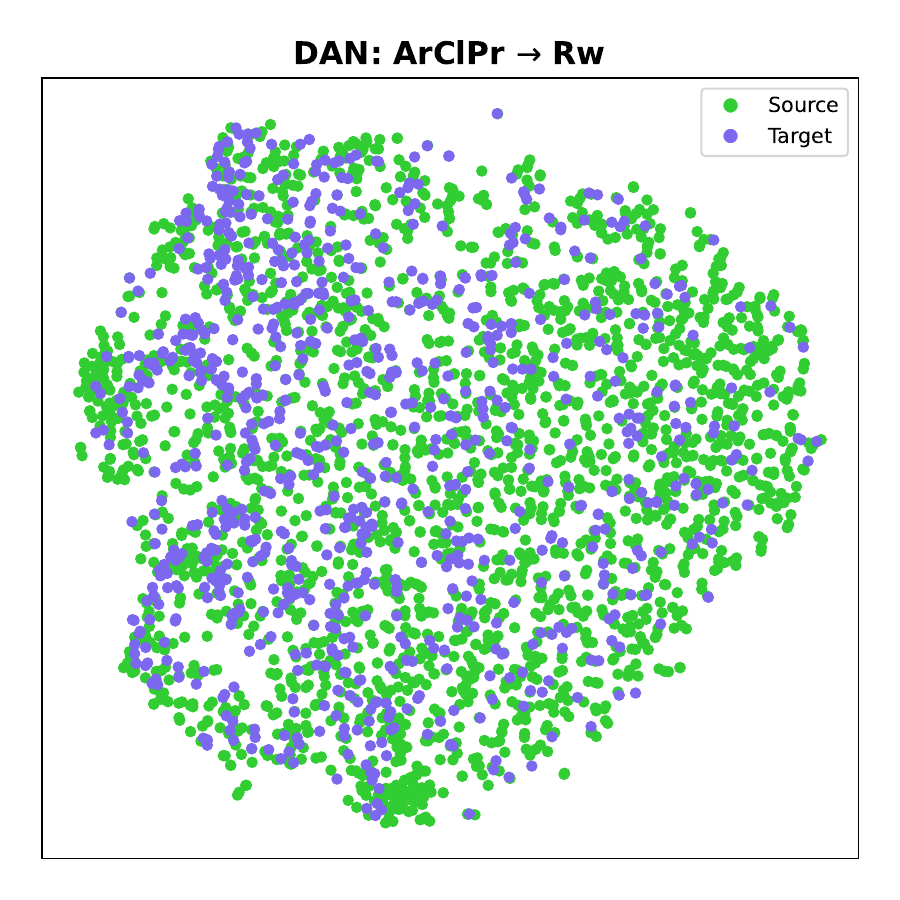}
    \end{minipage}
    \hfill
    \begin{minipage}[b]{0.22\textwidth}
        \centering
        \includegraphics[width=\textwidth]{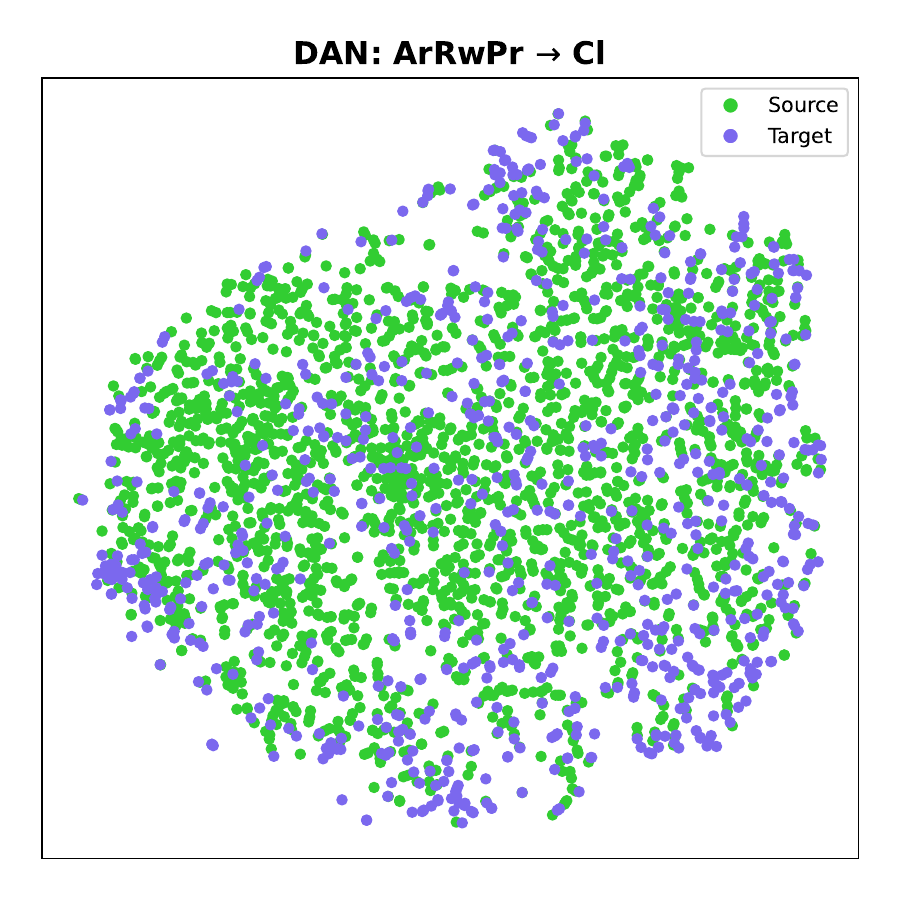}
    \end{minipage}
    \begin{minipage}[b]{0.22\textwidth}
        \centering
        \includegraphics[width=\textwidth]{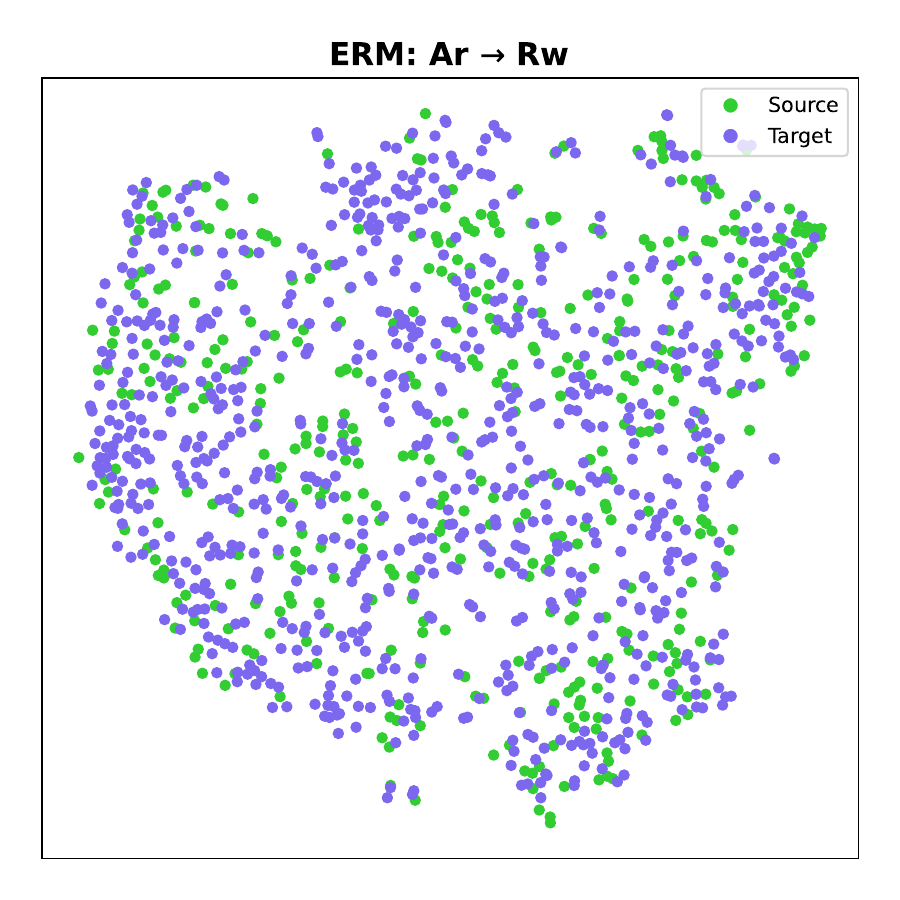}
    \end{minipage}
    \hfill
    \begin{minipage}[b]{0.22\textwidth}
        \centering
        \includegraphics[width=\textwidth]{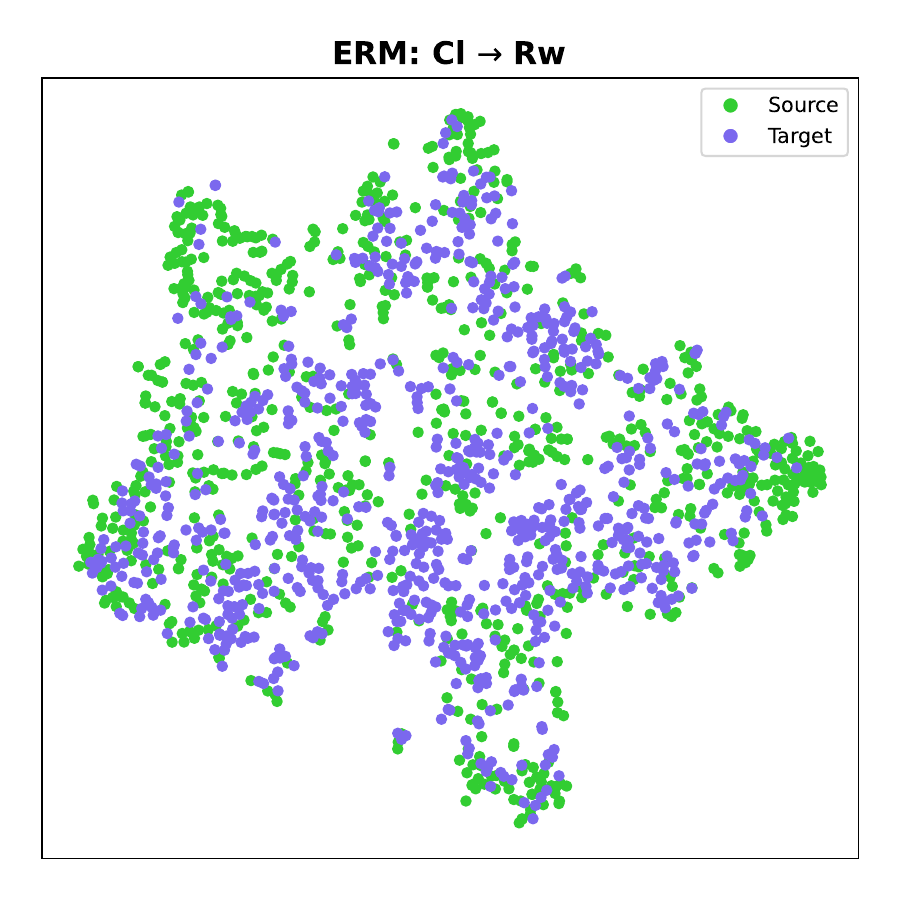}
    \end{minipage}
    \hfill
    \begin{minipage}[b]{0.22\textwidth}
        \centering
        \includegraphics[width=\textwidth]{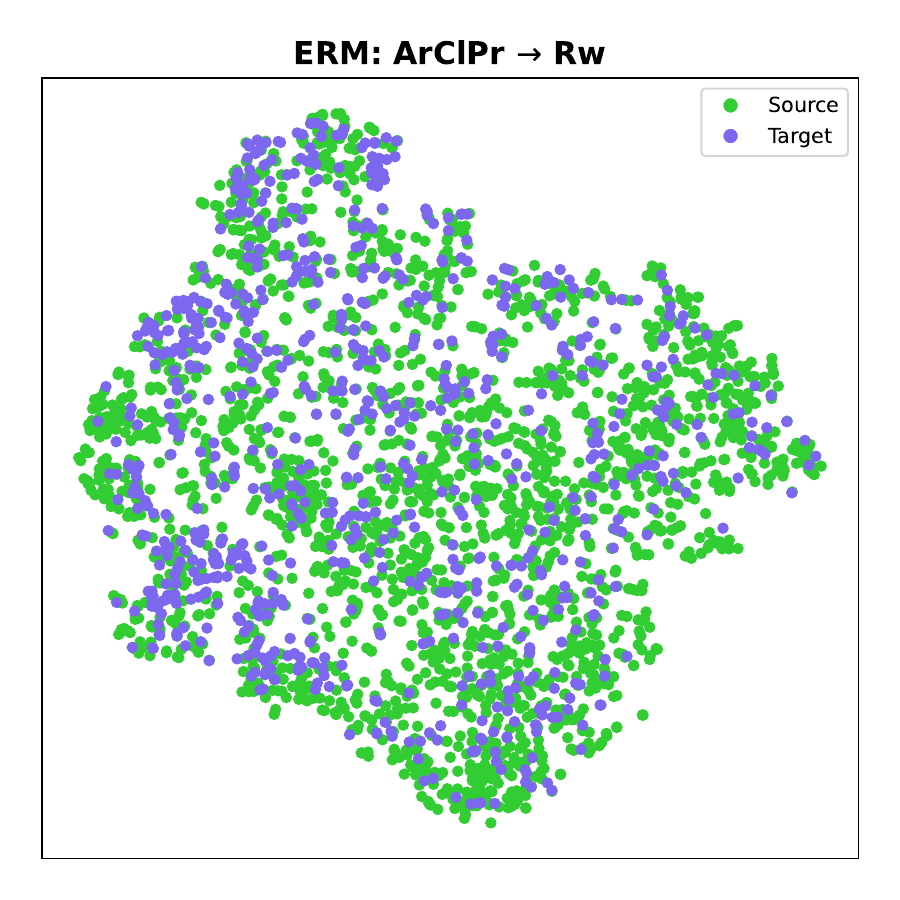}
    \end{minipage}
    \hfill
    \begin{minipage}[b]{0.22\textwidth}
        \centering
        \includegraphics[width=\textwidth]{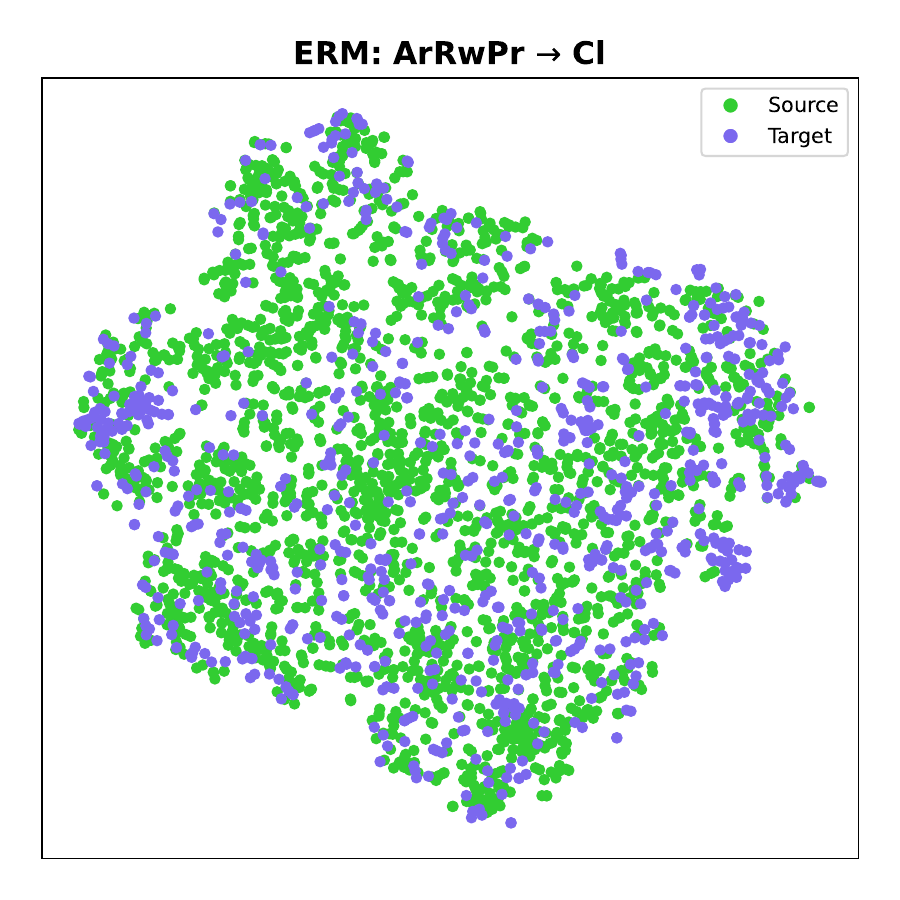}
    \end{minipage}
    \begin{minipage}[b]{0.22\textwidth}
        \centering
        \includegraphics[width=\textwidth]{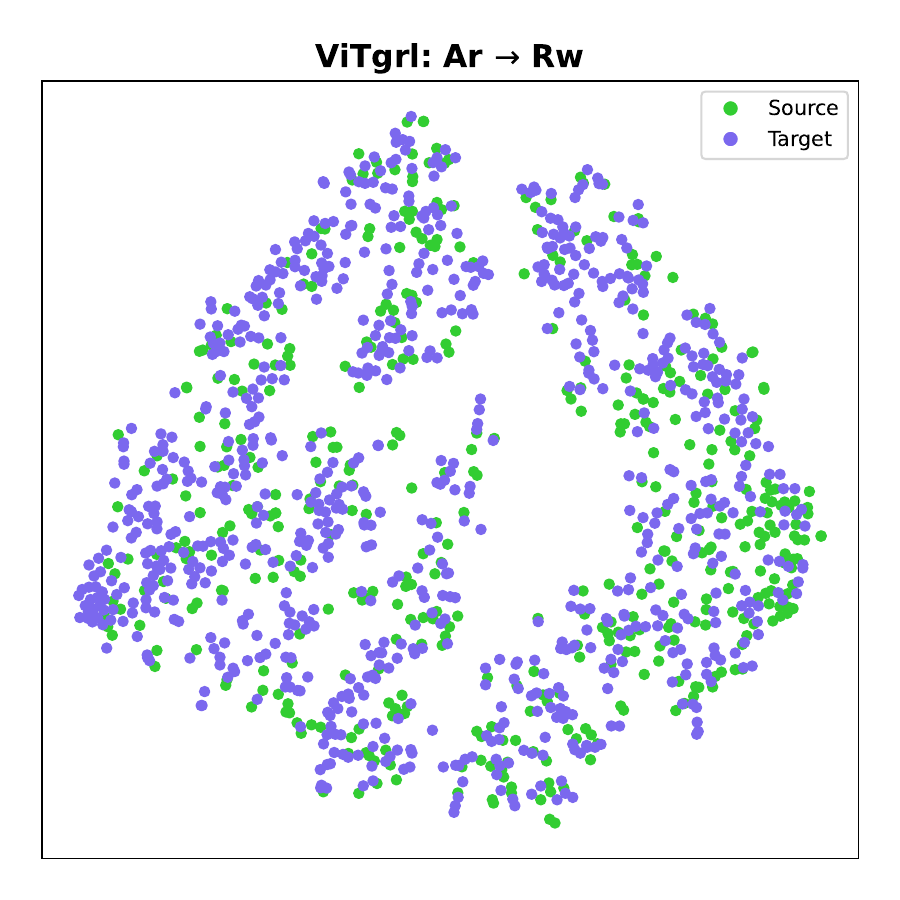}
    \end{minipage}
    \hfill
    \begin{minipage}[b]{0.22\textwidth}
        \centering
        \includegraphics[width=\textwidth]{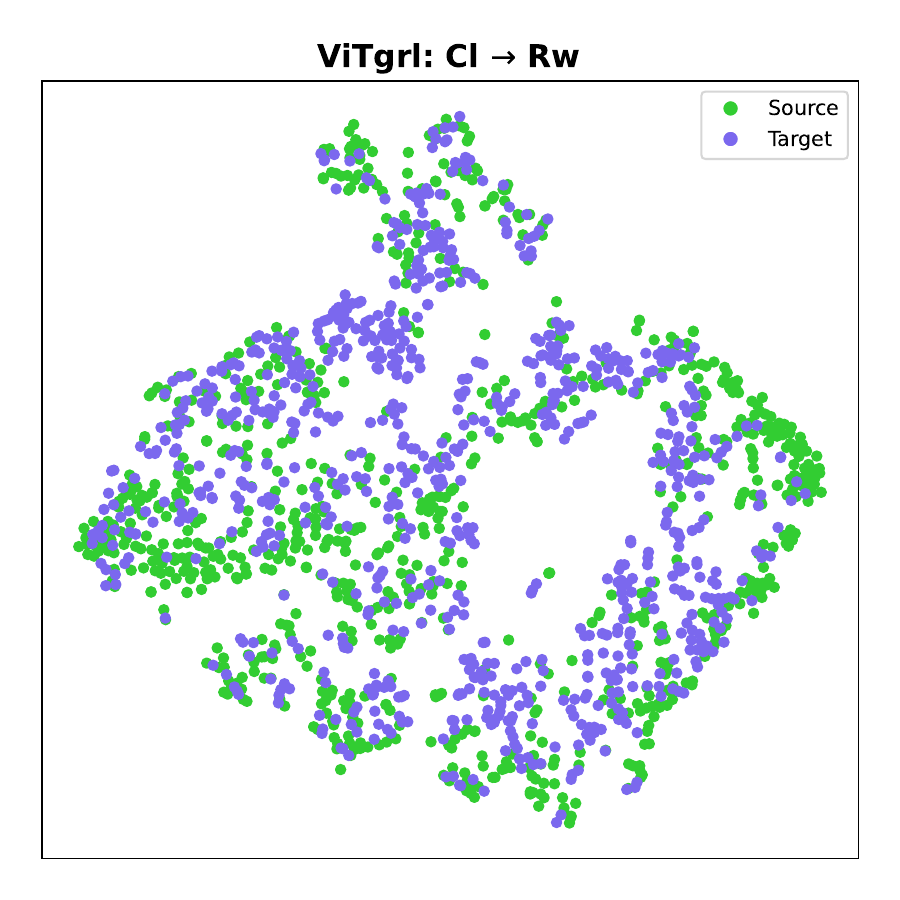}
    \end{minipage}
    \hfill
    \begin{minipage}[b]{0.22\textwidth}
        \centering
        \includegraphics[width=\textwidth]{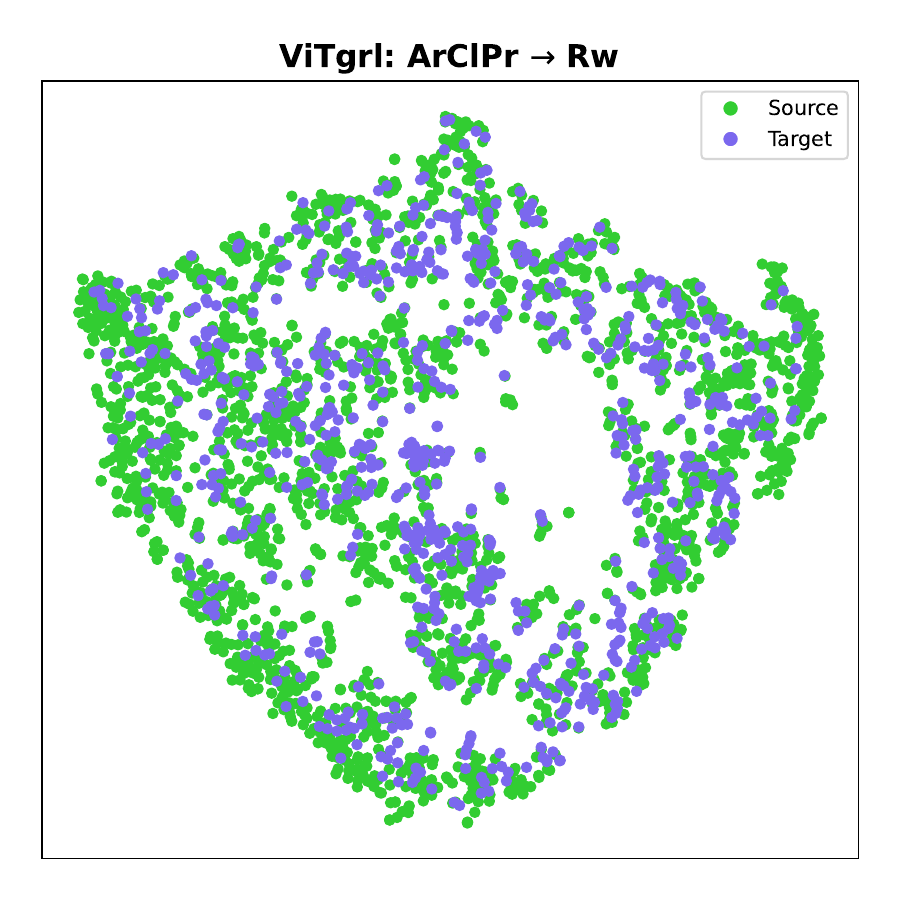}
    \end{minipage}
    \hfill
    \begin{minipage}[b]{0.22\textwidth}
        \centering
        \includegraphics[width=\textwidth]{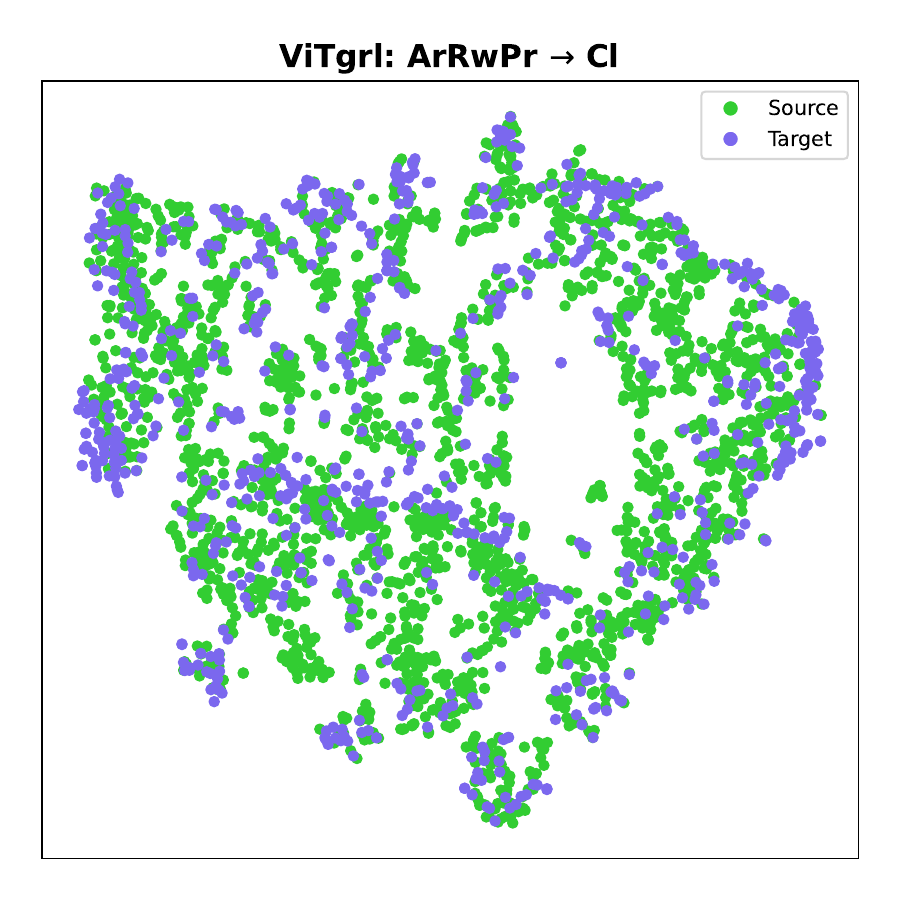}
    \end{minipage}
    \caption{The t-SNE plots for the representation of the source and target domains in Office-Home, the top row is our method, and other rows are the baselines. The first two columns correspond to single-source domain scenarios, and the remaining two correspond to multi-source domains.}
    \label{tsen-officeHome}
\end{figure}
\end{document}